%% file: main.tex
\documentclass{article}

\usepackage[final]{neurips_2025}
\input{math_commands}

\usepackage[utf8]{inputenc} % allow utf-8 input
\usepackage[T1]{fontenc}    % use 8-bit T1 fonts
\usepackage{hyperref}       % hyperlinks
\usepackage{url}            % simple URL typesetting
\usepackage{booktabs}       % professional-quality tables
\usepackage{amsfonts}       % blackboard math symbols
\usepackage{nicefrac}       % compact symbols for 1/2, etc.
\usepackage{microtype}      % microtypography
\usepackage{xcolor}         % colors
\usepackage{ragged2e}

\newtheorem{lemma}{Lemma}

\newtheorem{theorem}{Theorem}
\newtheorem{proposition}{Proposition}
\newtheorem{definition}{Definition}
\title{On Minimax Estimation of Parameters in Softmax-Contaminated Mixture of Experts}

\author{%
  Fanqi Yan$^{\star,1}$ 
  \quad
  Huy Nguyen$^{\star,2}$ 
  \quad
  Dung Le$^{\star,2}$ 
  \quad
  Pedram Akbarian$^{3}$ 
  \\
{\bf  Nhat Ho$^{2}$ 
  \quad 
  Alessandro Rinaldo$^{2}$} \\
  \textsuperscript{1} Department of Computer Science,~
  \textsuperscript{2} Department of Statistics and Data Sciences, \\
  \textsuperscript{3} Department of Electrical and Computer Engineering, ~
  The University of Texas at Austin\\
  \texttt{\{fanqi.yan, huynm, quangdung0110, akbarian, minhnhat\}@utexas.edu}, 
  \\
  \texttt{alessandro.rinaldo@austin.utexas.edu}
}

\begin{document}

\maketitle

\begin{abstract}
    The softmax-contaminated mixture of experts (MoE) model is deployed when a large-scale pre-trained model, which plays the role of a fixed expert, is fine-tuned for learning downstream tasks by including a new contamination part, or prompt, functioning as a new, trainable expert. Despite its popularity and relevance, the theoretical properties of the softmax-contaminated MoE have remained unexplored in the literature. In the paper, we study the convergence rates of the maximum likelihood estimator of gating and prompt parameters in order to gain insights into the statistical properties and potential challenges of fine-tuning with a new prompt. We find that the estimability of these parameters is compromised when the prompt acquires overlapping knowledge with the pre-trained model, in the sense that we make precise by formulating a novel analytic notion of distinguishability. Under distinguishability of the pre-trained and prompt models, we derive minimax optimal estimation rates for all the gating and prompt parameters. By contrast, when the distinguishability condition is violated, these estimation rates become significantly slower due to their dependence on the prompt convergence rate to the pre-trained model. Finally, we empirically corroborate our theoretical findings through several numerical experiments.
    % The contaminated mixture of experts (MoE) model is motivated from parameter-efficient fine-tuning methods where a large-scale pre-trained model, which plays a role as a fixed expert, is fine-tuned for learning downstream tasks by incorporating a contamination part called prompt, functioning as a trainable expert. Despite its relevance to practice, the theoretical properties of the softmax gating contaminated MoE has remained underexplored in the literature. In the paper, we perform a convergence analysis of maximum likelihood parameter estimation to study the statistical effects of the prompt on the parameter convergence. Our investigation indicates that a challenge on the convergence of gating parameters occurs when the prompt gains overlapped knowledge with the pre-trained model. To characterize this issue, we establish a distinguishability condition to manage the similarity between these two models. When the prompt is distinguishable from the pre-trained model, we obtain minimax optimal rates of parametric order for all the parameter estimations. Conversely, when the distinguishability condition is violated, we show that gating parameter estimation rates are dramatically slowed down due to their dependence on the prompt convergence rate to the pre-trained model. Finally, we derive minimax lower bounds for estimating all the parameters. Our theoretical findings are then empirically demonstrated through several numerical experiments.
\end{abstract}
\let\thefootnote\relax\footnotetext{$^\star$Co-first authors.}

\section{Introduction}
\label{sec:introduction}
Mixture of experts (MoE) \cite{Jacob_Jordan-1991,Jordan-1994} has emerged as a statistical machine learning model that  
aggregates the power of multiple sub-models. 
This model consists of two primary components: expert function (or, simply, expert) and a gating network. 
Experts can be, for example, a feed-forward network (FFN) \cite{shazeer2017topk,dai2024deepseekmoe}, a classifier \cite{chen2022theory,nguyen2024general}, or a regression model \cite{faria2010regression,kwon_em_2020}. The gating network softly divides the input space into multiple regions where the opinions of some experts are deemed to be more trustworthy than others. This is done by dynamically allocating higher input-dependent weights instead of constant weights to the various experts, making MoE more flexible and adaptive than traditional mixture models \cite{Lindsay-1995}. As a consequence, MoE has been leveraged in a wide range of fields, including natural language processing \cite{deepseekv3,jiang2024mixtral,grattafiori2024llama3,fedus2022switch,lepikhin_gshard_2021,nguyen2025competesmoe,shazeer2017topk}, computer vision \cite{Riquelme2021scalingvision,liang_m3vit_2022}, speech recognition \cite{You_Speech_MoE,You_Speech_MoE_2}, multimodal learning \cite{han2024fusemoe,yun2024flexmoe,nguyen2024hmoe}, continual learning \cite{le2024mixture,li2025cl}, domain adaptation \cite{nguyen2025cosine,li2023sparse}, and reinforcement learning \cite{ceron2024rl,chow_mixture_expert_2023}. 

Unlike these applications where all experts are trainable, parameter-efficient fine-tuning methods such as prefix tuning \cite{li-liang-2021-prefix,le2025revisiting,le2025expressivenessvisualpromptexperts} can be interpreted as a mixture of a frozen or pre-trained expert and a trainable prompt expert responsible for learning downstream or more specialized tasks, which we refer to as \emph{contaminated MoE} throughout this paper. 
Despite the empirical success of this fine-tuning approach, there is a very limited theoretical understanding of their properties and limitations in the literature. 
To the best of our knowledge, contaminated MoE has only been previously studied in \cite{yan2025contaminated,nguyen2024deviated} to characterize expert structures achieving the optimal parameter estimation rates. However, the analysis in that work is conducted under a simplified setting where the gating (mixture weight) is independent of the input value, which is a very impractical assumption. 
To close this gap, we undertake a thorough theoretical analysis of the more commonly used \emph{softmax-contaminated MoE} model, specified in equation~\eqref{eq:contaminated_pretrain_model_general} below, a contaminated MoE model whose gating function takes the form of a soft-maxed linear network.% followed by a softmax function. 
We analyze the issue of identifiability and the convergence properties of the maximum likelihood estimator of the prompt parameters to shed light on the understanding of prompt behavior in prefix tuning methods. 
A main take-away of our analysis is the potential for the prompt to be exceedingly similar to --  and thus to acquire the same knowledge as --  the pre-trained model, a situation greatly impacting the estimability of the prompt parameter. 
To overcome this issue, in Definition~\ref{def:distinguishability} we formulate analytical properties of the pre-trained and prompt models, which we refer to as {\it distinguishability,} that are guaranteed to rule out excessive overlap between the models and ensure good estimation rates. 
We make the
following contributions.

%\textbf{Contributions.} Firstly, we present the problem setup as well as study the identifiability and convergence of the softmax-contaminated MoE in Section~\ref{sec:preliminaries}. Then, in Section~\ref{sec:theory}, we perform a convergence analysis of prompt and gating parameter estimations to shed light on the understanding of prompt behavior in prefix tuning methods. We notice that a challenge arises when the prompt acquires the same knowledge as the pre-trained model. Therefore, we establish a distinguishability condition to control the similarity between these two models. Our main contributions are two-fold and can be summarized as follows:

\emph{(i) Distinguishability of the prompt model from the pre-trained model.} In Section~\ref{sec:preliminaries}, we propose a novel notion of distinguishability between the pre-trained and prompt models and then illustrate its properties.

\emph{(ii) When the distinguishability condition is satisfied,} we show in Section~\ref{sec:distinguishable} that the prompt does not converge to the pre-trained model --  intuitively, these two models have distinct expertise. In fact, we demonstrate that the convergence rates of the MLE of all the prompt and gating parameters are of parametric order in the sample size $n$, that is, 
% $\widehat{\mathcal{O}}_P(n^{-1/2})$
{$\widetilde{\mathcal{O}}(n^{-1/2})$}. Furthermore, we establish minimax lower bounds on the estimation errors with matching rates, thus showing that the convergence rate of MLE is minimax optimal.

\emph{(iii) When the distinguishability condition is violated,} the prompt will converge to the pre-trained model, that is, both models employ the same expert structure and thus will gain similar expertise. In Section~\ref{sec:non-distinguishable}, we show that, under this setting, the estimation rates for prompt and gating parameters are negatively affected by the prompt convergence to the pre-trained model and, therefore, become substantially slower than the parametric rate 
% $\widehat{\mathcal{O}}_P(n^{-1/2})$
{$\widetilde{\mathcal{O}}(n^{-1/2})$}. We confirm that these slower rates are tight by deriving matching minimax lower bounds. See Table~\ref{table:parameter_rates} for a summary of our results. 

Lastly, in Section~\ref{sec:experiments}, we carry out several numerical experiments to empirically justify our theoretical results, and then conclude the paper in Section~\ref{sec:conclusion}.
% before concluding the paper in Section~\ref{sec:conclusion}. 
Rigorous proofs are provided in the 
Appendices. %due to the space limit. 

A major technical innovation in our contribution that sets it apart from existing theoretical analyses of MoE models is the fact that we let the parameters of the prompt model to vary with the sample size $n$, thus potentially allowing for a more challenging estimation task as the sample size increases. This approach is necessary to carry out a minimax analysis.

%\textbf{Organization.}

\textbf{Notation.}
For any $n\in\mathbb{N}$,  we let $[n]: = \{1, 2, \ldots, n\}$. For a vector $u$ we denote with $\|u\|$ its Euclidean norm value. Given any two positive sequences $(a_n)_{n\geq 1}$ and $(b_n)_{n\geq 1}$, we write $a_n = \mathcal{O}(b_n)$ or $a_{n} \lesssim b_{n}$ if $a_n \leq C b_n$ for all $ n\in\mathbb{N}$ and some $C > 0$. We further write $a_n=\widetilde{\mathcal{O}}(b_n)$ to denote $a_{n} \lesssim b_{n} \mathrm{polylog}(b_n)$, where $\mathrm{polylog}(b_n)$ indicate any term that is polylogarithmic in $b_n$. 
% Similarly, if  $(A_n)_{n\geq 1}$ is a sequence of non-negative random variable, we write $A_n = \mathcal{O}_P(b_n)$ (resp. $A_n = \widetilde{\mathcal{O}}_P (b_n)$ ) if $A_n/b_n$ (resp. $A_n/(b_n \mathrm{polylog}(b_n))$) is bounded in probability. 
Lastly, for any two densities $p$ and $q$ (dominated by the Lebesgue measure), their squared Hellinger distance is computed as $d_H^2(p, q):=\frac{1}{2} \int[\sqrt{p(x)}-\sqrt{q(x)}]^2 d x$, while the total variation distance is given by $d_V(p,q):=\frac{1}{2}\int|p(x)-q(x)|d x$.

\begin{table*}[!ht]
\caption{Summary of parameter estimation rates in the softmax-contaminated MoE model. Notice that the rates are in expectation. For the notation, please refer to equations~\eqref{eq:contaminated_pretrain_model_general} and \eqref{eq:MLE}. In addition, we also denote $\Delta\eta^*:=\eta^*-\eta_0$ and $\Delta\nu^*:=\nu^*-\nu_0$. 
%\textcolor{violet}{Say that the estimation rates are in expectation, otherwise they look like $\mathcal{O}_P$ rates}
% \textcolor{violet}{ALE: after going through the paper, I notice that these are $\widetilde{\mathcal{O}}$ and NOT $\widetilde{\mathcal{O}}_P$ rates because we are not bounding  the loss but the risk function.}
% Below, we denote $\Delta \Gs:=(\as-a_0,\bs-b_0,\nus-\nu_0)$ and $\vartheta_{\Delta \Gs}^{\alpha,\beta,\gamma}:=\Vert \as-a_0 \Vert^{\alpha}+|\bs-b_0 |^{\beta}+|\nus -\nu_0|^{\gamma}$. 
% Additionally, $f_0=f$ means that $f_0$ is a Gaussian density, while $f_0\neq f$ indicates the opposite.
}
\centering
\begin{tabular}{ |c|c|c|c|c|c|c|c|c|c| } 
\hline
 \textbf{Setting}
& $\boldsymbol{|\exp(\widehat{\tau}_n)-\exp(\tau^*)|}$ &$\boldsymbol{\|\widehat{\beta}_n-\beta^*\|}$
&$\boldsymbol{\|\widehat{\eta}_n-\eta^*\|}$ 
&$\boldsymbol{|\widehat{\nu}_n-\nu^*|}$
\\
\hline
 Distinguishable 
 & $\widetilde{\mathcal{O}} (n^{-\frac{1}{2}})$ 
 & \multicolumn{3}{c|}{$\widetilde{\mathcal{O}} (n^{-\frac{1}{2}})$}
 \\ 
 \hline
 Non-distinguishable
& $\widetilde{\mathcal{O}} (n^{-\frac{1}{2}}\cdot\Vert(\Delta\eta^*,\Delta\nu^*)\Vert^{-2})$ & \multicolumn{3}{c|}{$\widetilde{\mathcal{O}} (n^{-\frac{1}{2}}\cdot\Vert (\Delta\eta^*,\Delta\nu^*)
 \Vert^{-1})$} 
\\ \hline
\end{tabular}
\label{table:parameter_rates}
\end{table*}

\section{Preliminaries}
\label{sec:preliminaries}
In this section, we begin with setting up the problem, followed by a discussion on related works in Section~\ref{sec:problem_setup}. Then, in Section~\ref{sec:fundamental}, we introduce the distinguishability condition and provide an investigation into the fundamental properties of the softmax-contaminated MoE, including the model identifiability and the model convergence.

\subsection{Problem Setup}
\label{sec:problem_setup}
\textbf{Problem setting.} 
Suppose that $(X_1 , Y_1), (X_2,Y_2), \ldots , (X_n , Y_n ) \in \mathcal{X}\times\mathcal{Y} \subset\mathbb{R}^d\times\mathbb{R} $
are i.i.d. samples of covariate-response pairs of size $n$. We assume that the input covariates $X_1,X_2,\ldots,X_n$ are drown in an i.i.d. manner from some known continuous
% \fqyedit{continuous}
probability distribution on $\mathbb{R}^d$ 
% \textcolor{violet}{ALE: any regularity conditions?} 
and that the responses are generated according to a softmax-contaminated MoE model, which postulates that the conditional density function of the response given the covariates is given by
\begin{align}
\label{eq:contaminated_pretrain_model_general}
    p_{ G_*}(y|x) & := \frac{1}{1+\exp((\beta^*)^{\top}x+\tau^*)}\cdot f_{0}(y|h_0(x,\eta_0), \nu_{0})  
    \nonumber \\& \hspace{5 em} 
    + \frac{\exp((\beta^*)^{\top}x+\tau^*)}{1+\exp((\beta^*)^{\top}x+\tau^*)}\cdot f(y|h(x,\eta^*),\nu^*).
\end{align}
Above, the pre-trained model corresponds to as a {\it fixed} and known conditional probability density function $f_0(\cdot|h_0(\cdot,\eta_0),\nu_0)$, parametrized by the pre-trained mean expert function $x \mapsto h_0(x,\eta_0)$ and variance $\nu_0$. Meanwhile, the prompt model, denoted as $f(\cdot|h(\cdot,\eta^*),\nu^*)$ is modeled as an unknown Gaussian density function with the prompt mean expert $x \mapsto h(x,\eta^*)$ and variance $\nu^*$. We collect all the unknown parameters of the prompt model into the vector $\Gs=(\betas,\taus,\etas,\nus)$, belonging to some {\it parameter space} $\Xi\subseteq\mathbb{R}^d\times\mathbb{R}\times\mathbb{R}^q\times\mathbb{R}_{+}$. Note that we allow the values of these parameters to vary with the sample size $n$. However, for notational convenience, we suppress the dependence of $G_*$ on $n$ throughout the paper. In addition, it should also be noted that the ``probabilistic" MoE model~\eqref{eq:contaminated_pretrain_model_general} can be related to ``deterministic" MoE models used in deep learning \cite{shazeer2017topk} by taking the expectation of the response given the covariate, that is,
\begin{align*}
    \mathbb{E}[Y|X]=\frac{1}{1+\exp((\beta^*)^{\top}x+\tau^*)}\cdot h_0(x,\eta_0)   
    + \frac{\exp((\beta^*)^{\top}x+\tau^*)}{1+\exp((\beta^*)^{\top}x+\tau^*)}\cdot h(x,\eta^*).
\end{align*}

\textbf{Maximum likelihood estimation (MLE).} We utilize the maximum likelihood method \citep{Vandegeer-2000} to estimate the unknown parameters $\Gs=(\betas,\taus,\etas,\nus)$ of the softmax-contaminated MoE model~\eqref{eq:contaminated_pretrain_model_general} as follows:
\begin{align}
\label{eq:MLE}
    \widehat{G}_n:=(\widehat{\beta}_n,\widehat{\tau}_n,\widehat{\eta}_n,\widehat{\nu}_n)\in
    \argmax_{G\in\Xi}
    \sum_{i=1}^n
    \log(p_{G}(Y_i|X_i)).
\end{align}
For the sake of theory, we assume that the input space $\mathcal{X}$ is bounded, whereas the parameter space $\Xi$ is compact. In addition, we assume that the prompt expert function $x\mapsto h(x,\eta)$ is differentiable with respect to $\eta\in\mathbb{R}^q$ for almost all $x\in\mathcal{X}$. Note that these assumptions are mild and have been used in previous works \cite{ho2022gaussian, nguyen2024deviated, yan2025contaminated}.

\textbf{Related work.} %There is a line of works focusing on the convergence behavior of parameter estimation in the MoE literature. 
Mendes et al. \cite{mendes2011convergence} considered an MoE model where each expert was formulated as a polynomial regression model. Their objective was to address the trade-off between the number of experts and the expert size to obtain the optimal parameter estimation rates. Next, Ho et al. \cite{ho2022gaussian} took into account the parameter estimation problem for Gaussian MoE models with input-free gating. They demonstrated that when expert functions satisfied an algebraic independence condition, the convergence rates of MLE were optimal of parametric order on the sample size. Conversely, if the expert functions are not algebraic independent, then the parameter estimation rates became inversely proportional to the number of fitted experts. These results were then extended to more practical settings of input-dependent gatings, including softmax gating \cite{nguyen2023demystifying} and sigmoid gating \cite{nguyen2024sigmoid}, revealing that the latter was more sample-efficient than the former in terms of expert estimation.

It was not until 2024 that Nguyen et al. \cite{nguyen2024deviated} investigated a contaminated MoE where a frozen pre-trained model was fine-tuned by a mixture of prompts rather than a single prompt model. However, they imposed two unrealistic assumptions on their model of interest: they equipped the contaminated MoE with input-free gating and kept the ground-truth parameters unchanged with the sample size. Then, Yan et al. \cite{yan2025contaminated} overcame the second limitation by allowing ground-truth parameters to hinge on the sample size as in the case of traditional mixture models \cite{do2023deviated}, while the first limitation remained unsolved. Therefore, in this work, our goal is to completely address both limitations by studying the softmax-contaminated MoE in equation~\eqref{eq:contaminated_pretrain_model_general}.

\textbf{Challenges.} There are three fundamental challenges of our analysis compared to previous work.

\emph{1. Uniform convergence rates.} We allow ground-truth parameters $G_*$ to change with sample size $n$, which is challenging yet closer to practice than the settings in previous works on MoE \cite{nguyen2023demystifying,nguyen2024sigmoid}, where $G_*$ does not change with $n$. Thus, the convergence rates of parameter estimations in our work are uniform rather than point-wise as in those works. 

\emph{2. Minimax lower bounds.} We determine minimax lower bounds under both distinguishable and non-distinguishable settings. Based on these lower bounds, we can claim that our derived convergence rates are optimal. However, no minimax lower bounds are provided in \cite{nguyen2023demystifying,nguyen2024sigmoid}.

\emph{3. Input-dependent gating.} The latest work on understanding the contaminated MoE model is \cite{yan2025contaminated}, but it considers input-free gating in the analysis. On the other hand, in this paper, we take into account softmax gating, which hinges upon the input value. This input-dependence yields several challenges on the convergence of density estimation and parameter estimation.

\subsection{Fundamental Properties of the Softmax-Contaminated MoE}
\label{sec:fundamental}
As mentioned above, when the prompt's learned skills overlap with those of the pre-trained model, estimating the prompt parameters becomes challenging due to potential non-identifiability. To capture that issue accurately, we introduce an analytic condition called distinguishability in Definition~\ref{def:distinguishability}.
\begin{definition}[Distinguishability]
\label{def:distinguishability}
% Let \( f_0(Y  |  h_0(X, \eta_0), \nu_0) \) be a fixed base density, and let \( f(Y  |  h(X, \eta), \nu) \) be a family of prompt densities parameterized by \( (\eta, \nu) \in \Theta \subset \mathbb{R}^q \times \mathbb{R}_{>0} \), where \( h : \mathcal{X} \times \mathbb{R}^q \to \mathbb{R} \) is a differentiable expert mean function.
We say that \( f_0 \) is \emph{distinguishable from} \( f \) if the following hold:
for any distinct pairs of parameters \( (\eta_1, \nu_1), (\eta_2,\nu_2) \in \Theta \), if there exist measurable real-valued functions \( x \in \mathcal{X} \mapsto b_0(x) \), \( x \in \mathcal{X} \mapsto b_1(x) \), and \(  x \in \mathcal{X} \mapsto  \{c_\alpha(x)\}_{0\leq|\alpha| \leq 1}  \), where \( \alpha = (\alpha_1, \alpha_2) \in \mathbb{N}^q \times \mathbb{N} \) with \( |\alpha| = |\alpha_1| + \alpha_2 \leq 1 \) such that
\begin{align*}
    b_0(x)\cdot f_0(y  |  h_0(x, \eta_0), \nu_0)
    &+b_1(x)\cdot f(y|h(x,\eta_1),\nu_1)\\
    &+ \sum_{0\leq|\alpha| \leq 1} c_\alpha(x) \cdot 
    \frac{\partial^{|\alpha|} f}{\partial \eta^{\alpha_1} \partial \nu^{\alpha_2}}(y  |  h(x, \eta_2), \nu_2) = 0,
\end{align*}
for almost every \( (x,y)  \in \mathcal{X} \times \mathcal{Y} \), then it must be the case that
    \[
    b_0(x) = b_1(x) = 0, \quad c_\alpha(x) = 0 \quad \text{ for all } 0\leq|\alpha| \leq 1, \quad \text{for almost every } x.
    \]
\end{definition}

To help understand the notion of distinguishability better, in our next result we characterize the class of pre-trained models distinguishable from the prompt $f$. The proof can be found in Appendix~\ref{appendix:lemma:distinguish-linear sigma not Gaussian}.
% \begin{proposition}
% \label{lemma:distinguish-linear sigma not Gaussian}
%     If a pre-trained model $f_0$ does not belong to the family of Gaussian densities $f$, 
%     and  \( f(Y  |  h(X, \eta), \nu) \) be a Gaussian conditional density with differentiable mean function \( h \),
%     then $f_0$ is distinguishable from the prompt model $f$. 
% \end{proposition}

\begin{proposition}
\label{lemma:distinguish-linear sigma not Gaussian}
If a pre-trained model \( f_0 \) does not belong to the family of Gaussian densities, then \( f_0 \) is distinguishable from the prompt model \( f \) in the sense of Definition~\ref{def:distinguishability}.
\end{proposition}

On the other hand, if $f_0$ belongs to the family of Gaussian distributions and the pre-trained expert shares the same structure as the prompt expert, that is, $h_0=h$, then the above condition is violated. It should be noted that the distinguishability condition ensures that the prompt does not acquire overlapping knowledge with the pre-trained model since the equation $f_0(y|h(x,\eta_0),\nu_0)= f(y|h(x,\eta),\nu)$ %f(\yha)$ 
cannot hold for almost all $(x,y)\in \mathcal{X}\times\mathcal{Y}$. Moreover, we illustrate in the following proposition that the distinguishability condition also implies that the softmax-contaminated MoE is identifiable.
% Furthermore, according to Proposition~\ref{prop:identifiability}, whose proof is in , this condition also ensures that the contaminated MoE model~\eqref{eq:contaminated_pretrain_model_general} is identifiable.

\begin{proposition}[Identifiability]
\label{prop:identifiability}
    Let $G, G'$ be two components in $\Xi$. Suppose that $f$ is distinguishable from $f_0$, then if the identifiability equation 
    $p_{G}(y|x) =p_{G^\prime}(y|x) $
    holds for almost all $(x,y)\in \mathcal{X}\times\mathcal{Y}$, then we obtain $G = G^\prime $.
\end{proposition}
The proof of Proposition~\ref{prop:identifiability} is provided in Appendix~\ref{appendix:identifiability}. Given the consistency of the softmax-contaminated MoE, we continue to investigate the convergence behavior of density estimation under this model in Proposition~\ref{theorem:ConvergenceRateofDensityEstimation} whose proof can be found in Appendix \ref{appendix:ConvergenceRateofDensityEstimation}.
%\subsection{Convergence Rate of Density Estimation}
We conclude this section with a consistency guarantee for the contaminate density itself, which under mild tail conditions on $f_0$, can be estimated at a parametric rate in the Hellinger distance, regardless of the distinguishability between $f_0$ and $f$. 
Below and throughout the paper,  $\mathbb{E}_{p_{G_*,n}}$ denotes the expectation operator with respect to the joint distribution of the data $(X_1,Y_1),\ldots,(X_n,Y_n)$ and assuming the softmax-contaminated MoE model \eqref{eq:contaminated_pretrain_model_general} parametrized by $G^* \in \Xi$, i.e. $Y_i | X_i \sim p_{G_*}$ for all $i$. Instead, $\mathbb{E}_X$ indicates the expectation with respect to the input distribution.
\begin{proposition}[Model Convergence]
\label{theorem:ConvergenceRateofDensityEstimation}
Suppose that the pre-trained model $f_0$ is bounded and, for some $p>0$,
\begin{align}
\label{eq:tail.condition}
    \mathbb{E}_X\left[-\log f_0(y|h_0(X, \eta_0 ),\nu_0)\right]\gtrsim y^p, \quad \text{
for almost every } y\in\mathcal{Y}.
\end{align}
    % Assume the following assumption holds:
    % A2. Given a universal constant $J > 0$, there exists $N > 0$, possibly depending on  $\Xi$, such that for all $n \geq N$ and all $\epsilon > (\log(n)/n)^{1/2}$, we have $\mathcal{J}_B(\epsilon, \overline{P}^{1/2}(\Xi, \epsilon)) \leq J \sqrt{n} \epsilon^2$.
Then, for the MLE $ \widehat{G}_n$ defined in equation~\eqref{eq:MLE}, it holds, for almost all $x \in \mathcal{X}$,
\begin{align}
\label{eq:density.estimation.rate}
   % \sup_{\Gs\in\Xi}
%\mathbb{E}_{p_{G_*,n}}
 %   h\Big(p_{\widehat{G}_n}(\cdot|x),p_{\Gs}(\cdot|x)\Big)
  %  \lesssim \sqrt{\log(n)/n}.
\sup_{\Gs\in\Xi}
\mathbb{E}_{p_{G_*,n}} \Big[ \mathbb{E}_X\left[
    d_H\left(p_{\widehat{G}_n}(\cdot|X),p_{\Gs}(\cdot|X)\right) \right] \Big]
    \lesssim \sqrt{\log(n)/n}.
\end{align}
% \begin{align*}
%     \sup_{\Gs\in\Xi}
%     \mathbb{E}_{p_{ \Gs}}
%     h(p_{ \widehat{G}_n},p_{ \Gs})
%     \lesssim
%     \sqrt{\log n/n}.
% \end{align*}
\end{proposition}
% \textcolor{violet}{ALE: I think what it means by the above bound is $\sup_{\Gs\in\Xi}
%  \mathbb{E}_{X_1,\ldots,X_n} \mathbb{E}_{Y_i|X_i \sim p_{\Gs}, i = 1,\ldots,n}
%     h(p_{\widehat{G}_n}(\cdot,|x),p_{\Gs}(\cdot|x))$ for almost all $x$.} 
The above result shows that the density estimator $p_{\widehat{G}_n}$ converges to the true density $p_{G_*}$ under the Hellinger distance at the near-parametric rate of order $\widetilde{\mathcal{O}}(n^{-1/2})$. 
To extract from this result a convergence guarantee for the MLE $\widehat{G}_n$ itself, we follow a by-now-standard approach in the latest analysis of MoEs; see, e.g., \cite{nguyen2023demystifying}. The main idea is that, if one can exhibit a loss function among parameters, say $D(\widehat{G}_n,\Gs)$, such that $\mathbb{E}_{p_{G_*,n}} [D(\widehat{G}_n,\Gs)] \lesssim \mathbb{E}_{p_{G_*,n}} \Big[ \mathbb{E}_X\left[
    d_H\left(p_{\widehat{G}_n} (\cdot|X),p_{\Gs}(\cdot|X)\right) \right]$, then convergence of $\widehat{G}_n$ in the expected $D(\cdot,\cdot)$ loss, as well potentially information on the rate of convergence, will follow. See Appendix~\ref{appsec:main_results} for further details. 
Throughout the rest of the paper, we assume that the tail condition \eqref{eq:tail.condition} on $f_0$ and the distribution of $X$ used in Proposition~\ref{theorem:ConvergenceRateofDensityEstimation} is in effect.

%In order to leverage the result of Proposition~\ref{theorem:ConvergenceRateofDensityEstimation}, we impose an assumption on the pre-trained model $f_0$ that it is bounded with the tail 
%$\mathbb{E}_X
%\left[
%-\log f_0(Y|h_0 (X, \eta_0 ),\nu_0)
%\right]
%\gtrsim
%Y^p
%$,
%for almost surely $Y \in \mathcal{Y}$ for some $p > 0$, in the rest of the paper unless stating otherwise.

\section{Convergence Analysis of Parameter Estimation}
\label{sec:theory}
In this section, we present various convergence rates for the MLE estimator of the model prompt and gating parameters. In Sections~\ref{sec:distinguishable} and~\ref{sec:non-distinguishable} we provide separate minimax analyses, depending on whether the distinguishability condition of Definition~\ref{def:distinguishability} holds or not, respectively.

%the results for the convergence analysis of parameter estimation. Based on the distinguishability condition, we conduct the analysis in main settings:
%\begin{itemize}
%    \item \emph{Distinguishable setting in Section~\ref{sec:distinguishable}:} the distinguishability condition is satisfied;
 %   \item \emph{Non-distinguishable setting Section~\ref{sec:non-distinguishable}:} the distinguishability condition is violated.
%\end{itemize}
\subsection{Distinguishable Setting}
\label{sec:distinguishable}
% \begin{definition}[First-order Distinguishability]
%     We call that $f_0$ is distinguishable from $f$ up to the first order if the relation 
%     \begin{equation*}
%         c f_0(y|\mu_0,\sigma_0) + a_0f(y|\mu_1,\sigma_1) + a_1 \dfrac{\partial f}{\partial \mu}(y|\mu_1,\sigma_1) + a_2 \dfrac{\partial^2 f}{\partial \mu^2}(y|\mu_1,\sigma_1) = 0
%     \end{equation*}
%     for almost every $y$, implies the coefficient equals to 0:
%     \begin{equation*}
%         c = a_0 = a_1 = a_2 = 0. 
%     \end{equation*}
% \end{definition}

% \begin{definition}[First-order Distinguishability]
% We say that $f$ is distinguishable from $f_0$ up to the first order if $f$ is differentiable in $(\mu, \nu)$, and the following holds:

% \begin{description}
%     \item[D1.] For any component $(\mu', \nu') \in \Theta$, if we have real coefficients $c$, $\tau_\alpha$ for all $\alpha = (\alpha_1, \alpha_2) \in \mathbb{N}^{2}$, with $|\alpha| = \alpha_1 + \alpha_2 \leq 1$, such that
%     \begin{align*}
%         c f_0(x) + \sum_{|\alpha| \leq 1} \tau_\alpha 
%         \frac{\partial^{|\alpha|} f}{\partial \mu^{\alpha_1} \partial \nu^{\alpha_2}}(x  |  \mu', \nu') = 0
%         \quad \text{for all } x \in \mathcal{X},
%     \end{align*}
%     then $c = \tau_\alpha = 0$ for all $|\alpha| \leq 1$.
% \end{description}
% \end{definition}
% We can verify that when $ f_0 \neq f $, the first-order distinguishability condition is satisfied.
To start with, we consider a scenario in which the pre-trained model $f_0$ is distinguishable from the prompt model $f$. Recall that given the density estimation rate in Proposition~\ref{theorem:ConvergenceRateofDensityEstimation}, we need to construct a loss function between the MLE $\widehat{G}_n$ and the ground-truth parameters $G_*$, which should be bounded by the Hellinger distance between the two corresponding densities, in order to capture the parameter estimation rates. Tailored to the distinguishable setting, we measure the discrepancy between two arbitrary parameters $G$ and $G_*$ in $\Xi$ via the loss 
\begin{align}
\label{applossdef:D1-loss}
    D_1(G, G_*)
    =
    |
    \exp(\tau)-\exp(\tau^*)
    |
    +
    \big(
    \exp(\tau)+\exp(\tau^*)
    \big)
    \|(\beta,\eta,\nu)-(\beta^*,\eta^*,\nu^*)\|.
\end{align}
We are ready to determine the convergence behavior of the MLE under distinguishable settings. 
% Futhermore when $f_0$ is not a Gaussian density as we can verify that
% \begin{proposition}
% \label{prop:d1_loss}
% If $f_0$ does not belong to the family of Gaussian densities, then following holds:

% assume that there exist real coefficients
% $c$, $\tau_\alpha$ for all 
% $\alpha \in \mathbb{N}$, with 
% $\alpha \leq 2$, such that
%     \begin{align*}
%         c_0 f_0(Y|h(X,\eta_0),\nu_0) + \sum_{\alpha \leq 2} \tau_\alpha 
%         \frac{\partial^{\alpha} f}{\partial h^{\alpha} }(Y  |  h(X,\eta), \nu) = 0
%         \quad \text{for all } x \in \mathcal{X},
%     \end{align*}
%     then $c = \tau_\alpha = 0$ for all $|\alpha| \leq 1$.

% \end{proposition}

% We can verify that when $ f_0 \neq f $, the first-order distinguishability condition is satisfied.
% Particularly when $h(X,(a,b))=\sigma(a^\top X+b)$ is in a scalar form, 
% % the convergence rate, 
% $f$ still satisfy the Proposition
% \ref{prop:d1_loss}.

\begin{theorem}
\label{thm:not_equal}
Suppose that the pre-trained model $f_0$ is distinguishable from the prompt model $f$.
For almost every \( x \in \mathcal{X} \), and for any \( \eta \in \mathbb{R}^q \), we assume that the Jacobian of the prompt expert function does not vanish, i.e., $\frac{\partial h}{\partial \eta}(x, \eta) \neq 0$. Then, there exists a positive constant $C_1$ that depends on $\Xi$ and $f_0$ such that the Hellinger lower bound $\mathbb{E}_X\left[ d_H(p_{ G}\left(\cdot|X),p_{G_*}(\cdot|X)\right) \right]\geq C_1D_1( G, G_*)$ holds for all parameters $G\in\Xi$. %\textcolor{violet}{ALE: I have not had the chance of going through the proofs, but I want to clarify one point. The term $ h(p_{G},p_{G_*})$ should depend on $x$, because $p_{G}$ and $p_{G_*}$ are conditional densities given $X$. Does this mean that I should interpret the previous bound as $ h(p_{G}(\cdot|x),p_{G_*}(\cdot|x))\geq C_1D_1( G, G_*)$ for almost all $x \in \mathcal{X}$?} 
As a result, we obtain
\begin{align}
    \label{eq:dis_bound_1}
    \sup_{G_*\in\Xi}\mathbb{E}_{p_{\Gs,n}} 
    \Big[
    |
    \exp(\widehat{\tau}_n)
    -
    \exp(\tau^*)
    |^2 
    \Big] 
    \lesssim \log(n)/n,\\
    \label{eq:dis_bound_2}
    \sup_{G_*\in\Xi}\mathbb{E}_{p_{\Gs,n}} 
    \Big[
    \exp^2(\tau^*)
    \Vert 
    (\widehat{\beta}_n, \widehat{\eta}_n, \widehat{\nu}_n)-(\beta^*,\eta^*,\nu^*) 
    \Vert^2 
    \Big] 
    \lesssim 
    \log(n)/n.
\end{align}
% \begin{align}
% \label{applossineq:D1-loss}
%     V(p_{ G},p_{G_*})\geq C_1D_1( G, G_*).
% \end{align}
\end{theorem}
The proof of Theorem \ref{thm:not_equal} is deferred to Appendix \ref{app_proof: d1_loss}. The bound in equation~\eqref{eq:dis_bound_1} reveals that the gating parameter estimator $\exp(\widehat{\tau}_n)$ converges to its ground-truth counterpart $\exp(\tau^*)$ at a rate of order $\widetilde{\mathcal{O}}(n^{-1/2})$. Analogously, looking at the bound in equation~\eqref{eq:dis_bound_2}, since the terms $\exp(\tau^*)$ cannot go to zero due to the compactness of the parameter space $\Xi$, it follows that the convergence rates of the parameter estimators $\widehat{\beta}_n$, $\widehat{\eta}_n$, and $\widehat{\nu}_n$ to $\beta^*$, $\eta^*$ and $\nu^*$ are also of order $\widetilde{\mathcal{O}}(n^{-1/2})$. Meanwhile, in the contaminated MoE with input-free gating in \cite{yan2025contaminated}, the estimation rates for prompt parameters $\eta^*,\nu^*$ are slower than $\widetilde{\mathcal{O}}(n^{-1/2})$ as they depend on the convergence rate of the gating parameter to zero. 
Therefore, replacing the input-free gating with the softmax gating in the contaminated MoE helps reduce the sample complexity of parameter estimation. 
% Particularly when $h(X,(a,b))=\sigma(a^\top X+b)$ is in a scalar form, the contaminated model still follow the MLE rate where $\eta=(a,b)$.
% the convergence rate, 
% $f$ still satisfy the Proposition
% \ref{prop:d1_loss}.

% \begin{theorem}[MLE rates]
%     \label{theorem:sigma-linear-f0-notGaussian}
% When $f_0$ is not a Gaussian density, we obtain the MLE convergence rates as follows:
% \begin{align*}
%     \sup_{G_*\in\Xi}\mathbb{E}_{p_{\Gs}} 
%     \Big[
%     |
%     \exp(\widehat{\tau}_n)
%     -
%     \exp(\tau^*)
%     |^2 
%     \Big] 
%     \lesssim \frac{\log n}{n},\\
%     \sup_{G_*\in\Xi}\mathbb{E}_{p_{\Gs}} 
%     \Big[
%     \exp^2(\tau^*)
%     \Vert 
%     (\widehat{\beta}_n, \widehat{\eta}_n, \widehat{\nu}_n)-(\beta^*,\eta^*,\nu^*) 
%     \Vert^2 
%     \Big] 
%     \lesssim 
%     \frac{\log n}{n}.
% \end{align*}
% \end{theorem}

Given the near-parametric convergence rates in Theorem~\ref{thm:not_equal}, it is natural to wonder if they are optimal. To answer this question in the affermative, below we derive minimax lower bounds. %for estimating the ground-truth parameters in Theorem~\ref{thm:lower-distinguish}.
% \begin{theorem}[Minimax lower bounds]
% % [Distinguishable settings]
% \label{thm:d1_minimax}
% % Assume that classes of densities $f_0$ and $f$ satisfy the conditions in Theorem~\ref{theorem:sigma-linear-f0-notGaussian}
% % or Theorem~\ref{theorem:sigma-linear-f0-Gaussian-varphi-nonlinear}
% % or Theorem~\ref{theorem:sigma-nonlinear-f0-notGaussian}
% % or Theorem~\ref{theorem:sigma-nonlinear-f0-Gaussian-varphi-linear}
% % . 
% %     % \colorbox{BurntOrange}{We need to clarify here, this distinguish setting is for thm1,3,4,5}
% % Then, 
% Under the setting of Theorem~\ref{theorem:sigma-linear-f0-notGaussian}, we have for any $0<r < 1$ that
% \begin{align*}
%     \inf_{(\llambdan,\overline{G}_n)\in \Xi }\sup_{(\lambda,G)\in \Xi }
%     \mathbb{E}_{p_{\lambda,G}} \Big( |\overline{\lambda}_n
%     -\lambda|^2 \Big) 
%     \gtrsim n^{-1/r},
%     \\
%     \inf_{(\llambdan,\overline{G}_n)\in \Xi }\sup_{(\lambda,G)\in \Xi }
%     \mathbb{E}_{p_{\lambda,G}} 
%     \Big( \lambda^2 \Vert \lGn-G \Vert^2 \Big) 
%     \gtrsim n^{-1/r}.
% \end{align*}
% % Here, the infimum is taken over all sequences of estimates $(\widehat{\lambda}_n, \widehat{G}_n)=(\widehat{\lambda}_n, \widehat{a}_n, \widehat{b}_n, \widehat{\nu}_n)$.
% \end{theorem}
\begin{theorem}
% [Distinguishable settings]
\label{thm:lower-distinguish}
% Assume that classes of densities $f_0$ and $f$ satisfy the conditions in Theorem~\ref{theorem:sigma-linear-f0-notGaussian}
% or Theorem~\ref{theorem:sigma-linear-f0-Gaussian-varphi-nonlinear}
% or Theorem~\ref{theorem:sigma-nonlinear-f0-notGaussian}
% or Theorem~\ref{theorem:sigma-nonlinear-f0-Gaussian-varphi-linear}
% . 
%     % \colorbox{BurntOrange}{We need to clarify here, this distinguish setting is for thm1,3,4,5}
% Then, 
If the pre-trained model $f_0$ is distinguishable from the prompt model $f$, then the following minimax lower bounds hold for any $0<r < 1$:
\begin{align*}
    \inf_{\overline{G}_n\in \Xi }\sup_{G\in \Xi }
    \mathbb{E}_{p_{G,n}} \Big( 
    |\exp(\overline{\tau}_n)
    -\exp(\tau)|^2 \Big) 
    \gtrsim n^{-1/r},
    \\
    \inf_{\overline{G}_n\in \Xi }\sup_{G\in \Xi }
    \mathbb{E}_{p_{G,n}} 
    \Big( \exp^2(\tau) \Vert (\overline{\beta}_n,\overline{\eta}_n,\overline{\nu}_n)-(\beta,\eta,\nu) \Vert^2 \Big) 
    \gtrsim n^{-1/r},
\end{align*}
where the infimum is over all estimators $\overline{G}_n:=(\overline{\beta}_n,\overline{\tau}_n,\overline{\eta}_n,\overline{\nu}_n)$ taking values in $\Xi$.
% Here, the infimum is taken over all sequences of estimates $(\widehat{\lambda}_n, \widehat{G}_n)=(\widehat{\lambda}_n, \widehat{a}_n, \widehat{b}_n, \widehat{\nu}_n)$.
\end{theorem}
The proof of Theorem \ref{thm:lower-distinguish} can be found in Appendix \ref{apppf:d1_minimax}.
The above minimax lower bounds imply that, under distinguishability, the convergence rates of the MLE, of order $\widetilde{\mathcal{O}}(n^{-1/2})$ is nearly minimax optimal, save for a logarithmic factor.

\subsection{Non-distinguishable Setting}
\label{sec:non-distinguishable}
We now turn to the much subtler case in which the distinguishability condition is violated. Since we assume a Gaussian prompt, it follows from Proposition~\ref{prop:lower-distinguish} that the pre-trained model $f_0$ necessarily belongs to the family of Gaussian densities. Furthermore, if the pre-trained and prompt model use the same expert function, i.e. $h_0=h$, then $f_0$ is not distinguishable from the prompt model $f$. We will thus focus on this challenging scenario. 

Under this setting, the prompt model may converge to the pre-trained model. In particular, if the pair of prompt parameters $(\eta^*,\nu^*)$ converge to the pair of pre-trained parameters $(\eta_0,\nu_0)$ as $n\to\infty$, then it follows that $f(\cdot|h(\cdot,\eta^*),\nu^*)$ converges to $f_0(\cdot|h(\cdot,\eta_0),\nu_0)$, indicating that the prompt learns the same expertise as the pre-trained model. Therefore, it becomes difficult for the gating network to assign higher weight to either the pre-trained model or the prompt than the other as they have similar expertise. As a result, one may expect the estimation rates of the gating parameters to be substantially slower. To formalize these setttings precisely, we need to pay more attention to the expert structure. 

It should be noted that a key step in obtaining the MLE convergence rates in Theorem~\ref{thm:not_equal} is to decompose the density discrepancy $p_{\widehat{G}_n}-p_{G_*}$ into a combination of linearly independent terms through an appropriate Taylor series expansion of the function $g(y  |  x; \beta, \eta, \nu) := \exp(\beta^\top x) \cdot f(y  |  h(x, \eta), \nu)$ with respect to its parameters $\beta,\eta,\nu$. This process involves, in particular, higher derivatives of the expert function $h$ with respect to $\eta$, which may not be algebraically independent. To ensure the linear independence of the terms in the Taylor expansion, we formulate a  \emph{strong identifiability} condition that is indeed sufficient for these purposes. %\textcolor{violet}{ALE: can we say something about necessity of this condition?}

\begin{definition}[Strong Identifiability]
\label{appendix_def:distinguish_linear_independent}    
The expert function $x \mapsto h(x,\eta)$ is \emph{strongly identifiable} if it is twice differentiable with respect to $\eta \in \mathbb{R}^q$ for almost all $x \in \mathcal{X}$, and if, for any fixed $\beta \in \mathbb{R}^d$ and $\eta \in \mathbb{R}^q$, each of the following sets of real-valued functions (of $x$) consists of linearly independent functions over $\mathbb{R}$. For notational simplicity, we write $h(\cdot)$ in place of $h(\cdot, \eta)$ below.

\begin{enumerate}
    \item \textbf{The first-order gating independence set:} 
    % $\left\{
    %     \dfrac{\partial h}{\partial \eta^{(u)}},\ 
    %     \exp(\beta^{\top} x)  \dfrac{\partial h}{\partial \eta^{(u)}}
    %     \right\}_{u \in [q]}$.
    \begin{equation*}
        \left\{
        \frac{\partial h}{\partial \eta^{(u)}},\ 
        \exp(\beta^{\top} x)  \frac{\partial h}{\partial \eta^{(u)}}
        \right\}_{u \in [q]}.
    \end{equation*}
    \item \textbf{The gradient product independence set:}
    \begin{equation*}
        \left\{
        1,\ x^{(w)},\ \exp(\beta^{\top} x),\ 
        \frac{\partial h}{\partial \eta^{(u)}}  \frac{\partial h}{\partial \eta^{(v)}},\ 
        \exp(\beta^{\top} x)  \frac{\partial h}{\partial \eta^{(u)}}  \frac{\partial h}{\partial \eta^{(v)}}
        \right\}_{u,v \in [q],\ w \in [d]}.
    \end{equation*}

    \item \textbf{The mixed and second-order independence set:} 
    \begin{equation*}
    \hspace{-2em}
        \left\{
        \frac{\partial h}{\partial \eta^{(u)}},\ 
        \exp(\beta^{\top} x)  \frac{\partial h}{\partial \eta^{(u)}},\ 
        x^{(w)}  \frac{\partial h}{\partial \eta^{(u)}},\ 
        \frac{\partial^2 h}{\partial \eta^{(u)} \partial \eta^{(v)}},\ 
        \exp(\beta^{\top} x)  \frac{\partial^2 h}{\partial \eta^{(u)} \partial \eta^{(v)}}
        \right\}_{u,v \in [q],\ w \in [d]}.
    \end{equation*}
\end{enumerate}
\end{definition}

Here, the \textit{First-order gating independence} condition guarantees that changes in $h$ with respect to $\eta$ remain distinguishable, even after modulation by the gating weights $\exp(\beta^\top X)$. This is a minimal requirement to ensure that the expert and gating mechanisms interact in a structurally non-degenerate way.
% The \textit{Gradient product independence} condition guarantees that the products of directional derivatives of $h$ remain distinguishable from gating terms 
The \textit{Gradient product independence} condition guarantees that the products of directional derivatives of $h$ are distinguishable from each other (even under modulation by gating terms) and cannot be expressed as a linear combination of basic functions. This prevents higher-order interactions among gradients from collapsing into lower-order structures.
Finally, the \textit{Mixed and second-order independence} condition is stronger than the first-order one. It rules out first-order interactions between expert and gating parameters of the form ${\partial h}/{\partial \eta^{(w)}} = x^{(w)} \cdot {\partial h}/{\partial \eta^{(v)}}$, which would imply ${\partial g}/{\partial \eta^{(w)}} = {\partial^2 g}/{(\partial \beta^{(w)} \partial \eta^{(v)})}$. It also requires that second-order derivatives remain linearly independent, even accounting for the effect of the gating function. This guarantees that both first- and second-order directional changes in $h$ convey distinct, non-redundant information, and that higher-order structure in $h$ cannot be reduced to or absorbed by lower-order terms. 
This is essential when handling second-order Taylor expansions of the model.

\textbf{Examples.} The expert functions $h(x, \eta) = \mathrm{GELU}(\eta^\top x)$, $h(x, \eta) = \mathrm{sigmoid}(\eta^\top x)$, and $h(x, \eta) = \tanh(\eta^\top x)$ satisfy the strong identifiability condition, as their nonlinearities avoid degeneracies. 
In contrast, $h(x, \eta) = \mathrm{ReLU}(\eta^\top x)$ fails the second-order independence condition, as the second-order derivatives vanish almost everywhere.
% In contrast, $h(x, \eta) = \mathrm{ReLU}(\eta^\top x)$ fails to meet the strong identifiability requirements, 
% since the second-order derivatives vanish.
% % due to vanishing second-order derivatives. 
Another failure case arises when $h(x, \eta) = \sigma(a^\top x + b)$, where $\eta = (a, b)$ and $\sigma$ is any scalar activation function.
This leads to $\partial h / \partial a = x \cdot \partial h / \partial b$, directly violating Condition~3.
To determine the convergence rates for the MLE in these settings, we construct the following loss function between parameters $G$ and $G^*$, carefully tailored to the non-distinguishable setting:
\begin{align*}
     D_2(G,G_*) :&=
   % \|\beta-\beta^*\|
   %  \exp(\tau+\tau^*)
   %  \Vert (\Delta\eta^*,\Delta\nu^*) \Vert
%     \\&
% \|\beta - \beta^*\| \cdot 
% \min\{ \exp(\tau), \exp(\tau^*) \} \cdot 
% \left( 
% \|(\Delta\eta, \Delta\nu)\| + \|(\Delta\eta^*, \Delta\nu^*)\| 
% \right)
%     \\&
%     \|\beta - \beta^*\| \cdot 
% \frac{1}{2} \left[
% \exp(\tau) \cdot \|(\Delta\eta, \Delta\nu)\| + 
% \exp(\tau^*) \cdot \|(\Delta\eta^*, \Delta\nu^*)\|
% \right]
% \\&
% \|\beta - \beta^*\| \cdot 
% \exp\left( {\tau + \tau^*} \right) \cdot 
% \sqrt{ 
% \|(\Delta\eta, \Delta\nu)\| \cdot 
% \|(\Delta\eta^*, \Delta\nu^*)\| 
% }
% \\&
% \|\beta - \beta^*\| \cdot 
% \left[
% \exp(\tau) \cdot \|(\Delta\eta, \Delta\nu)\|^2 +
% \exp(\tau^*) \cdot \|(\Delta\eta^*, \Delta\nu^*)\|^2
% \right]^{1/2}
    % \\
    % +&
    % \left|
    % \exp(\tau^*)-\exp(\tau)
    % \right|
    % \cdot
    % \Vert (\Delta\eta,\Delta\nu)\Vert
    % \cdot
    % \Vert (\Delta\eta^*,\Delta\nu^*)\Vert
    % \\&
    \exp(\tau)
    \|(\Delta\eta,\Delta\nu) \|^2
    +
    \exp(\tau^*)
    \|(\Delta\eta^*,\Delta\nu^*) \|^2
    \\&
    -\min\{\exp(\tau), \exp(\tau^*)\}
    \left(
    \|(\Delta\eta,\Delta\nu) \|^2
    +
    \|(\Delta\eta^*,\Delta\nu^*) \|^2
    \right)
    \\&
    +
    \big( 
    \exp(\tau)\Vert (\Delta\eta,\Delta\nu) \Vert
    +\exp(\tau^*)\Vert (\Delta\eta^*,\Delta\nu^*) \Vert
    \big)
    \times
    \Vert  (\beta,\eta,\nu)-(\beta^*,\eta^*,\nu^*) \Vert,
    % \\
    % +&
    % \exp(\tau+\tau^*)\cdot
    % \Vert (\Delta\eta,\Delta\nu)-(\Delta\eta^*,\Delta\nu^*)\Vert^2
\end{align*}
where we denote $(\Delta\eta,\Delta\nu)=(\eta-\eta_0,\nu-\nu_0)$ and $(\Delta\eta^*,\Delta\nu^*)=(\eta^*-\eta_0,\nu^*-\nu_0)$.
% {\color{red} Fanqi, please make the presentation of the loss function nicer. Define $\Delta\eta$}
\begin{theorem}
\label{thm:d2_mle_rate}
Suppose that $f_0$ belongs to the family of Gaussian densities and $h_0=h$. Then, there exists a positive constant $C_2$ that depends on $\Xi,\eta_0,\nu_0$ such that $\mathbb{E}_X\left[ d_H(p_{ G}\left(\cdot|X),p_{G_*}(\cdot|X)\right) \right]\geq C_2D_2( G, G_*)$ holds for all parameters $G$. %\textcolor{violet}{ALE: I have not had the chance of going through the proofs, but I want to clarify one point. The term $ h(p_{G},p_{G_*})$ should depend on $x$, because $p_{G}$ and $p_{G_*}$ are conditional densities given $X$. Does this mean that I should interpret the previous bound as $ h(p_{G}(\cdot|x),p_{G_*}(\cdot|x))\geq C_2D_2( G, G_*)$ for almost all $x \in \mathcal{X}$?}
As a result, we obtain
\begin{align}
    \label{eq:nondis_bound_1}
    \sup_{G_*\in\Xi(l_n)}&\mathbb{E}_{p_{\Gs,n}} 
    \Big[
    \Vert (\Delta\eta^*,\Delta\nu^*) \Vert^4
    \times
    |\exp(\widehat{\tau}_n)
    -\exp(\tau^*)|^2 
    \Big]
    \lesssim \log(n)/n,\\
    \label{eq:nondis_bound_2}
    \sup_{G_*\in\Xi(l_n)}&\mathbb{E}_{p_{\Gs,n}} 
    \Big[
    \exp^2(\tau^*) 
    \Vert (\Delta\eta^*,\Delta\nu^*) \Vert^2
    \times\Vert (\widehat{\beta}_n, \widehat{\eta}_n, \widehat{\nu}_n)-(\beta^*,\eta^*,\nu^*) \Vert^2 
    \Big] 
    \lesssim 
    \log(n)/n,
\end{align}
for any sequence $(l_n)_{n\geq 1}$ such that $l_n/\log n\to\infty$ as $n\to\infty$ where we denote
\begin{align*}
    \Xi(l_n) := \left\{ G = (\tau,\beta,\eta, \nu) \in \Xi : 
\frac{
    l_n
}{
    \min\limits_{1 \leq i \leq q,1 \leq j \leq d,} 
    \left\{ |\eta^{(i)}|^2, |\nu|^2, |\beta^{(j)}|^2 \right\} 
    \sqrt{n}
} \leq \exp(\tau)
\right\}.
\end{align*}
\end{theorem}
The proof of Theorem~\ref{thm:d2_mle_rate} is in Appendix~\ref{app_proof: d2_loss}. Note that under the setting of Theorem~\ref{thm:d2_mle_rate}, the softmax-contaminated MoE model is not identifiable, that is, the equation $p_{G}(y|x)=p_{G_*}(y|x)$ for almost all $(x,y)$ does not imply $G=G_*$. For that reason, we restrict the parameter space to the set $\Xi(l_n)$ to guarantee the consistency of the MLE. %Furthermore, there are also some changes in the convergence behavior of parameter estimation compared to the distinguishable setting in 
Compared to Theorem~\ref{thm:not_equal}, the above rates exhibit differ in several aspects.

\emph{(i)} From equation~\eqref{eq:nondis_bound_1}, we observe that the convergence rate of $\exp(\widehat{\tau}_n)$ to $\exp(\tau^*)$ becomes slower than the parametric order $\widetilde{\mathcal{O}} (n^{-1/2})$ as they depend on the vanishing rate of $(\Delta\eta^*,\Delta\nu^*)$ to zero. For example, if the pair of prompt parameters $(\eta^*,\nu^*)$ approach $(\eta_0,\nu_0)$ at the rate of $\widetilde{\mathcal{O}} (n^{-1/8})$, then the bound~\eqref{eq:nondis_bound_1} implies that $\exp(\widehat{\tau}_n)$ goes to $\exp(\tau^*)$ at the rate of $\widetilde{\mathcal{O}} (n^{-1/4})$.  This toy example is indeed confirmed by our numerical experiments in the next section.

\emph{(ii)} Likewise,  the convergence rates of the estimators $(\widehat{\beta}_n,\widehat{\eta}_n,\widehat{\nu}_n)$ %to $(\beta^*,\eta^*,\nu^*)$ 
are also impacted by the convergence rates of the prompt parameters and therefore slower than $\widetilde{\mathcal{O}} (n^{-1/2})$. For example, if $(\Delta\eta^*,\Delta\nu^*)$ go to zero at the rate of $\widetilde{\mathcal{O}} (n^{-1/8})$, then the bound~\eqref{eq:nondis_bound_2} indicates that $\widehat{\beta}_n,\widehat{\eta}_n,\widehat{\nu}_n$ converges to $\beta^*,\eta^*,\nu^*$ at the rate of $\widetilde{\mathcal{O}} (n^{-3/8})$, respectively. Again, in our numerical experiments below we empirically verify this behavior.

In our final result, whose proof can be found in Appendix~\ref{apppf:d2_minimax}, we show that the slower  converge rates for the MLE under non-distinguishability are in fact essentially minimax optimal. %Therefore, distinguishability is a crucial condition to ensure fast   

%in Theorem~\ref{thm:d2_minimax} whose proof can be found in Appendix~\ref{apppf:d2_minimax}, we derive minimax lower bounds for estimating the ground-truth parameters $\Gs$ under the non-distinguishable setting to validate that the MLE convergence rates presented in Theorem~\ref{thm:d2_mle_rate} are optimal. 
% \begin{theorem}[MLE rates]
% \label{theorem:sigma-nonlinear-f0-Gaussian-varphi-nonlinear}
% Suppose that $f_0$ is a Gaussian density, the expert $h$ is twice differentiable and $h=h_0$ almost everywhere.
% For any sequence $(l_n)_{n\geq 1}$, we denote
% \begin{align*}
%     \Xi(\ell_n) := \left\{ G = (\tau,\beta,\eta, \nu) \in \Xi : 
% \frac{
%     \ell_n
% }{
%     \min\limits_{1 \leq i \leq q,1 \leq j \leq d,} 
%     \left\{ |(\eta)_i|^2, (\nu)|^2, |(\beta)_j|^2 \right\} 
%     \sqrt{n}
% } \leq \exp(\tau)
% \right\}.
% \end{align*}
% Then, the followings hold for any sequence $(l_n)_{n\geq 1}$ such that $l_n/\log n\to\infty$ as $n\to\infty$:
% \begin{align*}
%     \sup_{G_*\in\Xi(l_n)}&\mathbb{E}_{p_{\Gs}} 
%     \Big[
%     \Vert (\Delta\eta,\Delta\nu) \Vert^4
%     |\exp(\widehat{\tau}_n)
%     -\exp(\tau^*)|^2 
%     \Big]
%     \lesssim \frac{\log n}{n},\\
%     \sup_{G_*\in\Xi(l_n)}&\mathbb{E}_{p_{\Gs}} 
%     \Big[
%     (\exp(\tau^*))^2 
%     \Vert (\Delta\eta^*,\Delta\nu^*) \Vert^2
%     \times\Vert (\widehat{\beta}_n, \widehat{\eta}_n, \widehat{\nu}_n)-(\beta^*,\eta^*,\nu^*) \Vert^2 
%     \Big] 
%     \lesssim 
%     \frac{\log n}{n}.
% \end{align*}
% \end{theorem}

\begin{theorem}
\label{thm:d2_minimax}
    Suppose that $f_0$ belongs to the family of Gaussian densities and $h_0=h$. Then, the minimax lower bounds
% \begin{align*}
%     \Xi(l_n) := \left\{ G = (\tau,\beta,\eta, \nu) \in \Xi : 
% \frac{
%     l_n
% }{
%     \min\limits_{1 \leq i \leq q,1 \leq j \leq d,} 
%     \left\{ |\eta^{(i)}|^2, (\nu)|^2, |\beta^{(j)}|^2 \right\} 
%     \sqrt{n}
% } \leq \exp(\tau)
% \right\}.
% \end{align*}
\begin{align*}
    % &\inf_{\overline{G}_n\in \Xi }\sup_{G\in \Xi(l_n)  }
    % \mathbb{E}_{p_{G}} \Big[ 
    % \exp(\tau)^2
    % \Vert (\Delta\eta,\Delta\nu) \Vert^2    
    % \times \|\overline{\beta}_n-\beta\|^2
    % \Big] 
    % \gtrsim n^{-1/r},
    % \\
    &\inf_{\overline{G}_n }\sup_{G\in \Xi(l_n)  }
    \mathbb{E}_{p_{G,n}} \Big[ 
    \Vert (\Delta\eta,\Delta\nu) \Vert^4    
    \times \|\exp(\overline{\tau}_n)-\exp(\tau)\|^2
    \Big] 
    \gtrsim n^{-1/r},
    \\
    &\inf_{\overline{G}_n}\sup_{G\in \Xi(l_n) }
    \mathbb{E}_{p_{G,n}} 
    \Big[ \exp^2(\tau) 
    \Vert(\Delta\eta,\Delta\nu) \Vert^2
    \times\Vert (\overline{\beta}_n, \overline{\eta}_n, \overline{\nu}_n)-(\beta,\eta,\nu) \Vert^2 \Big] 
    \gtrsim n^{-1/r},
\end{align*}
hold for any sequence $(l_n)_{n\geq 1}$ and any $0<r < 1$, , where the infimum is over all estimators $\overline{G}_n$ taking values in $\Xi$.
\end{theorem}

\subsection{Practical Implications}
\label{sec:practical_implications}

There are two important practical implications for the design of a contaminated MoE model from our theoretical results.

\emph{1. Softmax gating is more sample-efficient than input-free gating.} We observe that softmax gating yields faster convergence rates of prompt parameter estimation in contaminated MoE than input-free gating in \cite{yan2025contaminated}. In particular, when using input-free gating, Table~\ref{table:comparison} reveals that the rates for estimating expert parameters and variance depend on the convergence rate of the gating parameter to zero. By contrast, when using softmax gating, estimation rates for expert parameters and variance become significantly faster as the previous rate dependence disappears. Therefore, our theories encourage the use of softmax gating over input-free gating when tuning contaminated-MoE-based models.

\emph{2. Prompt models should have different expertise from pre-trained models.} It can be seen from Table~\ref{table:comparison} that when the prompt model acquires overlapping knowledge with the pre-trained model (non-distinguishable setting), the convergence rates of parameter estimation are slower than when these models have distinct knowledge (distinguishable setting). Thus, our theories advocate using prompt models with different expertise from the pre-trained model.

\begin{table*}[!ht]
\caption{Comparison of parameter estimation rates in input-free-contaminated MoE \cite{yan2025contaminated} and softmax-contaminated MoE (Ours). Below, we consider gating parameters $\exp(\beta^*_0)$, expert parameters $\eta^*$, and variance $\nu^*$. In addition, $\lambda^*$ denotes the constant weight in input-free-contaminated MoE.}
\textbf{}\\
\centering
\begin{tabular}{ | c | c |c|} 
\hline
\multicolumn{3}{|c|}{\textbf{Distinguishable Setting}}\\
\hline
\textbf{}  & Gating parameters &  Expert parameters and Variance\\
\hline 
{\bf Input-free gating \cite{yan2025contaminated}}  & $\widetilde{\mathcal{O}}(n^{-1/2})$ &$\widetilde{\mathcal{O}}(n^{-1/2}(\lambda^*)^{-1})$  \\
\hline
{\bf Softmax gating (Ours)}  &$\widetilde{\mathcal{O}}(n^{-1/2})$  &$\widetilde{\mathcal{O}}(n^{-1/2})$ \\
\hline
\multicolumn{3}{|c|}{\textbf{Non-distinguishable Setting}}\\
\hline
\textbf{}  & Gating parameters &  Expert parameters and Variance\\
\hline 
{\bf Input-free gating \cite{yan2025contaminated}}  & $\widetilde{\mathcal{O}} (n^{-\frac{1}{2}}\cdot\Vert(\Delta\eta^*,\Delta\nu^*)\Vert^{-2})$ & $\widetilde{\mathcal{O}} (n^{-\frac{1}{2}}\cdot\Vert(\Delta\eta^*,\Delta\nu^*)\Vert^{-1}(\lambda^*)^{-1})$   \\
\hline
{\bf Softmax gating (Ours)}  &$\widetilde{\mathcal{O}} (n^{-\frac{1}{2}}\cdot\Vert(\Delta\eta^*,\Delta\nu^*)\Vert^{-2})$ & $\widetilde{\mathcal{O}} (n^{-\frac{1}{2}}\cdot\Vert(\Delta\eta^*,\Delta\nu^*)\Vert^{-1})$ \\
\hline
\end{tabular}
\label{table:comparison}
\end{table*}

\section{Numerical Experiments}
\label{sec:experiments}

% \begin{figure*}[t]
%     \centering
%     % \subfloat[\textbf{Case (i):} $\lambda^* = 0.5$.\label{fig:thm2-fixed}]{
%         \includegraphics[width=1\textwidth]{figs/dis_laplace.pdf}
%         \label{fig:dis_laplace}
%     % }
% % \begin{figure}
% %         \centering
% %         \includegraphics[width=1\linewidth]{figs/dis_laplace.pdf}
% %         \caption{Enter Caption}
% %         \label{fig:enter-label}
% %     \end{figure}
        
%     % \subfloat[\textbf{Case (ii):} $\lambda^* = 0.5~n^{-1/4}$.\label{fig:thm2-var}]{
%     %     \includegraphics[width=1\textwidth]{aistats2025/figures/thm2_var_lambda.pdf}
%     % }
%     \caption{Log-log graphs depicting the empirical convergence rates of the MLE $(\widehat{\beta}_n,\widehat{\tau}_n,\widehat{\eta}_n,\widehat{\nu}_n)$ to the ground-truth values $(\beta^*,\tau^*,\eta^*,\nu^*)$. 
%     %Above, the x-axis stands for the sample size $n$, while the y-axis represents the parameter estimation error. 
%     The blue lines display the parameter estimation errors, while the red dashed dotted lines are the fitted lines, highlighting the empirical MLE convergence rates. 
%     % Figure~\ref{} 
%     % % and Figure~\ref{fig:thm2-var} 
%     % illustrate result 
%     % % for the cases 
%     % when $\tau^* = 0.5$.
%     % % and $\lambdas = 0.5~n^{-1/4}$, respectively.
%     }
%     % \label{fig:thm1_experiments}
% \end{figure*}

In this section, we present several numerical experiments to verify our theoretical findings.

% \textbf{Experimental setup.} %(i) 
% Recall that, in the setting of Theorem~\ref{thm:not_equal}, the
% pre-trained model $f_0$ does not belong to the Gaussian density family. Thus we let $f_0$ be the density of a Laplace distribution, with mean function $h_0(x, \eta_0) = \mathrm{tanh}(\eta_0^{\top} x)$ and variance $\nu_0$ .
% Here, $\eta_0$ is a $d$-dimension vector $e_1 := (1, 0, \ldots, 0)$ and $\nu_0 = 0.001$. 
% Meanwhile, the prompt $f$ is formulated as a Gaussian density with the mean function also given by $\mathrm{tanh}$, though with a different patameter $\eta^*$ -- i.e.  $h(x, \eta^*) = \mathrm{tanh}((\etas)^\top x)$ -- and variance $\nu^*$.

\textbf{Experimental setup.}
Recall that, in the distinguishable setting, 
% of Theorem~\ref{thm:not_equal}, 
the
pre-trained model $f_0$ does not belong to the Gaussian density family. Thus, we let $f_0$ be the density of a Laplace distribution, with mean function $h_0(x, \eta_0) = \tanh(\eta_0^{\top} x)$ and variance $\nu_0$.
Here, $\eta_0$ is a $d$-dimensional vector defined as $e_1 := (1, 0, \ldots, 0)$, and $\nu_0 = 0.001$.
Meanwhile, the prompt $f$ is formulated as a Gaussian density, with the same $\tanh$ mean function but a different parameter $\eta^*$—i.e., $h(x, \eta^*) = \tanh((\eta^*)^\top x)$—and variance $\nu^*$.

%(ii) 
% On the other hand, in Theorem~\ref{thm:d2_mle_rate}, both $f$ and $f_0$  belong to the Gaussian density family and $h$ and $h_0$ are expert functions of the same form (albeit parametrized by different values of $\eta_0$ and $\eta^*$). Just like in the previous case, we let the expert function be the $\tanh$ function: so, in the pre-trained model the expert is $h(x,\eta_0)=\tanh((\eta_0)^\top x)$ and in the prompt model it is
% $h(x,\eta^*)=\tanh((\eta^*)^\top x)$.

On the other hand, in the non-distinguishable setting
% Theorem~\ref{thm:d2_mle_rate}
, both $f$ and $f_0$ belong to the Gaussian density family, and $h$ and $h_0$ are expert functions of the same form (albeit parameterized by different values of $\eta_0$ and $\eta^*$). As in the previous case, we let the expert function be the $\tanh$ function: in the pre-trained model, the expert is $h(x, \eta_0) = \tanh(\eta_0^\top x)$, and in the prompt model, it is $h(x, \eta^*) = \tanh((\eta^*)^\top x)$.

\textbf{Synthetic data generation.} We create synthetic datasets following the model outlined in equation~\eqref{eq:contaminated_pretrain_model_general}. Specifically, we generate data pairs $\{(X_i, Y_i)\}_{i=1}^n\in\mathcal{X}\times\mathcal{Y} \subset \mathbb{R}^d \times \mathbb{R}$ by 
first drawing each covariate $X_i$ independently from a standard Gaussian distribution, for $i = 1, \ldots, n$, and consistently set $d = 8$ across all trials.
% initially drawing the covariate  $X_i$ from a standard Gaussian distribution for $i = 1, \ldots, n$, consistently setting $d=8$ across all trials. 
The responses $Y_i$ are drawn from the density $p_{\Gs}(y|x)$, where $\Gs =(\beta^*,\tau^*,\eta^*,\nu^*)$:

(a) 
In the distinguishable setting,
we let $\beta^*=1/\sqrt{d}\cdot\mathbf{1}_d,\tau^*=1$, $\eta^*=-e_1= - \eta_0 $ and $\nu^*= \nu_0 = 0.001$. 

(b) 
% In the non-distinguishable setting, we consider the following two cases to observe the MLE convergence behavior when either $\eta^*$ or $\nu^*$ varies with $n$:
% In the first, $\eta^*$ is an $\mathcal{O}(n^{-1/8})$ perturbation of $\eta_0$ but $\nu^* = \nu_0$. In the second case, $\eta^* = -\eta_0$ but $\nu^*$ is an $\mathcal{O}(n^{-1/8})$ perturbation of $\nu_0$. 
In the non-distinguishable setting, we examine two cases to study the MLE convergence behavior as either $\eta^*$ or $\nu^*$ varies with $n$: in the first, $\eta^*$ is an $\mathcal{O}(n^{-1/8})$ perturbation of $\eta_0$ with $\nu^*$ fixed at $\nu_0$; in the second, $\eta^* = -\eta_0$ while $\nu^*$ is perturbed around $\nu_0$ at the same rate.
In detail, we set:
\begin{itemize}
    \item[(i)] In the first case, $\beta^* = {1}/{\sqrt{d}} \cdot \mathbf{1}_d$, $\tau^* = 1$, $\eta^* = e_1 (1 + n^{-1/8}) = \eta_0 (1 + n^{-1/8})$, and $\nu^* = \nu_0 = 0.001$.

    \item[(ii)] In the second case, $\beta^* = {1}/{\sqrt{d}} \cdot \mathbf{1}_d$, $\tau^* = 1$, $\eta^* = -e_1 = -\eta_0$, and $\nu^* = 0.001 (1 + n^{-1/8}) = \nu_0 (1 + n^{-1/8})$.
\end{itemize}

% (b) 
% % In the settings of Theorem~\ref{thm:d2_mle_rate}, 
% In the non-distinguishable setting,
% we consider two following cases to observe the MLE convergence behavior when either $\eta^*$ or $\nu^*$ varies with n:
% In the first, $\eta^*$ is an $\mathcal{O}(n^{-1/8})$ perturbation of $\eta_0$ but $\nu^*=\nu_0$. In the second case, $\eta^* =  - \eta_0$ but $\nu^*$ is an $\mathcal{O}(n^{-1/8})$ perturbation of $\nu_0$. In detail, we set
% \begin{itemize}
%     \item[(i)] $\beta^*=1/\sqrt{d}\cdot\mathbf{1}_d, \tau^*=1$, $\etas=e_1 (1 + n^{-1/8}) = \eta_0 (1 + n^{-1/8})$ and $\nus=\nu_0= 0.001$ in one case and

%     \item[(ii)] $\beta^*=1/\sqrt{d}\cdot\mathbf{1}_d, \tau^*=1$, $\etas=-e_1 = - \eta_0$ and $\nus=0.001 (1 + n^{-1/8}) = \nu_0 (1 + n^{-1/8})$ in the other.
% \end{itemize}
% Here, we consider the following case of the true parameters $\Gs =(\beta^*,\tau^*,\eta^*,\nu^*)$ to observe 
% % the difference in 
% the MLE convergence behavior when 
% % $\exp(\tau^*)$ is fixed:
% % versus when it varies with $n$: 
% $\tau^*=0.5,\beta^*=\mathbf{1}_d,\eta^*=1,\nu^*=0.01$.

% % (ii) $\lambdas=0.5~n^{-1/4},\as=\mathbf{1}_d,\bs=1,\nus=0.01$.

\textbf{Training procedure.} We conduct $40$ experiments and, for each of them, consider 20 different sample sizes $n$, ranging from $10^3$ to $10^5$.
In computing the MLEs, 
% the initialization of the parameters is set relatively close to the true values of the parameters 
the initialization is set relatively close to the true parameter values
to mitigate potential optimization instabilities. 
We use an EM algorithm~\citep{Jordan-1994} to compute the MLE, employing an off-the-shelf BFGS optimizer for the M-step due to the absence of a universal closed-form solution. All the numerical experiments are performed on 
% the Macbook Air Apple M4 Chip.
a MacBook Air with an Apple M4 chip.

% \textbf{Results.}
% The experiment results are presented in Figure \ref{fig:thm1_experiments} and Figure \ref{fig:thm3_experiments},
% where in the figure the x-axis displays varying sample sizes $n$, while on the y-axis we plot the parameter estimation error. 
% Now we will read the figures:

% (a) Figure~\ref{fig:thm1_experiments} displays the results for Theorem \ref{thm:not_equal}.
% We could see the convergence rates of 
% $(\widehat{\beta}_n,\widehat{\tau}_n,\widehat{\eta}_n,\widehat{\nu}_n)$
% is $\mathcal{O}(n^{-0.45}), \mathcal{O}(n^{-0.52}), \mathcal{O}(n^{-0.50}, \mathcal{O}(n^{-0.54}))$,
% respectively.
% All the MLE share similar rates of orders roughly $\mathcal{O}(n^{-1/2})$, in agreement with the theoretical rates predicted by Theorem~\ref{thm:not_equal}.

\textbf{Results.}
The experimental results are presented in Figure~\ref{fig:thm1_experiments} and Figure~\ref{fig:thm3_experiments},
where the x-axis displays varying sample sizes $n$, and the y-axis shows the parameter estimation error.
We now present a detailed analysis of the results shown in each figure:
% We now describe the findings in each figure:

(a) Figure~\ref{fig:thm1_experiments} displays the results for Theorem~\ref{thm:not_equal}.
We observe that the convergence rates of 
$(\widehat{\beta}_n,\widehat{\tau}_n,\widehat{\eta}_n,\widehat{\nu}_n)$
are $\mathcal{O}(n^{-0.45}), \mathcal{O}(n^{-0.52}), \mathcal{O}(n^{-0.50}), \mathcal{O}(n^{-0.54})$,
respectively, aligning with the theoretical rates of order $\mathcal{O}(n^{-1/2})$ in Theorem~\ref{thm:not_equal}.

% $\exp(\widehat{\tau}_n)$ is $\mathcal{O}(n^{-0.45})$, close to $\mathcal{O}(n^{-1/2})$.

(b) On the other hand, Figure~\ref{fig:thm3_experiments} illustrates the parameter estimation errors for the simulations conducted in the non-distinguishable setting as Theorem~\ref{thm:d2_mle_rate}.

\begin{itemize}
    \item[(i)]
    In the first case, $\eta^*$ converges to $\eta_0$ at the rate of $\mathcal{O}(n^{-1/8})$,
    while $\nu^*$ remains fixed,
    % Figure~\ref{fig:non_dis_eta} shows that the convergence rate of $\exp(\widehat{\tau}_n)$ to $\exp(\tau^*)$ is of order $\mathcal{O}(n^{-0.23})$, close to $\mathcal{O}(n^{-1/4})$.
    Figure~\ref{fig:non_dis_eta} shows that the convergence rate of $\exp(\widehat{\tau}_n)$ to $\exp(\tau^*)$ is $\mathcal{O}(n^{-0.23})$, which is consistent with the expected rate of $\mathcal{O}(n^{-1/4})$.
    % The convergence rates for $\widehat{\beta}_n$, $\widehat{\eta}_n$, and $\widehat{\nu}_n$ are $\mathcal{O}(n^{-0.37})$, $\mathcal{O}(n^{-0.39})$, and $\mathcal{O}(n^{-0.35})$, respectively, all close to $\mathcal{O}(n^{-0.375})$, as they hinge on the vanishing rate $\mathcal{O}(n^{-3/8})$.
    The convergence rates for $\widehat{\beta}_n$, $\widehat{\eta}_n$, and $\widehat{\nu}_n$ are  $\mathcal{O}(n^{-0.37})$, $\mathcal{O}(n^{-0.39})$, and $\mathcal{O}(n^{-0.35})$, respectively, all of which are approximately $\mathcal{O}(n^{-0.375})$, as they hinge on the vanishing rate $\mathcal{O}(n^{-3/8})$.
    These empirical rates are consistent with the theoretical rates in Theorem~\ref{thm:d2_mle_rate}.

    \item[(ii)]
    % In case (ii), where $\eta^*$ is fixed and $\nu^*$ converges to $\nu_0$ at the rate of $\mathcal{O}(n^{-1/8})$,
    In the alternative setting, $\eta^*$ is held fixed, while $\nu^*$ converges to $\nu_0$ at the rate of $\mathcal{O}(n^{-1/8})$.
    Figure~\ref{fig:non_dis_nu} reveals that the convergence rate of $\exp(\widehat{\tau}_n)$ to $\exp(\tau^*)$ is of order $\mathcal{O}(n^{-0.22})$, again close to $\mathcal{O}(n^{-1/4})$.
    Meanwhile, the MLEs $\widehat{\beta}_n$, $\widehat{\eta}_n$, and $\widehat{\nu}_n$ still empirically converge to $\beta^*$, $\eta^*$, and $\nu^*$ at rates of $\mathcal{O}(n^{-0.39})$, $\mathcal{O}(n^{-0.37})$, and $\mathcal{O}(n^{-0.39})$, respectively,
    which align well with the theoretical rates $\widetilde{\mathcal{O}}(n^{-3/8})$.
    % This observation aligns well with the theoretical convergence rate in Theorem~\ref{thm:d2_mle_rate}.
    This observation is consistent with the theoretical convergence rates in Theorem~\ref{thm:d2_mle_rate}.
\end{itemize}

\begin{figure*}[t]
    \centering
    % \subfloat[\textbf{Case (i):} $\lambda^* = 0.5$.\label{fig:thm2-fixed}]{
        \includegraphics[width=1\textwidth]{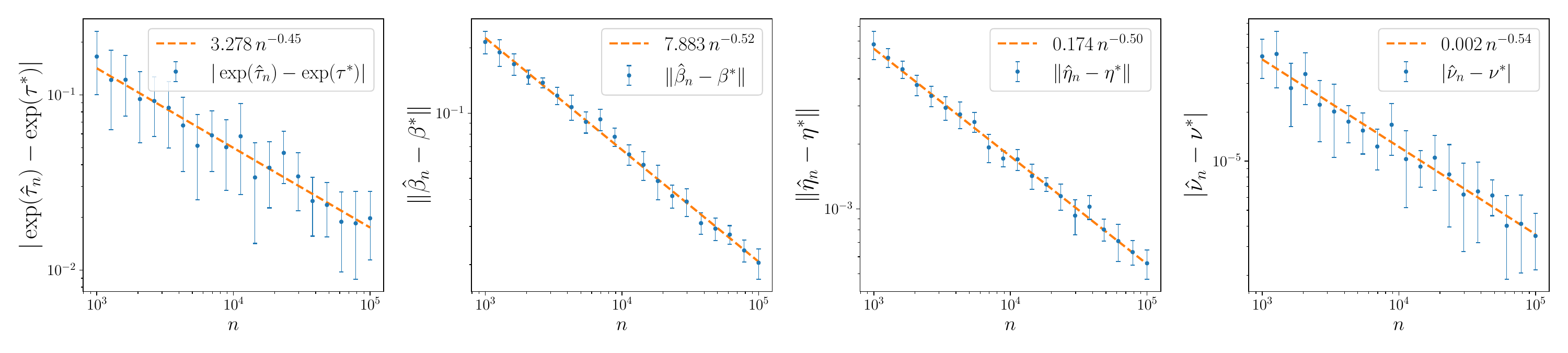}
        % \label{fig:dis_laplace}
    % }        
    % \subfloat[\textbf{Case (ii):} $\lambda^* = 0.5~n^{-1/4}$.\label{fig:thm2-var}]{
    %     \includegraphics[width=1\textwidth]{aistats2025/figures/thm2_var_lambda.pdf}
    % }
    % \vspace{- 1 em}
    \caption{ 
    % {(\textbf{
    % \textcolor{violet}{ALE: this is better, thank you, though the axes are still hard to read. not sure there is enough space} Settings of Theorem~\ref{thm:not_equal}:}
    % $f_0$ is Laplace distribution.)}
    (\textbf{Distinguishable Setting:} $f_0$ is the density of a Laplace distribution.)
    Log-log graphs depicting the empirical convergence rates of the MLE $(\widehat{\beta}_n,\widehat{\tau}_n,\widehat{\eta}_n,\widehat{\nu}_n)$ to the ground-truth values $(\beta^*,\tau^*,\eta^*,\nu^*)$. 
    %Above, the x-axis stands for the sample size $n$, while the y-axis represents the parameter estimation error. 
    The blue lines display the parameter estimation errors, while the orange dashed dotted lines are the fitted lines, highlighting the empirical MLE convergence rates. 
    % Figure~\ref{} 
    % % and Figure~\ref{fig:thm2-var} 
    % illustrate result 
    % % for the cases 
    % when $\tau^* = 0.5$.
    % % and $\lambdas = 0.5~n^{-1/4}$, respectively.
    }
    \label{fig:thm1_experiments}
    % \vspace{-0.5em}
\end{figure*}

    \begin{figure*}[h]
        \centering
        \subfloat[\textbf{Case (i):} $\betas=1/\sqrt{d}~\mathbf{1}_d, \taus = 1, \etasn = e_1(1 + n^{-1/8}), \nus = 0.001$.
        \label{fig:non_dis_eta}]{
            \includegraphics[width=1\textwidth]{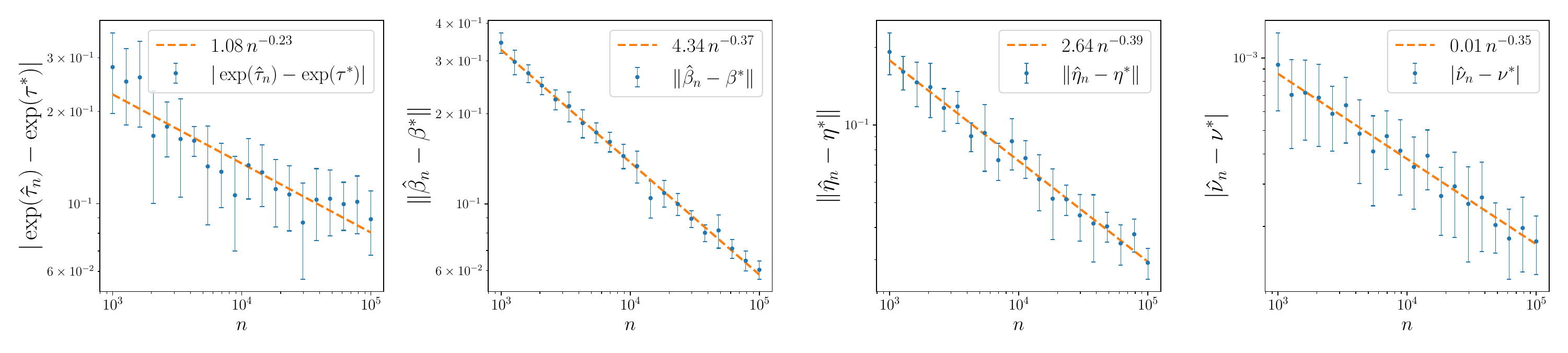}
        }
        
        \subfloat[\textbf{Case (ii):} $\betas=1/\sqrt{d}~\mathbf{1}_d , \taus = 1, \etasn = -e_1, \nusn = 0.001(1 + n^{-1/8})$.
        \label{fig:non_dis_nu}]
        {
        % \vspace{-0.5em}
        \includegraphics[width=1\textwidth]{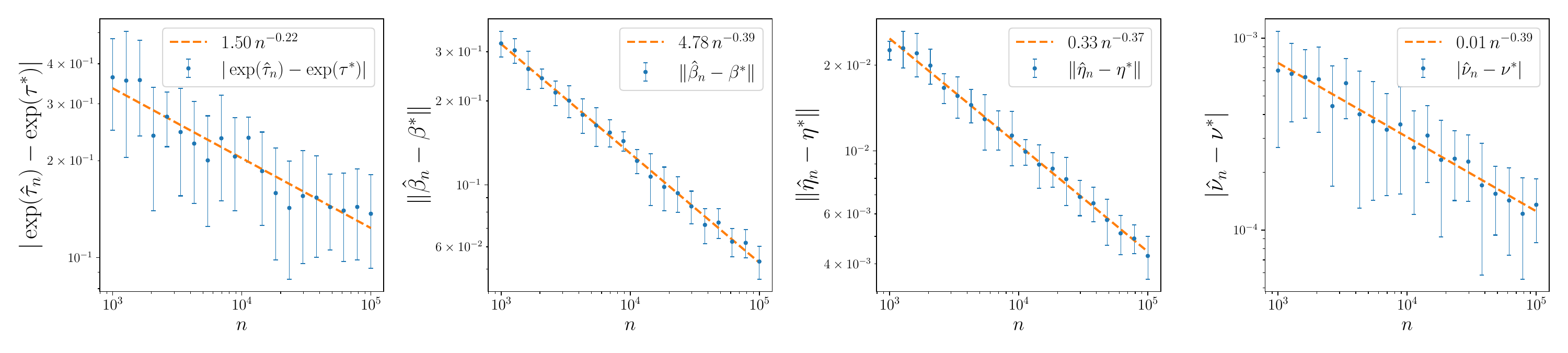}
        }
        \vspace{0.5em}
        \caption{
        % {(\textbf{\textcolor{violet}{ALE: this is better, thank you, though the axes are still hard to read. not sure there is enough space}Theorem~\ref{thm:d2_mle_rate}:} $f_0$ is Gaussian.)} 
        (\textbf{Non-distinguishable Setting:} $f_0$ is a Gaussian density.)
        Log-log graphs depicting the empirical convergence rates of the MLE $(\widehat{\beta}_n,\widehat{\tau}_n,\widehat{\eta}_n,\widehat{\nu}_n)$ to the ground-truth values $(\beta^*,\tau^*,\eta^*,\nu^*)$.
        %Above, the x-axis stands for the sample size $n$, while the y-axis represents the parameter estimation error. 
        The blue lines display the parameter estimation errors, while the orange dashed dotted lines are the fitted lines, highlighting the empirical MLE convergence rates. Figure~\ref{fig:non_dis_eta} and Figure~\ref{fig:non_dis_nu} illustrates results for Case (i) and Case (ii), respectively.
        }
        \label{fig:thm3_experiments}
        % \vspace{- 0.5em}
    \end{figure*}

\section{Conclusion}
\label{sec:conclusion}
In this paper, we characterize the convergence behavior of maximum likelihood estimators for parameters in the softmax-contaminated MoE model formulated as a mixture of a frozen pre-trained model and a trainable prompt model. 
To capture the challenge in which the prompt model admits the same expertise as the pre-trained model, we propose a novel analytic distinguishability condition and divide our analysis based on that condition. When the distinguishability condition is satisfied, we obtain minimax optimal parameter estimation rates of parametric order in the sample size, which are faster than those under the contaminated MoE with input-free gating. 
Conversely, when the distinguishability condition is violated, these rates become substantially slower than the parametric rates as they hinge on the convergence rates of prompt parameters to pre-trained parameters. 

Based on our theoretical analysis, we make the following observations. First, the softmax gating helps 
% seems 
to improve the sample efficiency for estimating the parameters in the contaminated MoE compared to the input-free gating. Second, the convergence rates for parameter estimation will be negatively affected if the prompt model acquires overlapping knowledge with the pre-trained model, thereby increasing the sample complexity of parameter estimation. 

In future work, we plan to consider a more challenging setting of the contaminated MoE where the pre-trained model is fine-tuned by multiple prompt models rather than a single prompt as in the current setting. 
% As future work, we aim to consider a more challenging variant of the contaminated MoE model, in which the pre-trained model is fine-tuned by multiple prompt models instead of a single one, as considered in this paper.
Furthermore, we can also generalize the analysis to the scenario where the prompt models belong to various families of distributions, rather than being restricted to Gaussian distributions.
% instead of only Gaussian distributions.

% Firstly, the current contaminated MoE model considers only one prompt. It is of practical importance to extend the current analysis to the setting where several prompts are incorporated. The convergence rates for parameter estimation under that setting will provide us new insights into reduce the prompt length in fine-tuning large-scale AI models based on merging the close experts. Secondly, the current model assumes that the density $f$ belongs to the Gaussian family. We plan to generalize the current theoretical analysis for general family of density functions $f$. 
%Finally, the gating function (mixing proportion) considered in this paper is input-independent. Thus, we can examine popular input-dependent gating functions, namely softmax gating \citep{shazeer2017topk,nguyen2023demystifying} and sigmoid gating \citep{csordas2023approximating,nguyen2024sigmoid}, in future works.

\bibliography{references}
\bibliographystyle{abbrv}

%\input{checklist}

%%%%%%%%%%%%%%%%%%%%%%%%%%%%%%%%%%%%%%%%%%%%%%%%%%%%%%%%%%%%
\begin{comment}
    
\end{comment}
\newpage

\appendix
\centering
\textbf{\Large{Supplementary Material for
``On Minimax Estimation of Parameters in Softmax-Contaminated Mixture of Experts''}}

\justifying
\setlength{\parindent}{0pt}
\textbf{}\\
\noindent
In this supplementary material, we provide the theoretical proofs omitted from the main text. Appendix~\ref{appsec:main_results} presents the proofs of our main results, including the theorems on convergence rates for parameter estimation and the minimax lower bounds stated in Section~\ref{sec:theory}. Proofs of auxiliary results concerning the fundamental properties of the softmax-contaminated MoE model, introduced in Section~\ref{sec:preliminaries}, are deferred to Appendix~\ref{appendix:ProofsforAuxiliaryResults}.

\section{Proof of Main Results}
\label{appsec:main_results}

% \section{Distinguishable Setting Proofs}
% \label{appsec:distinguish}
% In this section, we will provide the proofs for the MLE rate and minimax lower bounds theorems discussed in Sections \ref{sec:theory}, for distinguishable settings and non-distinguishable settings.

In this section, we present the proofs of the MLE rate theorem and the minimax lower bound theorem from Section~\ref{sec:theory}, covering both distinguishable and non-distinguishable settings.

% \subsection{$D_1$ loss}
\subsection{Proof of Theorem \ref{thm:not_equal}}
\label{app_subsec:d1_rate}
% $    D_1(G,G_*)
%     =
%     |
%     \exp(\tau)-\exp(\tau^*)
%     |
%     +
%     \big(
%     \exp(\tau)+\exp(\tau^*)
%     \big)
%     \|(\beta,\eta,\nu)-(\beta^*,\eta^*,\nu^*)\|.$

We begin by proving Theorem~\ref{thm:not_equal} under the distinguishable setting.
\begin{proof}[Proof of Theorem \ref{thm:not_equal}]
\label{app_proof: d1_loss}
Let $\overline{G}=(\Bar{\beta},\Bar{\tau},\Bar{\eta},\Bar{\nu})$ , we need to demonstrate that
 \begin{align*}
     \lim_{\varepsilon\to 0}\inf_{G,G_*}\left\{\frac{\bbE_X[d_V(p_G(\cdot|X),p_{ G_*}(\cdot|X))]}{D_1(G,G_*)}: 
     D_1(G,\overline{G})\vee D_1(G_*,\overline{G})\leq\varepsilon\right\}>0.
 \end{align*}
Using the argument with Fatou's lemma as in Theorem 3.1, \cite{Ho-Nguyen-EJS-16} , it is sufficient to show that
\begin{align*}
     \lim_{\varepsilon\to 0}\inf_{G,G_*}\left\{\frac{\|p_{ G}-p_{ G_*}\|_{\infty}}{D_1(G,G_*)}:
     D_1(G,\overline{G})\vee D_1(G_*,\overline{G})\leq\varepsilon\right\}>0.
 \end{align*}
 Assume by contrary that the above claim is not true. 
Then, there exist two sequences $G_n=(\beta_n,\tau_n,\eta_n,\nu_{n})$ and $G_{*,n}=(\beta_n^*,\tau^*_n,\eta_n^*,\nu^*_{n})$,
% , and two sequences of gating parameters 
% $\lambda_n=(\beta_n,\tau_n)$ and $\lambdasn=(\beta_n^*,\tau^*_n) \in \Gamma$
% , a compact set,
such that when $n$ tends to infinity, we get
 \begin{align*}
     \begin{cases}
        % D_{1 n}:=
        D_1(G_n,\overline{G}) \to 0 ,\\
        % D^*_{1 n}:=
        D_1(G_{*,n},\overline{G}) \to 0 ,\\
        {\|p_{ G_n}-p_{ G_{*,n}}\|_{\infty}}/D_1(G_n,G_{*,n})\to 0.
     \end{cases}
 \end{align*}
 In this proof, we will take into account only the most challenging setting of $(\betan, \eta_n, \nu_n)$ and $(\betasn, \eta^*_n, \nu^*_n)$ when they converge to the same limit point $(\beta',\eta',\nu')$, where $(\beta',\eta',\nu')$ is not necessarily equal to $(\Bar{\beta},\Bar{\eta}, \Bar{\nu} )$ .

\textbf{Step 1: Density Decomposition.} 
Subsequently, we consider
$Q_n(Y|X)=[1+\exp((\beta_n)^{\top}X+\tau_n)]
\cdot
[p_{ G_{n}}(Y|X)
    -
    p_{ G_{*,n}}(Y|X)]
$,
which can decomposed as
 \begin{align*}
Q_n(Y|X)
&=
\exp(\tau_n)
\left[
\exp((\beta_n)^{\top}X)
f(Y|h(X,\eta_n),\nu_n)
-
\exp((\beta_n^*)^{\top}X)
f(Y|h(X,\eta^*_n)),\nu_n^*)
\right]
:=\Ione_{n}
\\&
-
\exp(\tau_n)
\left[
\exp((\beta_n)^{\top}X)
-
\exp((\beta_n^*)^{\top}X)
\right]
p_{ \Gsn}(Y|X)
:=\Itwo_{n}
\\&
+
\left[
\exp(\tau_n)
-
\exp(\tau_n^*)
\right]
\exp((\beta_n^*)^{\top}X)
\left[
f(Y|h(X,\eta_n^*),\nu_n^*)
-
p_{ \Gsn}(Y|X)
\right]
\end{align*}
Based on the first order Taylor expansion, $\Ione_{n}$ and $\Itwo_{n}$ %and $\Ithree_{n}$ 
could be denoted as 
\begin{align}
    \Ione_{n}
&
=
\exp(\tau_n)
\sum_{|\alpha|=1}
\frac{1}{2^{\alpha_3}\alpha!}
(\beta_n-\beta_n^*)^{\alpha_1}
(\eta_n-\eta_n^*)^{\alpha_2}
(\nu_n-\nu_n^*)^{\alpha_3}
\nonumber
\\&
\hspace{1cm}
\cdot
X^{\alpha_1}
\exp((\beta_n^*)^{\top}X)
\cdot
\frac{\partial^{|\alpha_2|+2\alpha_3}f}{\partial h^{|\alpha_2|+2\alpha_3}}
\left(
Y|h(X,\eta_n^*),\nusn
\right)\dfrac{\partial^{\alpha_2} h}{\partial^{\alpha_2} \eta}(X,\eta_n^*)
+R_1(Y|X)
\nonumber
\\&
=
\exp(\tau_n)
\sum_{2|\ell_1|+\ell_2=1}^2
\sum_{\alpha\in\cali_{\ell_1,\ell_2}}
\frac{1}{2^{\alpha_4}\alpha!}
(\beta_n-\beta_n^*)^{\alpha_1}
(\eta_n-\eta_n^*)^{\alpha_2}
(\nu_n-\nu_n^*)^{\alpha_3}
\nonumber
\\&
\hspace{1cm}
\cdot
X^{\ell_1}
\exp((\beta_n^*)^{\top}X)
\cdot
\frac{\partial^{\ell_2}f}{\partial h^{\ell_2}}
\left(
Y|h(X,\eta_n^*),\nusn
\right)\dfrac{\partial^{\alpha_2} h}{\partial^{\alpha_2} \eta}(X,\eta_n^*)
+R_1(Y|X)
\end{align}
where 
$\ell_1=\alpha_1$, 
$\ell_2={|\alpha_2|+2\alpha_3}$,
and
\begin{align}
    \cali_{\ell_1,\ell_2}:=
    \left\{
\alpha=(\alpha_i)_{i=1}^3
\in\bbn^d\times\bbn^q\times\bbn:
\alpha_1=\ell_1,
2\alpha_3=\ell_2-|\alpha_2|
    \right\},
\end{align}
for all $(\ell_1,\ell_2)\in\bbn^d\times\bbn$
such that
$1\leq 2|\ell_1|+\ell_2\leq 2$.

Similarly, $ \Itwo_{n} $ can be expressed as:
\begin{align}
\Itwo_{n}&
=-\exp(\tau_n)
\sum_{|\gamma|=1}
\frac{1}{\gamma!}
(\beta_n-\beta_n^*)^{\gamma}
X^{\gamma}
\exp((\beta_n^*)^{\top}X)
p_{ \Gsn}(Y|X)
+R_2(Y|X).
\end{align}
Here 
${R_p(Y|X)}/{D_1(\Gn,\Gsn)}\to0$ 
as $n\to\infty$,
where $R_p(X,Y), p\in[2]$ are Taylor remainders
.
Consequently, $ Q_n $ can be expressed as:
\begin{align}
\label{pfeq:distinguish_wn}
Q_n&=
\sum_{2|\ell_1|+\ell_2=0}^{2}
T^n_{\ell_1,\ell_2}
\cdot
X^{\ell_1}
\exp((\beta_n^*)^{\top}X)
\dfrac{\partial^{\alpha_2} h}{\partial^{\alpha_2} \eta}
(X,\eta_n^*)
\frac{\partial^{\ell_2}f}{\partial h^{\ell_2}}
\left(
Y|h(X,\eta_n^*),\nusn
\right)
\nonumber
\\&
\hspace{5mm}
+
\sum_{|\gamma|=0}^{1}
S^n_{\gamma}
\cdot
X^{\gamma}
\exp((\beta_n^*)^{\top}X)
p_{ \Gsn}(Y|X)
,
\end{align}

with coefficients $ T_{\ell_1,\ell_2}^{n} $ and $ S_\gamma^n $ are defined 
for any 
$ 0 \leq 2|\ell_1| + \ell_2 \leq 2 $
, 
and $ 0 \leq |\gamma| \leq 1 $ 
as:

\begin{align*}
T_{\ell_1,\ell_2}^{n}=
\begin{cases}
\displaystyle
\exp(\tau_n)
\sum_{\alpha\in\cali_{\ell_1,\ell_2}}
\frac{1}{2^{\alpha_3}\alpha!}
(\beta_n-\beta_n^*)^{\alpha_1}
(\eta_n-\eta_n^*)^{\alpha_2}
(\nu_n-\nu_n^*)^{\alpha_3},
\quad
(\ell_1,\ell_2)\neq(0_d,0),
    \\
\displaystyle
\exp(\tau_n)
-
\exp(\tau_n^*),
\hspace{6.3cm}
(\ell_1,\ell_2)=(0_d,0);
\end{cases}
\end{align*}

and
\begin{align*}
S_{\gamma}^{n}=
\begin{cases}
\displaystyle
-\exp(\tau_n)
\frac{1}{\gamma!}
(\beta_n-\beta_n^*)^{\gamma}
,
\hspace{1.7cm}
|\gamma|\neq0,
    \\
\displaystyle
-\exp(\tau_n)
+\exp(\tau_n^*)
,
\hspace{2.2cm}
|\gamma|=0.
\end{cases}
\end{align*}
where $Q_n$ 
can be viewed as linear combinations of elements of the set 
$\calh_1$
defined as   
\begin{align}
\label{pfeq:distinguish_wn_elements}
    \calh_1=
\Big\{
X^{\ell_1}
\exp((\beta_n^*)^{\top}X)
% \cdot
\dfrac{\partial^{\alpha_2} h}{\partial\eta^{\alpha_2}} (X,\eta_n^*)
\frac{\partial^{\ell_2}f}{\partial h^{\ell_2}}
\left(
Y|h(X,\eta_n^*),\nusn
\right),
X^{\gamma}
\exp((\beta_n^*)^{\top}X)
p_{ \Gsn}(Y|X)
\Big\}.
\end{align}
\textbf{Step 2: Non-vanishing coefficients.}
In this step, we will use a contradiction argument to demonstrate that not all the coefficients in the set 
\begin{align}
\label{pfeq:coefficient_set_one}
\mathcal{S}_1=
\left\{ 
\frac{T^n_{\ell_1,\ell_2}}{D_{1n}},
\frac{S^n_{\gamma}}{D_{1n}}
:
0 \leq 2|\ell_1| + \ell_2 \leq 2,
0 \leq |\gamma| \leq 1
\right\}
\end{align}
% representations of 
% $ {T^{n}_{\ell_1,\ell_2}}/{\done} $, 
% $ {S^{n}_\gamma}/{\done} $, 
% $ \frac{B_n}{D_{1n}} $, and 
% $ \frac{C_n}{D_{1n}} $ 
% $[\plbgn(Y|X) - \plbgs(Y|X)]/D'_{1n}$
vanish as $ n \to \infty $ where
$D_{1n}:=\done$. 
Specifically, suppose that all these coefficients converge to zero, when $n\to\infty$, then we get,
% Under this assumption, it follows from the property 
% of distinguishibilty 
% and the definition
% of $ U_n $ in equation \eqref{pfeq:distinguish_p_decompose} that:
% \begin{align}
%     \frac{\lambdas-\lambdan}{D'_{1n}}\to0.
% \end{align}
% Now we recall $W_n$ is the linear combinations of elements in the set $\calh_1$
% defined in equation
% \eqref{pfeq:distinguish_wn_elements}
% with coefficients in the set 
% $\cals_1$ defined as
% \begin{align}
% \mathcal{S}_1=
% \left\{ 
% \frac{T^n_{\ell_1,\ell_2}(j)}{D'_{1n}},
% \frac{S^n_{\gamma}(j)}{D'_{1n}}
% :
% j \in [\ks],
% 0 \leq 2|\ell_1| + \ell_2 \leq 2\rbar,
% 0 \leq |\gamma| \leq 2
% \right\}.
% \end{align}
% Under the assumption that all of them converge to zero when $n\to 0$, 
% We will have 
\begin{align}
\label{pfeq:distinguish_coefficient_d1_1}
    \frac{    \left|
\exp(\tau_n)
-
\exp(\tau_n^*)
    \right|}{D_{1n}}
=
\frac{|T^n_{0_d,0}(j)|}{D_{1n}}
\to0,
\end{align}
Similarly, by analyzing the limits of 
${T_{\ell_1, \ell_2}^n}/{D_{1n}}$ 
s.t. $ 1 \leq 2| \ell_1 | + \ell_2 \leq 2 $, we conclude that:
\begin{align*}
    \frac{\exp(\tau_n) (\beta_n-\beta^*_n)^{(u)}}{D_{1n}}\to 0, 
    \frac{\exp(\tau_n)(\eta_n-\eta_n^*)^{(v)}}{D_{1n}}\to 0,
    \frac{\exp(\tau_n)(\nu_n-\nu_n^*)^{}}{D_{1n}}\to 0,
\end{align*}
as $n\to\infty$ for all $u\in[d], v\in[q]$.
Given that our parameter lies in a compact set, there exists a positive constant $C$ such that $|\exp(\tau_n^*)/\exp(\tau_n)|\leq C$. Thus, we have 
\begin{align*}
    \frac{\exp(\tau^*_n) (\beta_n-\beta^*_n)^{(u)}}{D_{1n}}\to 0, 
    \frac{\exp(\tau^*_n)(\eta_n-\eta_n^*)^{(v)}}{D_{1n}}\to 0,
    \frac{\exp(\tau^*_n)(\nu_n-\nu_n^*)^{}}{D_{1n}}\to 0,
\end{align*}
The limits imply that
\begin{align}
\label{pfeq:distinguish_coefficient_d1_2}
    (\exp(\tau_n)+ \exp(\tau^*_n))
    \|(\beta_n,\eta_n,\nu_n)-(\beta^*_n,\eta^*_n,\nu^*_n)\|
    /D_{1n}\to 0.
\end{align}
Combining the results in equations \eqref{pfeq:distinguish_coefficient_d1_1} 
and
\eqref{pfeq:distinguish_coefficient_d1_2} 
with the formulation of $ D_{1n} $, we deduce that
{
\begin{align*}
    1=\left[|\exp(\tau_n)-\exp(\tau_n^*)|+(\exp(\tau_n)+\exp(\tau_n^*))
    \|
    (\beta_n,\eta_n,\nu_n)-(\beta_n^*,\eta_n^*,\nu^*_n)
    \|\right]/D_{1n}\to 0,
\end{align*}
}
which is a contradiction. 
Thus, not all the coefficients in the set $\mathcal{S}_1$ tend to 0 as $n\to\infty$. 

\textbf{Step 3 - Application of Fatou’s lemma.}
Let us denote by $m_n$ the maximum of the absolute values of those coefficients. It follows from the previous result that $1/m_n\not\to\infty$. 
Then  
$ |T^n_{\ell_1,\ell_2}| / (m_n D_{1n}) $
and
$ |S^n_{\gamma}| / (m_n D_{1n}) $
remain bounded, we can consider subsequences of these terms, ensuring that:
$
{|T^n_{\ell_1,\ell_2}|}/{m_n D_{1n}} 
\to
\eta_{\ell_1,\ell_2}
,~
{|S^n_{\gamma}|}/{m_n D_{1n}}
\to
\omega_{\gamma}
,
$
as $n\to\infty$ 
for all
$
0 \leq 2|\ell_1| + \ell_2 \leq 2,
0 \leq |\gamma| \leq 1$.
Here, at least one among 
$\eta_{\ell_1,\ell_2}(j)$
and
$\omega_{\gamma}(j)$
is different from zero.
By applying the Fatou’s lemma, we get
\begin{align}
\label{pfeq:distinguish_fatou_lemma}
% 0=
\lim_{n\to\infty}
\frac{\bbE_X[d_V(\plbgn(\cdot|X),\plbgs(\cdot|X))]}{m_nD_{1n}}
\geq
\int
\liminf_{n\to\infty}
\frac{|\plbgn(Y|X)-\plbgs(Y|X)|}{2m_nD_{1n}}
d(X,Y)
% \geq0.
\end{align}
% {\color{red} This indicates that
% \begin{align*}
%     \frac{|\plbgn(Y|X)-\plbgs(Y|X)|}{m_nD'_{1n}}
%     \to0
% \end{align*}
% for almost surely $(X,Y)$.
% }
Under the given assumption, the left-hand side of the equation \eqref{pfeq:distinguish_fatou_lemma} is zero. Consequently, the integrand on the right-hand side of the equation \eqref{pfeq:distinguish_fatou_lemma} must also be zero almost surely with respect to $(X, Y)$. 
This results in:
\begin{align*}
    &
\sum_{2|\ell_1|+\ell_2=0}^{2}
\eta_{\ell_1,\ell_2}
\cdot
X^{\ell_1}
\exp((\beta^*)^{\top}X)\dfrac{\partial^{\alpha_2}h}{\partial \eta^{\alpha_2}}(X,\eta^*)
\frac{\partial^{\ell_2}f}{\partial h^{\ell_2}}
\left(
Y|h(X,\eta ^*),\nu^*
\right)
\\&
+
\sum_{|\gamma|=0}^{1}
\omega_{\gamma}
\cdot
X^{\gamma}
\exp((\beta^*)^{\top}X)
p_{ \Gs}(Y|X)
% \\&
% +
% % \tau    
% \sum_{j=1}^{\ks}
%     % \sum_{|\gamma|=0}^{1+\bfone_{|\calaj|>1}}
% \tau
% \cdot
% \exp(\bzj+t_1)
% \exp((\boj+t_2)^{\top}X)
% p_{G_{0}}(Y|X)
=0,
\end{align*}
for almost surely $(X, Y )$.
Furthermore, by Lemma \ref{appendix_lemma:distinguish_linear_independent}, the collection
\begin{align}
\label{pfeq:distinguish_set_linear_independent}
\calw_1
:=&
\left\{
X^{\ell_1}
\exp((\beta^*)^{\top}X)\dfrac{\partial^{\alpha_2}h}{\partial \eta^{\alpha_2}}(X,\eta^*)
\frac{\partial^{\ell_2}f}{\partial h^{\ell_2}}
\left(
Y|h(X,\eta^*),\nus
\right)
:
0 \leq 
% 2|\ell_1| + 
\ell_2 \leq 2
\right\}
% \\
% &
\nonumber
\\
&\cup
\left\{ X^{\gamma}
\exp((\beta^*)^{\top}X)
p_{ \Gs}(Y|X)
% :
% j \in [\ks],
% 0 \leq |\gamma| \leq 1+1_{\{|\calaj|>1 \}}
\right\}
% \cup
% \left\{
% p_{G_0}(Y|X)
% \right\}
\end{align}

is linearly independent with respect to $(X,Y)$. Consequently, it follows that
% for
% $
% j \in [\ks],
% 0 \leq 2|\ell_1| + \ell_2 \leq 2\rbar,
% 0 \leq |\gamma| \leq 1+1_{\{|\calaj|>1 \}}
% $
% ,
% (see Lemma 2), we deduce that $\eta = 0$.
% Since the exponential function is strictly larger than zero
% Since $\exp(x) > 0$ for all $x \in \mathbb{R}$, we know the exponential function never attains the value zero,
% it follows that
$ \eta_{\ell_1, \ell_2} = \omega_{\gamma}=0$,
for all
$
0 \leq 2|\ell_1| + \ell_2 \leq 2,
0 \leq |\gamma| \leq 1
$.
But this contradicts that from the definition, at least one among $ \eta_{\ell_1, \ell_2} $, $ \omega_{\gamma} $ is nonzero.
% which is a contradiction.
% When $\tau\neq0$,% {\color{red} ??? .}
% So now we conclude that 
% $ \eta_{\ell_1, \ell_2}(j) = \omega_{\gamma}(j) =\tau=0$, 
% which is a contradiction.
Hence, we reach the desired conclusion.
\end{proof}

\begin{lemma}
\label{appendix_lemma:distinguish_linear_independent}
Suppose that $f_0$ is distinguishable with $f$, 
% up to the first order. Also, there does not exist any $\eta$ and a direction $a \neq 0$ such that
% \begin{equation*}
%     \nabla_a h(X,\eta) = 0, \ a.s. 
% \end{equation*}
then the set $\calw_1$ defined in equation \eqref{pfeq:distinguish_set_linear_independent} is linearly independent w.r.t. $(X,Y )$. 
\end{lemma}

\begin{proof}[Proof of Lemma \ref{appendix_lemma:distinguish_linear_independent}]    
Recall the set 
\begin{align*}
\calw_1
&:=
\left\{
X^{\ell_1}
\exp((\beta^*)^{\top}X)\dfrac{\partial^{\alpha_2}h}{\partial \eta^{\alpha_2}}(X,\eta^*)
\frac{\partial^{\ell_2}f}{\partial h^{\ell_2}}
\left(
Y|h(X,\eta^*),\nus
\right)
:
0 \leq 
% 2|\ell_1| + 
\ell_2 \leq 2
\right\}
% \\
% &
\\
&\cup
\left\{ X^{\gamma}
\exp((\beta^*)^{\top}X)
p_{\lambdas,\Gs}(Y|X)
% :
% j \in [\ks],
% 0 \leq |\gamma| \leq 1+1_{\{|\calaj|>1 \}}
\right\}
% \cup
% \left\{
% p_{G_0}(Y|X)
% \right\}
\end{align*}

and the density 
\begin{align*}
        p_{G_*}(Y|X) & := \frac{1}{1+\exp((\beta^*)^{\top}X+\tau^*)}\cdot f_{0}(Y|h(X,\eta_0), \nu_{0})  \nonumber \\
    & \hspace{5 em} + \frac{\exp((\beta^*)^{\top}X+\tau^*)}{1+\exp((\beta^*)^{\top}X+\tau^*)}\cdot f(Y|h((X,\eta^*),\nu^*).
\end{align*}
In words, $p_{G_*}$ is a convex combination (depending on $X$) of 
\begin{align*}
  f_0(Y | h(X,\eta_0), \nu_0)
  \quad\text{and}\quad
  f(Y | h(X,\eta^*), \nu^*).
\end{align*}
% At first we know that $f_{0}(Y|\varphi(a_0^{\top}X+b_0), \nu_{0})\neq f(Y|\sigma((a^*)^{\top}X+b^*),\nu^*)$
% for almost every $(X,Y)$, and 
% $\displaystyle\dfrac{\partial^k f}{\partial \sigma^k}, k=0,1,2$ are linear independent.
\begin{comment}
Moreover, by assumption these two densities are distinguishable up to first order for almost every $(X,Y)$; symbolically,
\begin{align*}
  f_0(\cdot)  \neq  f(\cdot)
  \quad
  \text{(a.e.).}
\end{align*}
Also, by hypothesis on the Gaussian family, the set 
\begin{align*}
  \left\{
    \frac{\partial^k}{\partial \sigma^k}
    f\bigl(Y | h(X,\eta^*), \nu^*\bigr) 
    : 
    0 \le k \le 2
  \right\}
\end{align*}
is linearly independent as functions of $(X,Y)$ (refer to Lemma 9 in \cite{yan2024understanding}).
\end{comment}

Noting that the term in set $\calw_1$ can be divided as the density function or its first and second derivatives 
\begin{equation*}
    p_{G_*}(Y|X), f(
Y|h(X,\eta^*),
\frac{\partial f}{\partial h}\left(
Y|h(X,\eta^*)\right), \frac{\partial ^2f}{\partial h^2}\left(
Y|h(X,\eta^*)\right),
\end{equation*}
along with the factor involving only $X$. 

\textbf{Step 1:  Distinguishable property with respect to $Y$. }

First, fix $X$. Suppose for contradiction that there exist real numbers $c_0, c_1, c_2, d$ (may depend on $X$), not all zero, such that

\begin{align*}
% \label{eq:contradiction_combination}
  c_0 \frac{\partial^0 f}{\partial h^0}
  + c_1 \frac{\partial^1 f}{\partial h}
  + c_2 \frac{\partial^2 f}{\partial h^2}
  + d p_{G_*}(Y | X)
   = 
  0,
  \quad
  \text{for almost every }Y.
\end{align*}

Note that $\tfrac{\partial^0 f}{\partial h^0} = f$. Hence we have

\begin{align*}
  % \eqref{eq:contradiction_combination}
  % \quad\Longrightarrow\quad
  c_0 f(Y | h(X,\eta^*), \nu^*)
  + c_1 \frac{\partial f}{\partial h}(Y | h(X,\eta^*), \nu^*)
  + c_2 \frac{\partial^2 f}{\partial h^2}(Y | h(X,\eta^*), \nu^*)
  + d p_{G_*}(Y | X)
   = 
  0.
\end{align*}

Since
\begin{align*}
  p_{G_*}(Y | X) 
   =  
  \phi(X) f_0(\cdot) + \bigl(1 - \phi(X)\bigr) f(\cdot), 
  \quad
  \phi(X) 
  := 
  \frac{1}{1 + \exp \bigl((\beta^*)^{\top} X + \tau^*\bigr)},
\end{align*}
the above can be rewritten as
% \begin{align*}
%   c_0 f(\cdot)
%    + 
%   c_1 \frac{\partial f}{\partial\sigma}(\cdot)
%    + 
%   c_2 \frac{\partial^2 f}{\partial\sigma^2}(\cdot)
%    + 
%   d\phi(X) f_0(\cdot)
%    + 
%   d\bigl(1 - \phi(X)\bigr) f(\cdot)
%    = 
%   0.
% \end{align*}
% Group terms by $f_0(\cdot)$ versus $f(\cdot)$:
\begin{align*}
  % \underbrace
  {\Bigl[c_0 + d\bigl(1 - \phi(X)\bigr)\Bigr]}
  % _{\text{coefficient of }f(\cdot)}
  f(\cdot)
   + 
  c_1 \frac{\partial f}{\partial h}(\cdot)
   + 
  c_2 \frac{\partial^2 f}{\partial h^2}(\cdot)
   + 
  % \underbrace{
  d\phi(X)
  % }_{\text{coefficient of }f_0(\cdot)}
  f_0(\cdot)
   = 
  0.
\end{align*} 
%Because $f_0 \neq f$ a.e., for the above sum to be zero almost everywhere, we must have

Using the hypothesis about the distinguishable property of $f_0$ with respect to $f$ as well as the Gaussian property of $f$, which implies $\partial^2 f/\partial h^2 = 1/2 \cdot \partial f/\partial \nu$, we have
\begin{align*}
   d\phi(X)  =  0
   \quad \text{for almost all } X,
   \quad\text{and}\quad
   c_0 + d\bigl(1-\phi(X)\bigr)  =  0
   \quad \text{for almost all } X,
\end{align*}
and simultaneously $c_1=c_2=0$.
% because the three functions
% \begin{align*}
%   f(\cdot), \quad
%   \frac{\partial f}{\partial\sigma}(\cdot), \quad
%   \frac{\partial^2 f}{\partial\sigma^2}(\cdot)
% \end{align*}
% are linearly independent as functions of $(X,Y)$.  
But $\phi(X)\neq 0$ on a set of $X$-values of positive measure
% (indeed $\phi(X)\in(0,1)$ for all $X$)
, so $d=0$.  
Plugging $d=0$ into $c_0 + d(1-\phi(X))=0$ yields $c_0=0$.  Hence $c_0=c_1=c_2=d=0$.
% \dqledit{
% We might assume that the distribution of $f_0$ does not have the form of $p(y)e^{-ay^2/2}$? In this case, the distinguishability will be more clear/obvious. 
% }
% \fyedit{you mean $p(y)$ is a polynomial here?} \dqledit{yes}
% \fyedit{we have a $\phi(X)$ before $f_0$, so I think $f_0$ is not restricted?} \dqledit{Firstly, in the equation from which we use the independency of $f$ and $f_0$, there are also the term $X$ and $\exp(\beta^\top X+ \tau)$, so we should take into account of these terms. In addition, I think the term $Y$ is useful to derive the independency among the parameter, because $Y$ control the rapid decay rate of each distribution function. But now I think that \textbf{it is not so important}, we can concentrate of another part}. 
Since no nontrivial linear combination of 
$  \bigl\{ 
    f, \tfrac{\partial f}{\partial\sigma}, \tfrac{\partial^2 f}{\partial\sigma^2}, p_{G_*}
  \bigr\}$
can vanish almost everywhere, these four functions are linearly independent {when $X$ is fixed.}  This completes the proof of step 1.

\textbf{Step 2: Distinguishable property with respect to $X$. }

Let us consider coefficients appear in each density factor.

$\bullet$ Term related to $p_{G_*}(Y | X)$: The factor appearing along with $p_{G_*}(Y | X)$ are $\exp((\beta^*)^\top X)$, and $X^{(i)}\exp((\beta^*)^\top X)$, where $1\leq i \leq d$. Suppose there exists constants $c,a_1,\ldots,a_d$ such that 
\begin{equation*}
    c\exp((\beta^*)^\top X) + \sum_{i=1}^d a_iX^{(i)} \exp((\beta^*)^\top X) = 0, \ a.s.
\end{equation*}
This equation means that $c+ \sum_{i=1}^d a_iX^{(i)}=0$, a.s. Given that $X$ has non-vanish almost everywhere density function, this relation implies that $c=0$, $a_i = 0$, $1\leq i \leq d$. 

$\bullet$ Terms related to $f(Y|h(X,\eta^*),\nu^*)$:  The factors appearing along with $f(Y|h(X,\eta^*),\nu^*)$ are $\exp((\beta^*)^\top X)$, and $X^{(i)}\exp((\beta^*)^\top X)$, where $1\leq i \leq d$. The identical argument as in the case for $p_{G_*}(Y | X)$ also gives us the independency.

$\bullet$ Terms related to $\frac{\partial f}{\partial h}\left(
Y|h(X,\eta^*),\nu^*\right)$:  The factors appearing along with $p_{G_*}(Y | X)$ are $$\dfrac{\partial h}{\partial \eta^{(i)}}(X,\eta^*)\exp((\beta^*)^\top X),\  1\leq i \leq d.$$ Suppose there exists constants $a_1,\ldots,a_d$ not all equal to zero such that 
\begin{equation*}
    \sum_{i=1}^d a_i\dfrac{\partial h}{\partial \eta^{(i)}}(X,\eta^*)\exp((\beta^*)^\top X) = 0, \ a.s.
\end{equation*}
This equation means that 

$$\sum_{i=1}^d a_i\dfrac{\partial h}{\partial \eta^{(i)}}(X,\eta^*) = 0, \ \text{a.s.}, \quad \text{or } \nabla_{a}h(X,\eta^*) = 0, \ \text{a.s.},$$ 
where $a = (a_1,\ldots,a_d)$. This is a contradiction.

$\bullet$ Terms related to $\frac{\partial ^2f}{\partial h^2}\left(
Y|h(X,\eta^*)\right)$: There is only one such term is 
\begin{equation*}
    \exp((\beta^*)^\top X)\frac{\partial ^2f}{\partial h^2}\left(
Y|h(X,\eta^*)\right).
\end{equation*}
Its coefficient obviously vanishes from the independent property with respect to $Y$. 

This completes the proof of Lemma \ref{appendix_lemma:distinguish_linear_independent}.
\end{proof}

\begin{comment}
\begin{lemma}
    Suppose that the prompt function $h(X,\eta)$ having the form $h(X,\eta) = \sigma(a^\top X + b)$ such that $\sigma: \mathbb{R} \to \mathbb{R}$ is continuously derivable, and the parameter  $(a,b) \in \mathbb{R}^d\times\mathbb{R}$ lies in a compact set.
    Then in this case ...
\end{lemma}

\begin{proof}
    We have 
    \begin{equation*}
        \dfrac{\partial \sigma}{\partial b}(a^\top X + b ) = \sigma'(a^\top X + b), \ \dfrac{\partial \sigma}{\partial a_i}(a^\top X + b) = X_i\sigma'(a^\top X+b). 
    \end{equation*}
    Suppose that there exists a direction $(v_0,v) \neq  0$ such that $\nabla_{(v_0,v)}\sigma(a^\top X + b) = 0, \ a.s$. This means that 
    \begin{equation*}
        (v_0 + v^\top X) \sigma'(a^\top X+b) = 0, \ a.s. 
    \end{equation*}
\end{proof}
Here we consider $\sigma$ is a non-constant function, which means that $\sigma'\neq 0$ in a set of non-zero measure, so only when $v_0 = 0$ and $v = \boldsymbol{0}$. 
\end{comment}

\subsection{Proof of Theorem \ref{thm:lower-distinguish}}
% \label{app_subsec:d1_rate}
% \subsection{$D_1$ loss minimax}
\label{apppf:d1_minimax}

% \begin{theorem}[Minimax lower bounds]
% % [Distinguishable settings]
% \label{appthm:lower-distinguish}
% % Assume that classes of densities $f_0$ and $f$ satisfy the conditions in Theorem~\ref{theorem:sigma-linear-f0-notGaussian}
% % or Theorem~\ref{theorem:sigma-linear-f0-Gaussian-varphi-nonlinear}
% % or Theorem~\ref{theorem:sigma-nonlinear-f0-notGaussian}
% % or Theorem~\ref{theorem:sigma-nonlinear-f0-Gaussian-varphi-linear}
% % . 
% %     % \colorbox{BurntOrange}{We need to clarify here, this distinguish setting is for thm1,3,4,5}
% % Then, 
% Under the setting of Theorem~\ref{thm:not_equal}, we have for any $0<r < 1$ that
% \begin{align*}
%     \inf_{\overline{G}_n\in \Xi }\sup_{G\in \Xi }
%     \mathbb{E}_{p_{ G}} \Big( 
%     |\exp(\overline{\tau}_n)
%     -\exp(\tau)|^2 \Big) 
%     \gtrsim n^{-1/r},
%     \\
%     \inf_{\overline{G}_n\in \Xi }\sup_{G\in \Xi }
%     \mathbb{E}_{p_{ G}} 
%     \Big( \exp^2(\tau) \Vert (\overline{\beta}_n,\overline{\eta}_n,\overline{\nu}_n)-(\beta,\eta,\nu) \Vert^2 \Big) 
%     \gtrsim n^{-1/r}.
% \end{align*}
% % Here, the infimum is taken over all sequences of estimates $(\widehat{\lambda}_n, \widehat{G}_n)=(\widehat{\lambda}_n, \widehat{a}_n, \widehat{b}_n, \widehat{\nu}_n)$.
% \end{theorem}

% In order to prove the minimax lower bounds for distinguishable settings, that is, Theorem \ref{thm:lower-distinguish}, we will at first define two distances:
As a first step in proving the minimax lower bounds for the distinguishable setting (Theorem~\ref{thm:lower-distinguish}), we define two distances:
\begin{align*}
    &d_1(G_1,G_2)
    =\exp(\tau_1)\Vert (\beta_1,\eta_1,\nu_1)-(\beta_2,\eta_2,\nu_2) \Vert,
    \\&
    d_2(G_1,G_2)
    = | \exp(\tau_1) -\exp(\tau_2) |^2,
\end{align*}
for any $G_1=(\beta_1,\tau_1,\eta_1,\nu_1)\in\Xi$ and $G_2=(\beta_2,\tau_2,\eta_2,\nu_2)\in\Xi$.
Obviously  $d_2(G_1, G_2)$ is a proper distance. 
The structure for $d_1(G_1, G_2)$ tells us that it is not symmetric.
Only when $\tau_1=\tau_2=\tau$, $d_1(G_1, G_2)$ is symmetric.
Also $d_1(G_1, G_2)$ still satisfies a weak triangle inequality:
\begin{align*}
    d_1(G_1, G_2)+
    d_1(G_2, G_3)
    \geq
    \min
    \{
    d_1(G_1, G_2),
    d_1(G_2, G_3)
    \}.
\end{align*}
Therefore, we will apply the modified Le Cam method for nonsymmetric loss, as outlined in Lemma C.1 of \cite{gadat2020parameter}, to handle this distance. 
For $f$ satisfies all assumptions in Theorem \ref{thm:lower-distinguish}, based on the Taylor expansion, we have the following results:
\begin{lemma}
\label{prop:lower-distinguish}
    Given $f$ in Theorem \ref{thm:lower-distinguish}, 
we denote
\begin{align*}
    S_{1 } = (\tau , \beta_{1 }, \eta_{1 }, \nu_{1 }), 
    S_{2 } = (\tau , \beta_{2 }, \eta_{2 }, \nu_{2 }),
    ~\text{and}~
    S'_{1} = (\tau_{1}, \beta, \eta, \nu), 
    S'_{2} = (\tau_{2}, \beta, \eta, \nu),
\end{align*}
    we achieve for any $r < 1$ that
\begin{align*}
        &\text{(i)}~~
            \lim_{\epsilon \rightarrow 0} 
        \inf_{S_1,S_2}
        \left\{\displaystyle
        \frac{\bbE_X[d_H\left(p_{S_1}(\cdot|X), p_{S_2}(\cdot|X)\right)] }{d_1^r\left(S_1,S_2\right)}
        : d_1\left(S_1,S_2\right) \leq \epsilon\right\}=0,\\
        &\text{(ii)}~~
            \lim_{\epsilon \rightarrow 0}
        \inf_{S'_1,S'_2}
        \left\{\displaystyle
        \frac{\bbE_X[d_H\left(p_{S'_1}(\cdot|X), p_{S'_2}(\cdot|X)\right)] }{d_2^r\left(S'_1, S'_2\right)}
        : d_2\left(S'_1, S'_2\right) \leq \epsilon\right\}=0.
        \end{align*}
\end{lemma}
We will prove this lemma later.

\begin{proof}[Proof of Theorem \ref{thm:lower-distinguish}]
Denote $\Gs=(\betas,\taus,\etas,\nus)$ and assume $r<1$. 
Given Lemma \ref{prop:lower-distinguish} part (i) , for any sufficiently small $\epsilon > 0$, there exists 
$ G'_{*} = (\beta^{*}_1,\tau^{*},  \eta^{*}_1,\nu^{*}_1) $ 
such that 
$d_1( G_{*} , G'_{*}) = d_1(G'_{*},  G_{*} ) = \epsilon $
,
there exists a constant $C_0$, s.t.
\begin{align}
\label{apppf:lower1-hellinger-distance}
    \bbE_X[d_H(p_{ G_{*}}(\cdot|X), p_{ G'_{*}}](\cdot|X)) \leq C_0 \epsilon^r.
\end{align} 

Now we will denote $p^n_{ G_{*}}$ as the density of the $n$-i.i.d. sample $(X_1,Y_1),\cdots,(X_n,Y_n)$.
Lemma C.1 in \cite{gadat2020parameter} tells us that
\begin{align*}
\label{proof:lower-distinguish-eq1}
    \inf_{\llbgn \in \Xi }\sup_{\lbg\in \Xi }\mathbb{E}_{\plbg} \Big( \exp^2(\tau) \Vert (\overline{\beta}_n, \overline{\eta}_n, \overline{\nu}_n)-(\beta,\eta,\nu) \Vert^2 \Big) 
    &\geq \frac{\epsilon^2}{2}\Big(1-\bbE_X[d_V(p^n_{ G_{*}}(\cdot|X), p^n_{ G'_{*}}(\cdot|X))] \Big)
    \\
    &
    \geq \frac{\epsilon^2}{2}
    \sqrt{1-\left( 1-C_0^2\epsilon^{2r} \right)^n}.
\end{align*}
\begin{comment}
Building on that, 
\begin{align*}
    V(p^n_{\lambdas,G_{*}}, p^n_{\lambdas,G'_{*}})
    &\leq
    h(p^n_{\lambdas,G_{*}}, p^n_{\lambdas,G'_{*}})
    \\&
    =\sqrt{1-\left( 1-h^2(p_{\lambdas,G_{*}}, p_{\lambdas,G'_{*}}) \right)^n}
    \\&
    \leq \sqrt{1-\left( 1-C_0^2\epsilon^{2r} \right)^n},
\end{align*}
plugging this result into (\ref{proof:lower-distinguish-eq1}), we will have
\begin{align*}
    \inf_{\hlbgn\in \Xi }\sup_{\lbg\in \Xi }\mathbb{E}_{\plbg} \Big( \lambda^2 \Vert (\widehat{a}_n, \widehat{b}_n, \widehat{\nu}_n)-(a,b,\nu) \Vert^2 \Big) 
    \geq \frac{\epsilon^2}{2}
    \sqrt{1-\left( 1-C_0^2\epsilon^{2r} \right)^n}.
\end{align*}
\end{comment}
Last inequality is from the definition of the Total Variation distance and Hellinger distance and equation \eqref{apppf:lower1-hellinger-distance}. 
Let $\epsilon^{2r}=\displaystyle\frac{1}{C_0^2n}$, then for any $r<1$ we have
\begin{align*}
    \inf_{\llbgn\in \Xi }\sup_{\lbg\in \Xi }\mathbb{E}_{\plbg} \Big( \exp^2(\tau) \Vert (\overline{\beta}_n, \overline{\eta}_n, \overline{\nu}_n)-(a,b,\nu) \Vert^2 \Big) 
    \geq c_1 n^{-1/r},
\end{align*}
where $c_1$ is some positive constant.
Following a similar reasoning and using Lemma \ref{prop:lower-distinguish} part (ii) , we will obtain
\begin{align*}
    \inf_{\llbgn\in \Xi }\sup_{\lbg\in \Xi }\mathbb{E}_{\plbg} \Big( |\exp(\overline{\tau}_n)
    -\exp(\tau)|^2 \Big) \geq c_2 n^{-1/r},
\end{align*}
for some positive constant $c_2$. 
Consequently, we establish all of the results for Theorem \ref{thm:lower-distinguish}.
\end{proof}

% \subsubsection*{Proof of Lemma \ref{prop:lower-distinguish}:}
\begin{proof}[Proof of Lemma \ref{prop:lower-distinguish} (i) ]
\begin{comment}
Now consider two sequences 
% \begin{align*}
%     \begin{cases}
%         (\lambdan,G_{1,n})=(\tau_n,\beta_{1,n},\eta_{1,n},\nu_{1,n})\\
%         (\lambdan,G_{2,n})=(\tau_n,\beta_{2,n},\eta_{2,n},\nu_{2,n})
%     \end{cases}
% \end{align*}
\begin{align*}
    \begin{cases}
        S_{1,n}=(\tau_n,\beta_{1,n},\eta_{1,n},\nu_{1,n})\\
        S_{2.n}=(\tau_n,\beta_{2,n},\eta_{2,n},\nu_{2,n})
    \end{cases}
\end{align*}
with the same $\tau_n$. Just refer to the contaminated MoE model definition \eqref{eq:contaminated_pretrain_model_general} we will have
% \begin{align*}
% h^2\left(p_{\lambdan,G_{1, n}}, p_{\lambdan,G_{2, n}}\right) 
%  \leq 
% \frac{1}{\lambda_n} 
% \int 
% \frac{\left(p_{\lambdan,G_{1, n}}(Y|X)-p_{\lambdan,G_{2, n}}(Y|X)\right)^2}{f\left(\ysigmatwon\right)} d(X,Y) 
% \end{align*}
\begin{align*}
h^2\left(p_{S_{1, n}}, p_{S_{2, n}}\right) 
 \leq \exp(\tau_n)
\int \frac{\left[\exp(\beta_{1,n}^\top X)f
\left(\yhonen\right)
-\exp(\beta_{2,n}^\top X)f\left(\yhtwon\right)\right]^2}{\exp(\beta_{2,n}^\top X)f\left(\yhtwon\right)} d(X,Y)
\end{align*}
due to 
$\sqrt{p_{\lambdan,G_{1, n}}(Y|X)}+\sqrt{p_{\lambdan,G_{2, n}}(Y|X)}>\sqrt{\lambdan f\left(\ysigmatwon\right)}$
and
$p_{\lambdan,G_{1, n}}(Y|X)-p_{\lambdan,G_{2, n}}(Y|X)=\lambdan\left(f
\left(\ysigmaonen\right)
-f\left(\ysigmatwon\right)\right)$
. 
\end{comment}
Consider two sequences
\begin{align*}
    S_{1,n} &= (\tau_n, \beta_{1,n}, \eta_{1,n}, \nu_{1,n}), \\
    S_{2,n} &= (\tau_n, \beta_{2,n}, \eta_{2,n}, \nu_{2,n}),
\end{align*}
with the same $\tau_n$. By the contaminated MoE model definition, we have
\begin{align*}
    p_{S_{j,n}}(Y | X) 
    &= \frac{1}{1 + \exp(\beta_{j,n}^\top X + \tau_n)} f_0(Y | h_0(X, \eta_0), \nu_0)
    \\
    &\hspace{4cm}+ \frac{\exp(\beta_{j,n}^\top X + \tau_n)}{1 + \exp(\beta_{j,n}^\top X + \tau_n)} f(Y| h(X, \eta_{j,n}), \nu_{j,n}),
\end{align*}
for $j = 1,2$.
Since $(\tau_n,\beta_{j,n})$ lie in a compact set, and both $f_0$ and $f$ are 
non-negative.
% bounded away from zero on their supports, there exists a constant $c > 0$ such that
% \begin{align*}
%     p_{S_{j,n}}(Y |X) \geq c > 0 \quad \text{for all } X, Y, j\in\{1,2\}.
% \end{align*}
Hence, the squared Hellinger distance satisfies
\begin{align*}
    \bbE_X&[d_H^2(p_{S_{1,n}}(\cdot|X), p_{S_{2,n}}(\cdot|X))]
    \leq C \int \left( \frac{p_{S_{1,n}}(Y  |  X) - p_{S_{2,n}}(Y  |  X)}{p_{S_{2,n}}(Y  |  X)} \right)^2 d(X, Y) \\
    &\leq C' \int \left[
    \frac{
        \exp(\beta_{1,n}^\top X) f(Y  |  h(X, \eta_{1,n}), \nu_{1,n})
        - \exp(\beta_{2,n}^\top X) f(Y  |  h(X, \eta_{2,n}), \nu_{2,n})
    }{
        \exp(\beta_{2,n}^\top X) f(Y  |  h(X, \eta_{2,n}), \nu_{2,n})
    }
    \right]^2 d(X, Y),
\end{align*}
for some constants $C, C'$ depending on the compactness bounds.

Consider the Taylor expansion of the map
\begin{align*}
(\beta, \eta, \nu) \mapsto \exp(\beta^\top X) f\left(Y  |  h(X, \eta), \nu\right)
\end{align*}
at the point $(\beta_{2,n}, \eta_{2,n}, \nu_{2,n})$, expanded up to first order with integral remainder. Let $\alpha = (\alpha_1, \alpha_2, \alpha_3)$ denote a multi-index where $\alpha_1 \in \mathbb{N}^d$, $\alpha_2 \in \mathbb{N}^{q}$, and $\alpha_3 \in \mathbb{N}$ index components of $\beta$, $\eta$, and $\nu$, respectively. Then we have:
\begin{align*}
    &\exp(\beta_{1,n}^\top X) f\left(Y  |  h(X, \eta_{1,n}), \nu_{1,n} \right)
    - \exp(\beta_{2,n}^\top X) f\left(Y  |  h(X, \eta_{2,n}), \nu_{2,n} \right)
    \\
    &= \sum_{|\alpha| = 1}
    \frac{
        (\beta_{1,n} - \beta_{2,n})^{\alpha_1}
        (\eta_{1,n} - \eta_{2,n})^{\alpha_2}
        (\nu_{1,n} - \nu_{2,n})^{\alpha_3}
    }{
        \alpha_1! \alpha_2! \alpha_3!
    }
    \\
    &\hspace{5cm}\cdot X^{\alpha_1}
    \exp(\beta_{2,n}^\top X)
    \frac{\partial^{|\alpha_2| + \alpha_3} f}{\partial \eta^{\alpha_2} \partial \nu^{\alpha_3}}(Y  |  h(X, \eta_{2,n}), \nu_{2,n})
    \\
    & + \sum_{|\alpha| = 1}
    \frac{
        (\beta_{1,n} - \beta_{2,n})^{\alpha_1}
        (\eta_{1,n} - \eta_{2,n})^{\alpha_2}
        (\nu_{1,n} - \nu_{2,n})^{\alpha_3}
    }{
        \alpha_1! \alpha_2! \alpha_3!
    }
    \int_0^1 
    X^{\alpha_1}
    \exp\left((\beta_{2,n} + t (\beta_{1,n} - \beta_{2,n}))^\top X\right)
    \\
    &\hspace{4cm} \cdot
    \frac{\partial^{|\alpha_2| + \alpha_3} f}{\partial \eta^{\alpha_2} \partial \nu^{\alpha_3}} 
    \left(Y  |  h(X, \eta_{2,n} + t (\eta_{1,n} - \eta_{2,n})), \nu_{2,n} + t (\nu_{1,n} - \nu_{2,n})\right) dt.
\end{align*}

So it follows that 
\begin{align*}
    \frac{\bbE_X[d_H^2(p_{S_{1,n}}(\cdot|X), p_{S_{2,n}}(\cdot|X))]}{d_{1}^{2r}(S_{1,n}, S_{2,n})} \to 0.
\end{align*}
since $\tau_n$ lies in a compact set.
% when we have
%     $d_1^{2-2r}((\lambdan,G_{1,n}),(\lambdan,G_{2,n}))/\lambdan
%     =\lambdan^{1-2r}\Vert (a_{1,n},b_{1,n},\nu_{1,n})-(a_{2,n},b_{2,n},\nu_{2,n}) \Vert^{2-2r}\to 0$
% and 
%     $d_1((\lambdan,G_{1,n}),(\lambdan,G_{2,n}))/\lambdan
%     =\Vert (a_{1,n},b_{1,n},\nu_{1,n})-(a_{2,n},b_{2,n},\nu_{2,n}) \Vert \to 0$.
% Naturally we got part (i).
This establishes part (i) of the lemma.

\end{proof}

\begin{proof}[Proof of Lemma \ref{prop:lower-distinguish} (ii)]
We consider two sequences
\begin{align*}
S'_{1,n} &= (\tau_{1,n}, \beta_n, \eta_n, \nu_n), \\
S'_{2,n} &= (\tau_{2,n}, \beta_n, \eta_n, \nu_n),
\end{align*}
with different $\tau_{1,n} \neq \tau_{2,n}$ but the same $(\beta_n, \eta_n, \nu_n)$. 

Using the contaminated MoE definition, the difference in conditional densities is:
\begin{align*}
p_{S'_{1,n}}(Y  |  X) - p_{S'_{2,n}}(Y  |  X)
&= \frac{e^{\beta_n^\top X} \left( e^{\tau_{2,n}} - e^{\tau_{1,n}} \right)}
{ \left( 1 + e^{\beta_n^\top X + \tau_{1,n}} \right) \left( 1 + e^{\beta_n^\top X + \tau_{2,n}} \right) }\\
&\hspace{3cm}\cdot \left[ f(Y  |  h(X, \eta_n), \nu_n) - f_0(Y  |  h_0(X,\eta_0), \nu_0) \right].
\end{align*}

By the standard bound for squared Hellinger distance,
\begin{align*}
\bbE_X[d_H^2(p_{S'_{1,n}}(\cdot|X), p_{S'_{2,n}}(\cdot|X))]
\leq C \int \left( \frac{p_{S'_{1,n}}(Y  |  X) - p_{S'_{2,n}}(Y  |  X)}{p_{S'_{2,n}}(Y  |  X)} \right)^2 d(X, Y).
\end{align*}

Since $(\beta_n, \eta_n, \nu_n)$ lie in a compact set, and both $f$ and $f_0$ are bounded away from zero, we have $ p_{S'_{2,n}}(Y  |  X) \geq c > 0 $. So the denominator is lower bounded.

Then there exists a constant $C'$ such that:
\begin{align*}
\bbE_X[d_H^2(p_{S'_{1,n}}(\cdot|X), p_{S'_{2,n}}(\cdot|X))]
\leq C' \left(e^{\tau_{1,n}} - e^{\tau_{2,n}} \right)^2.
\end{align*}

Now recall the definition of the distance:
\begin{align*}
d_2((\tau_{1,n}, \beta_n, \eta_n, \nu_n), (\tau_{2,n}, \beta_n, \eta_n, \nu_n)) := | e^{\tau_{1,n}} - e^{\tau_{2,n}} |^2.
\end{align*}

So we conclude:
\begin{align*}
\frac{\bbE_X[d_H^2(p_{S'_{1,n}}(\cdot|X), p_{S'_{2,n}}(\cdot|X))]}{d_2((S'_{1,n}, S'_{2,n}))^r}
\leq \frac{C' \left| e^{\tau_{1,n}} - e^{\tau_{2,n}} \right|^2}{\left| e^{\tau_{1,n}} - e^{\tau_{2,n}} \right|^{2r}} 
= C' \left| e^{\tau_{1,n}} - e^{\tau_{2,n}} \right|^{2(1 - r)} \to 0
\end{align*}
as long as $ e^{\tau_{1,n}} - e^{\tau_{2,n}} \to 0 $, and $ r < 1 $. 

Hence,
\begin{align*}
\frac{\bbE_X[d_H^2(p_{S'_{1,n}}(\cdot|X), p_{S'_{2,n}}(\cdot|X))]}{d_2^r(S'_{1,n}, S'_{2,n})} \to 0,
\end{align*}
which proves part (ii).
\end{proof}

\subsection{Proof of Theorem \ref{thm:d2_mle_rate}}
% \label{app_subsec:d1_rate}
% \subsection{$D_2$ loss}

% Now we will prove Theorem \ref{thm:d2_mle_rate} for non-distinguishable settings.
We proceed to prove Theorem~\ref{thm:d2_mle_rate} for the non-distinguishable setting.

% \begin{align*}
%     \overline{D_2}\left( G, G_*) \right)
%  &:=
%     \|\beta-\beta^*\|
%     \exp(\tau)
%     \Vert (\Delta\eta,\Delta\nu) \Vert
%     \\&
%     +
%     \left|
%     \exp(\tau^*)-\exp(\tau)
%     \right|
%     \cdot
%     \Vert (\Delta\eta,\Delta\nu)\Vert
%     \cdot
%     \Vert (\Delta\eta^*,\Delta\nu^*)\Vert
%     \\&
%     +\Vert  (\Delta\eta,\Delta\nu)-(\Delta\eta^*,\Delta\nu^*) \Vert
%     \big( 
%     \exp(\tau)\Vert (\Delta\eta,\Delta\nu) \Vert
%     +\exp(\tau^*)\Vert (\Delta\eta^*,\Delta\nu^*) \Vert
%     \big)
%     \\&
%     +
%     \exp(\tau+\tau^*)\cdot
%     \Vert (\Delta\eta,\Delta\nu)-(\Delta\eta^*,\Delta\nu^*)\Vert^2
% \end{align*}

\begin{proof}
\label{app_proof: d2_loss}

% First we denote $\overline{D_2}(G,\Gs)$ by
% \begin{align*}
%     \overline{D_2}\left( G, G_*) \right)
%  &:=
%     \|\beta-\beta^*\|
%     % \exp(\tau)
%     % \Vert (\Delta\eta,\Delta\nu) \Vert
%     \big( 
%     \exp(\tau)\Vert (\Delta\eta,\Delta\nu) \Vert
%     +\exp(\tau^*)\Vert (\Delta\eta^*,\Delta\nu^*) \Vert
%     \big)
%     \\&
%     +
%     \left|
%     \exp(\tau^*)-\exp(\tau)
%     \right|
%     \cdot
%     \Vert (\Delta\eta,\Delta\nu)\Vert
%     \cdot
%     \Vert (\Delta\eta^*,\Delta\nu^*)\Vert
%     \\&
%     +\Vert  (\Delta\eta,\Delta\nu)-(\Delta\eta^*,\Delta\nu^*) \Vert
%     \big( 
%     \exp(\tau)\Vert (\Delta\eta,\Delta\nu) \Vert
%     +\exp(\tau^*)\Vert (\Delta\eta^*,\Delta\nu^*) \Vert
%     \big)
%     \\&
%     +
%     \exp(\tau+\tau^*)\cdot
%     \Vert (\Delta\eta,\Delta\nu)-(\Delta\eta^*,\Delta\nu^*)\Vert^2
% \end{align*}

Let $\overline{G}=(\Bar{\beta},\Bar{\tau},\Bar{\eta},\Bar{\nu})$ 
% and $\overline{\lambda}=(\Bar{\beta},\Bar{\tau})$ lies in a compact set 
and $(\Bar{\eta},\Bar{\nu})$ can be identical to $({\eta_0},{\nu_0})$. 
Then, we will show that
\begin{itemize}
    \item[(i)] When 
    $({\eta_0},{\nu_0})\neq(\Bar{\eta},\Bar{\nu})$,
    \begin{align*}
     \lim_{\varepsilon\to 0}\inf_{G,\Gs}\left\{\frac{\|p_{G} - p_{G_*}\|_{\infty}}{D_1(G,G_*)}:D_1(G,\overline{G})\vee D_1(G_*,\overline{G})\leq\varepsilon\right\}>0.
    \end{align*}
    \item[(ii)] 
    When $({\eta_0},{\nu_0})=(\Bar{\eta},\Bar{\nu})$,
    \begin{equation}
    \label{eq:claim_nondistinguishable_independent}
     \lim_{\varepsilon\to 0}\inf_{G,\Gs}\left\{\frac{\|p_{G} - p_{G_*}\|_{\infty}}{D_2(G,G_*)}:D_2(G,\overline{G})\vee D_2(G_*,\overline{G})\leq\varepsilon\right\}>0.
    \end{equation}
\end{itemize}
 Part (i) can be proved by using the same arguments as in the proof~\ref{app_proof: d1_loss}.
 Thus, we will consider only part (ii) in this section, specifically the most challenging setting that $({\eta_0},{\nu_0})=(\Bar{\eta},\Bar{\nu})$. 
Under this assumption, we know that $h_0$ and $h$ are the same 
expert function, s.t. 
$f_0(Y|h_0 (X,\eta_0),\nu_0)=f(Y|h(X,\eta_0 ),\nu_0)$ for almost surely $(X,Y)\in \mathcal{X}\times\mathcal{Y}$.
Assume that the above claim in equation~\eqref{eq:claim_nondistinguishable_independent} does not hold, 
then there exist two sequences $G_n=(\beta_n,\tau_n,\eta_n,\nu_{n})$ and $G_{*,n}=(\beta^*_n,\tau^*_n,\eta_n^*, \nu^*_{n})$, 
% and two sequences of gating parameters $\lambda_n=(\beta_n,\tau_n)$ and $\lambda_n^*=(\beta^*_n,\tau^*_n) \in \Lambda$, a compact set, 
such that
\begin{align*}
    \begin{cases}
        D_2(\Gn,\overline{G})
        \to 0,\\
        D_2(\Gsn,\overline{G})
        \to 0,
        \\ {\|p_{ \Gn} - p_{  \Gsn}\|_{\infty}}/{D_2(\Gn,\Gsn)}\to 0.
    \end{cases}
\end{align*}
% Now, we have three primary cases regarding the convergence behaviors between $(\lambda_n,G_n)=(\beta_n,\tau_n,\eta_{n},\nu_{n})$ and $(\lambda_n^*,G_{*,n})=(\beta_n^*, \tau_n^*, \eta^*_{n},\nu^*_{n})$.\\
% 1. $G_n,G_n^*\to G'= G_0$\\
% 2. $G_n,G_n^*\to G'\neq G_0$\\
% 3. $G_n\to G', G_n^*\to G_0$ or $G_n\to G_0, G_n^*\to G'$\\
We now analyze the limiting behavior of the sequences $(\lambda_n, G_n)$ and $(\lambda_n^*, G_n^*)$ as they approach $(\bar{\lambda}, \bar{G})$. In particular, we distinguish between three asymptotic regimes based on how the expert parameters $\varsigma_n = (\eta_n, \nu_n)$ and $\varsigma_n^* = (\eta_n^*, \nu_n^*)$ converge.

First, it may occur that both $\varsigma_n$ and $\varsigma_n^*$ converge to the same limit $\varsigma_0 = (\eta_0, \nu_0)$. Alternatively, both sequences may converge to a common limit $\varsigma' \neq \varsigma_0$, which is distinct from the true expert. Finally, it is also possible that one sequence converges to $\varsigma_0$ while the other converges to a different point $\varsigma' \neq \varsigma_0$.

In the following, we analyze each of these cases and demonstrate that in all scenarios, the assumption that the normalized difference vanishes leads to a contradiction when $f_0=f$.

\subsubsection*{Case 1:}
At first we consider that 
$(\eta_n,\nu_n)$ and $(\eta_n^*,\nusn)$ share the same limit of $(\eta_0,\nu_0)$. Without loss of generality, we can suppose that  $\tau_n^*\geq \tau_n$. Subsequently, we consider
$W_n:=[p_{  G_{n}}(Y|X)-
    p_{  G_{*,n}}(Y|X)]
    \cdot[1+\exp((\beta_n^*)^{\top}X+\tau_n^*)]
    \cdot[1+\exp((\beta_n)^{\top}X+\tau_n)]$,
which can decomposed as
\begin{align*}
    W_n
    &=
    \exp(\tau_n) 
    \cdot
    [g(Y|X;\beta_n,\eta_n,\nu_n)-g(Y|X;\beta_n^*,\eta_n^*,\nu_n^*)]
    \\&
    -
    \exp(\tau_n)
    \cdot
    [g(Y|X;\beta_n,\eta_0,\nu_0)-g(Y|X;\beta_n^*,\eta_n^*,\nu_n^*)]
    \\&
    +\exp(\tau^*_n)
    \cdot
    [g(Y|X;\beta_n^*,\eta_0,\nu_0)-g(Y|X;\beta_n^*,\eta_n^*,\nu_n^*)]
    \\&
    +\exp\left( (\beta_n^* + \beta_n)^{\top}X + \tau_n^* + \tau_n \right)
    \cdot[f(Y|h(X,\eta_n),\nu_n)-f(Y|h(X,\eta_n^*),\nu_n^*)]
    \\&
    :=\Ione_{n}-\Itwo_{n}+\Ithree_{n}+\Ifour_{n}
\end{align*}
where we denote
$
    g(Y|X;\beta, \eta, \nu) =
    e(X;\beta)f(Y|X;\eta,\nu)
    =
    \exp\left(\beta^{\top}X\right) f\left(Y |  h(X,\eta), \nu\right).
$

We expand around the reference parameters $\beta_n^*, \eta_n^*,  \nu_n^*$, 
where the parameter differences are given by
$
    \Delta \eta_n = \eta_n - \eta_0, 
    \Delta \nu_n = \nu_n - \nu_0,
$ and $
    \Delta \eta_n^* = \eta_n^* - \eta_0, 
    \Delta \nu_n^* = \nu_n^* - \nu_0 .   
$
Applying a second-order Taylor expansion, then we obtain:
\begin{align}
    {\Ione}_{n}&=\exp(\tau_n)
    \Big[
    \sum_{|\alpha|=1}^2\frac{1}{\alpha!}
    \prod_{u=1}^d
    [(\beta_n- \beta^*_n)^{(u)}]^{\alpha_{1u}}
    \prod_{v=1}^q
    [(\Delta\eta_n- \Delta\eta^*_n)^{(v)}]^{\alpha_{2v}}
    (\Delta\nu_n- \Delta\nu^*_n)^{\alpha_3}
    \nonumber
    \\&
    \hspace{2.5cm}
    \cdot\frac{\partial^{|\alpha|}g}{\partial \beta^{\alpha_1}\partial \eta ^{\alpha_2}\partial\nu^{\alpha_3}}
    (Y|X;\beta_n^*,\eta_n^*,\nu_n^*)+R_1(X,Y)
    \Big]\nonumber\\
    &=\exp(\tau_n)\Big[\sum_{|\alpha|=1}^2\frac{1}{\alpha!2^{\alpha_3}}
    \prod_{u=1}^d
    [(\beta_n- \beta^*_n)^{(u)}]^{\alpha_{1u}}
    \prod_{v=1}^q
    [(\Delta\eta_n- \Delta\eta^*_n)^{(v)}]^{\alpha_{2v}}
    (\Delta\nu_n- \Delta\nu^*_n)^{\alpha_3}
    \nonumber
    \\&
    \hspace{2.5cm}
    \cdot
    \exp((\betasn)^{\top}X)
    \cdot
    X^{\alpha_1}
    \frac{\partial^{|\alpha_2|} h}{\partial \eta^{|\alpha_2|}}(X,\etasn)
    \frac{\partial^{|\alpha_2|+2\alpha_3}f}{\partial h ^{|\alpha_2|+2\alpha_3}}
    (Y|h\left( X,\eta_n^*\right),\nu_n^*)
    +R_1(X,Y)
    \Big],
\end{align}
% \begin{align*}
% % \frac{\partial^3 g}{\partial \eta^{(i)} \partial \eta^{(j)} \partial \nu}
% % &= \frac{\partial^3 (e \cdot f)}{\partial \eta^{(i)} \partial \eta^{(j)} \partial \nu} \\
% % &= \frac{\partial^2 e}{\partial \eta^{(i)} \partial \nu} \cdot \frac{\partial f}{\partial \eta^{(j)}} \\
% % \\
% \frac{\partial g}{\partial \beta^{(u)} \partial \eta^{(i)}}
% &= \frac{\partial e}{\partial \beta^{(u)}} \cdot \frac{\partial f}{\partial \eta^{(i)}} \\
% &= X^{(u)} \cdot e \cdot \frac{\partial f}{\partial \eta^{(i)}} \\
% &= e \cdot \left( X^{(u)} \frac{\partial f}{\partial \eta^{(i)}} \right) \\
% &= e \cdot \frac{\partial f}{\partial \eta^{(j)}} \\
% &= \frac{\partial g}{\partial \eta^{(j)}}
% \end{align*}
where $ R_1(X,Y) $ is the remainder term containing higher-order terms, and the second equality is due to $\frac{\partial f}{\partial \nu}=\frac{1}{2} \frac{\partial^2 f}{\partial h^2}$. Similarly, we will have that
\begin{align*}
    \Itwo_{n} &= \exp(\tau_n) \Big[ \sum_{|\alpha|=1}^2 \frac{1}{\alpha!} 
    \prod_{u=1}^d[(\beta_n- \beta^*_n)^{(u)}]^{\alpha_{1u}}
    \prod_{v=1}^q[(\Delta\eta^*_n)^{(v)}]^{\alpha_{2v}}
    (\Delta\nu^*_n)^{\alpha_3}
    \\&
    \hspace{5.85cm}
    \cdot
    \frac{\partial^{|\alpha|} g}{\partial \beta^{\alpha_1} \partial \eta^{\alpha_2} \partial \nu^{\alpha_3}}(Y|X;\beta_n^*,\eta_n^*,\nu_n^*) + R_2(X,Y) \Big], 
    \\
    \Ithree_{n} &= \exp(\tau_n^*) \Big[ \sum_{|\alpha|=1}^2 \frac{1}{\alpha!}  
    \prod_{v=1}^q[(\Delta\eta^*_n)^{(v)}]^{\alpha_{2v}}
    (\Delta\nu^*_n)^{\alpha_3}
    \frac{\partial^{|\alpha|} g}{ \partial \eta^{\alpha_2} \partial \nu^{\alpha_3}}(Y|X;\beta_n^*,\eta_n^*,\nu_n^*) + R_3(X,Y) \Big],
    \\
    \Ifour_{n} &= \exp(\tau_n^* + \tau_n)
    \exp\left( (\beta_n^* + \beta_n)^{\top}X\right)
    \Big[ \sum_{|\alpha|=1}^2 \frac{1}{\alpha!} 
    \prod_{v=1}^q
    [(\Delta\eta_n- \Delta\eta^*_n)^{(v)}]^{\alpha_{2v}}
    (\Delta\nu_n- \Delta\nu^*_n)^{\alpha_3}
    \\&
    \hspace{7.15cm}
    \cdot
    \frac{\partial^{|\alpha|} f}{\partial \eta^{\alpha_2} \partial \nu^{\alpha_3}}(Y|X;\eta_n^*,\nu_n^*) + R_4(X,Y) \Big].
\end{align*}
Then, grouping the terms according to the order of derivative $\gamma := |\alpha_2| + 2\alpha_3$ and the monomial degree $\zeta := |\alpha_1|$, we can rewrite the expansion in the compact form:
\begin{align*}
    \Ione_{n}&=
    \sum_{\zeta=0}^2
    \left[
    \sum_{\gamma=0}^4\Ione_{n,\gamma,\zeta}(X)\frac{\partial^{\gamma} f}{\partial h^{\gamma}}(Y|h(X,\eta_n^*),\nu_n^*)
    \exp((\betasn)^{\top}X)
    \right]X^{\zeta}
    +R_1(X,Y) 
\end{align*}
where each coefficient $\mathsf{I}_{n, \gamma, \zeta}(X)$ depends on the parameter differences and derivatives of $h$ with respect to $\eta$.
More specifically we have that 
\begin{align*}
&\Ione_{n,0,1}(X)=\exp(\tau_n)
\sum_{1\leq w \leq d}
(\betan-\betasn)^{(w)}
\\
&\Ione_{n,0,2}(X)=
\exp(\tau_n)
\sum_{1\leq w,r \leq d}
\frac{(\betan-\betasn)^{(w)}(\betan-\betasn)^{(r)}}{1+\mathbf{1}_{w=r}} 
\\
&\Ione_{n,1,0}(X)=\exp(\tau_n)
    \Big[\sum_{u=1}^{q}\{(\Delta\etan-\Delta\etasn)^{(u)}\}\frac{\partial  h}{\partial \eta^{(u)}}(X,\etasn)
    \\&\hspace{2.8cm}
    +\sum_{1\leq u,v\leq q}\frac{(\Delta \etan-\Delta \etasn)^{(u)}(\Delta \etan-\Delta\etasn)^{(v)}}{1+\mathbf{1}_{u=v}}\frac{\partial^2 h}{\partial \eta^{(u)}\partial \eta^{(v)}}(X,\etasn)
    \Big],
\\
&\Ione_{n,1,1}(X)=\exp(\tau_n)
    \Big[\sum_{1\leq w\leq d, 1\leq u\leq q}{[(\betan-\betasn)^{(w)}][(\Delta\etan-\Delta\etasn)^{(u)}]}\frac{\partial h}{\partial \eta^{(u)}}(X,\etasn)
    \Big],
    \\
    &\Ione_{n,2,0}(X)=\exp(\tau_n)
    \Big[
    \frac{1}{2}
    (\Delta\nun-\Delta\nusn)
    +
    \\&\hspace{3.25cm}
    \sum_{1\leq u,v\leq q}\frac{(\Delta \etan-\Delta \etasn)^{(u)}(\Delta \etan-\Delta\etasn)^{(v)}}{1+\mathbf{1}_{u=v}}
    \frac{\partial h}{\partial \eta^{(u)}}(X,\etasn)
    \frac{\partial h}{\partial \eta^{(v)}}(X,\etasn)
    \Big],
    \\
    &\Ione_{n,2,1}(X)=\frac{\exp(\tau_n)}{2}
    \Big[
    \sum_{1\leq w\leq d, 1\leq u\leq q}
    {(\betan-\betasn)^{(w)}(\Delta\nun-\Delta\nusn)^{(u)}}
    \Big],
    \\
    &\Ione_{n,3,0}(X)=\frac{\exp(\tau_n)}{2}
    \Big[
    \sum_{u=1}^{q}
    (\Delta\etan-\Delta\etasn)^{(u)}
    (\Delta\nun-\Delta\nusn)^{}
    \frac{\partial h}{\partial \eta^{(u)}}(X,\etasn) 
    \Big],\\
    &\Ione_{n,4,0}(X)=\frac{\exp(\tau_n)}{8}
    (\Delta\nun-\Delta\nusn)^2.
\end{align*}

Similarly, we can rewrite $\Itwo_{n}$ in the same fashion as follows:
\begin{align*}
    \Itwo_{n}&=
    \sum_{\zeta=0}^2
    \left[
    \sum_{\gamma=0}^4\Itwo_{n,\gamma,\zeta}(X)\frac{\partial^{\gamma} f}{\partial h^{\gamma}}(Y|h(X,\eta_n^*),\nu_n^*)
    \exp((\betasn)^{\top}X)
    \right]X^{\zeta}
    +R_2(X,Y) 
\end{align*}
where 
\begin{align*}
&\Itwo_{n,0,1}(X)=\exp(\tau_n)
\sum_{1\leq w \leq d}
(\betan-\betasn)^{(w)}
\\
&\Itwo_{n,0,2}(X)=
\exp(\tau_n)
\sum_{1\leq w,r \leq d}
\frac{(\betan-\betasn)^{(w)}(\betan-\betasn)^{(r)}}{1+\mathbf{1}_{w=r}} 
\\
&\Itwo_{n,1,0}(X)=\exp(\tau_n)
    \Big[\sum_{u=1}^{q}\{(-\Delta\etasn)^{(u)}\}\frac{\partial  h}{\partial \eta^{(u)}}(X,\etasn)
    \\&\hspace{4.8cm}
    +\sum_{1\leq u,v\leq q}\frac{( -\Delta \etasn)^{(u)}( -\Delta\etasn)^{(v)}}{1+\mathbf{1}_{u=v}}\frac{\partial^2 h}{\partial \eta^{(u)}\partial \eta^{(v)}}(X,\etasn)
    \Big],
\\
&\Itwo_{n,1,1}(X)=\exp(\tau_n)
    \Big[\sum_{1\leq w\leq d, 1\leq u\leq q}{[(\betan-\betasn)^{(w)}][(-\Delta\etasn)^{(u)}]}\frac{\partial h}{\partial \eta^{(u)}}(X,\etasn)
    \Big],
\\
&\Itwo_{n,2,0}(X)=\exp(\tau_n)
    \Big[
    \frac{1}{2}
    (-\Delta\nusn)
    +\sum_{1\leq u,v\leq q}\frac{( -\Delta \etasn)^{(u)}( -\Delta\etasn)^{(v)}}{1+\mathbf{1}_{u=v}}
    \frac{\partial h}{\partial \eta^{(u)}}(X,\etasn)
    \frac{\partial h}{\partial \eta^{(v)}}(X,\etasn)
    \Big],
\\
&\Itwo_{n,2,1}(X)=\frac{\exp(\tau_n)}{2}
    \Big[
    \sum_{1\leq w\leq d, 1\leq u\leq q}
    {(\betan-\betasn)^{(w)}(-\Delta\nusn)^{(u)}}
    \Big],
\\
&\Itwo_{n,3,0}(X)=\frac{\exp(\tau_n)}{2}
    \Big[
    \sum_{u=1}^{q}
    (-\Delta\etasn)^{(u)}
    (-\Delta\nusn)^{}
    \frac{\partial h}{\partial \eta^{(u)}}(X,\etasn) 
    \Big],
\\
&\Itwo_{n,4,0}(X)=\frac{\exp(\tau_n)}{8}
    (-\Delta\nusn)^2.
\end{align*}

In the same way, we can rewrite $\Ithree_{n}$ in the same fashion as follows, here the difference for $\beta_n^*$ is zero, so all the coefficients with $\zeta\neq0$ is zero, but in order for the alignment of the expression, we will still express $\Ithree_{n}$ as follows
\begin{align*}
    \Ithree_{n}&=
    % \sum_{\zeta=0}^2
    % \left[
    \sum_{\gamma=1}^4\Ithree_{n,\gamma,0}(X)\frac{\partial^{\gamma} f}{\partial h^{\gamma}}(Y|h(X,\eta_n^*),\nu_n^*)
    \exp((\betasn)^{\top}X)
    % \right]X^{\zeta}
    +R_2(X,Y) 
\end{align*}
where 
\begin{align*}
&\Ithree_{n,1,0}(X)=\exp(\tau^*_n)
    \Big[\sum_{u=1}^{q}\{(-\Delta\etasn)^{(u)}\}\frac{\partial  h}{\partial \eta^{(u)}}(X,\etasn)
    % \\&\hspace{2.8cm}
    \\&\hspace{4.8cm}
    +\sum_{1\leq u,v\leq q}\frac{( -\Delta \etasn)^{(u)}( -\Delta\etasn)^{(v)}}{1+\mathbf{1}_{u=v}}\frac{\partial^2 h}{\partial \eta^{(u)}\partial \eta^{(v)}}(X,\etasn)
    \Big],
\\
&\Ithree_{n,2,0}(X)=\exp(\tau^*_n)
    \Big[
    \frac{1}{2}
    (-\Delta\nusn)
    +\sum_{1\leq u,v\leq q}\frac{( -\Delta \etasn)^{(u)}( -\Delta\etasn)^{(v)}}{1+\mathbf{1}_{u=v}}
    \frac{\partial h}{\partial \eta^{(u)}}(X,\etasn)
    \frac{\partial h}{\partial \eta^{(v)}}(X,\etasn)
    \Big],
\\
&\Ithree_{n,3,0}(X)=\frac{\exp(\tau^*_n)}{2}
    \Big[
    \sum_{u=1}^{q}
    (-\Delta\etasn)^{(u)}
    (-\Delta\nusn)^{}
    \frac{\partial h}{\partial \eta^{(u)}}(X,\etasn) 
    \Big],
\\
&\Ithree_{n,4,0}(X)=\frac{\exp(\tau^*_n)}{8}
    (-\Delta\nusn)^2.
\end{align*}

Now we consider 
$\Ifour_n=\exp\left( (\beta_n^* + \beta_n)^{\top}X + \tau_n^* + \tau_n \right)
    \cdot[f(Y|\sigma(X,\eta_n),\nu_n)-f(Y|\sigma(X,\eta_n^*),\nu_n^*)]$,
which is equivalent to 
\begin{align*}
    \Ifour_{n}&=
    \sum_{\gamma=1}^4\Ifour_{n,\gamma,0}(X)\frac{\partial^{\gamma} f}{\partial h^{\gamma}}(Y|h(X,\eta_n^*),\nu_n^*)
    \exp((\betasn)^{\top}X)
    \exp((\betan)^{\top}X)
    +R_4(X,Y) 
\end{align*}
where
\begin{align*}
&\Ifour_{n,1,0}(X)=\exp(\tau^*_n+\tau_n)
    \Big[\sum_{u=1}^{q}\{(\Delta\etan-\Delta\etasn)^{(u)}\}\frac{\partial  h}{\partial \eta^{(u)}}(X,\etasn)
    \\&\hspace{4.3cm}
    +\sum_{1\leq u,v\leq q}\frac{( \Delta\etan-\Delta \etasn)^{(u)}(\Delta\etan -\Delta\etasn)^{(v)}}{1+\mathbf{1}_{u=v}}\frac{\partial^2 h}{\partial \eta^{(u)}\partial \eta^{(v)}}(X,\etasn)
    \Big],
\\
&\Ifour_{n,2,0}(X)=\exp(\tau_n^*+\tau_n)
    \Big[
    \frac{1}{2}
    (\Delta\nun-\Delta\nusn)
    \\&\hspace{3.3cm}
    +\sum_{1\leq u,v\leq q}\frac{(\Delta\etan -\Delta \etasn)^{(u)}( \Delta\etan-\Delta\etasn)^{(v)}}{1+\mathbf{1}_{u=v}}
    \frac{\partial h}{\partial \eta^{(u)}}(X,\etasn)
    \frac{\partial h}{\partial \eta^{(v)}}(X,\etasn)
    \Big],
\\
&\Ifour_{n,3,0}(X)=\frac{\exp(\tau_n^*+\tau_n)}{2}
    \Big[
    \sum_{u=1}^{q}
    (\Delta\etan-\Delta\etasn)^{(u)}
    (\Delta\nun-\Delta\nusn)^{}
    \frac{\partial h}{\partial \eta^{(u)}}(X,\etasn) 
    \Big],
\\
&\Ifour_{n,4,0}(X)=\frac{\exp(\tau_n)}{8}
    (\Delta\nun-\Delta\nusn)^2.
\end{align*}
Then we could conclude that 
\begin{align*}
    W_{n}&=
    \sum_{\gamma=0}^4
    \Bigg[
    \left(
    \Ione_{n,\gamma,0}(X)+
    \Itwo_{n,\gamma,0}(X)+
    \Ithree_{n,\gamma,0}(X)
    \right)
    \\&\hspace{1.8cm}
    +
    \sum_{\zeta=1}^2
    \left(
    \Ione_{n,\gamma,\zeta}(X)+
    \Itwo_{n,\gamma,\zeta}(X)
    \right)
    X^{\zeta}
    +
    \Ifour_{n,\gamma,0}(X)\exp((\betan)^{\top}X)
    \Bigg]
    \\&
    \hspace{3cm}
    \cdot
     \frac{\partial^{\gamma} f}{\partial h^{\gamma}}(Y|h(X,\eta_n^*),\nu_n^*)
     \cdot
    \exp((\betasn)^{\top}X) .
\end{align*}
Therefore, we can view the quantity $W_n/\dtwo)$ as a linear combination of elements of 
the set $\mathcal{L}\cup\mathcal{K}$, and $\mathcal{L}=\cup_{\gamma=0}^4\cup_{\zeta=0}^2\mathcal{L}_{\gamma,\zeta}$ , $\mathcal{K}=\cup_{\gamma=1}^4\mathcal{K}_{\gamma}$, where
\begin{align*}
    \mathcal{L}_{0,1}&=\left\{
    Xf(\yhns)\exp((\betasn)^{\top}X) 
    \right\}
    \\
    \mathcal{L}_{0,2}&=\left\{
    XX^{\top}f(\yhns)\exp((\betasn)^{\top}X) 
    \right\}
    \\
    \mathcal{L}_{1,1}&=\left\{
    \frac{\partial h}{\partial \eta^{(u)}}(X,\etasn)
    X\frac{\partial f}{\partial  h}(\yhns)\exp((\betasn)^{\top}X) 
    :u\in[q]
    \right\}
    \\
    \mathcal{L}_{2,1}&=\left\{
    X\frac{\partial^2 f}{\partial  h^2}(\yhns)\exp((\betasn)^{\top}X) 
    \right\}
\\
    \mathcal{L}_{1,0}&=\left\{\frac{\partial  h}{\partial \eta^{(u)}}(\xetas)\frac{\partial f}{\partial  h}(\yhns) 
    \exp((\betasn)^{\top}X)
    :u\in[d]\right\}\\
    &\cup\left\{\frac{\partial^2  h}{\partial \eta^{(u)}\partial \eta^{(v)}}(\xetas)\frac{\partial f}{\partial  h}(\yhns)\exp((\betasn)^{\top}X) :u,v\in[d]\right\},\\
    \mathcal{L}_{2,0}
    &=\left\{
    \frac{\partial^2f}{\partial  h^2}(\yhns)\exp((\betasn)^{\top}X)  \right\}\\
    &\cup\left\{
    \frac{\partial  h}{\partial \eta^{(u)} }(\xetas)
    \frac{\partial  h}{\partial \eta^{(v)} }(\xetas)
    \frac{\partial^2f}{\partial  h^2}(\yhns)
    \exp((\betasn)^{\top}X)
    :u,v\in[q]\right\}\\
    \mathcal{L}_{3,0}&=
    \left\{\frac{\partial  h}{\partial \eta^{(u)}}(\xetas)
    \frac{\partial^3f}{\partial  h^3}(\yhns)
    \exp((\betasn)^{\top}X)
    :u\in[d] \right\}
    \\
    \mathcal{L}_{4,0}&=\left\{
    \frac{\partial^4 f}{\partial  h^4}(\yhns)
    \exp((\betasn)^{\top}X)
    \right\},
\end{align*}
and
\begin{align*}
    \mathcal{K}_{1 }&=\left\{\frac{\partial  h}{\partial \eta^{(u)}}(\xetas)
    \exp((\betan)^{\top}X)
    \frac{\partial f}{\partial  h}(\yhns) 
    \exp((\betasn)^{\top}X)
    :u\in[d]\right\}\\
    &\cup\left\{\frac{\partial^2  h}{\partial \eta^{(u)}\partial \eta^{(v)}}(\xetas)
    \exp((\betan)^{\top}X)
    \frac{\partial f}{\partial  h}(\yhns)
    \exp((\betasn)^{\top}X) :u,v\in[d]\right\},\\
    \mathcal{K}_{2 }
    &=\left\{
    \exp((\betan)^{\top}X)
    \frac{\partial^2f}{\partial  h^2}(\yhns)
    \exp((\betasn)^{\top}X)  \right\}\\
    &\cup\left\{
    \frac{\partial  h}{\partial \eta^{(u)} }(\xetas)
    \frac{\partial  h}{\partial \eta^{(v)} }(\xetas)
    \exp((\betan)^{\top}X)
    \frac{\partial^2f}{\partial  h^2}(\yhns)
    \exp((\betasn)^{\top}X)
    :u,v\in[q]\right\},\\
    \mathcal{K}_{3 }&=
    \left\{\frac{\partial  h}{\partial \eta^{(u)}}(\xetas)
    \exp((\betan)^{\top}X)
    \frac{\partial^3f}{\partial  h^3}(\yhns)
    \exp((\betasn)^{\top}X)
    :u\in[d] \right\},
    \\
    \mathcal{K}_{4 }&=\left\{
    \exp((\betan)^{\top}X)
    \frac{\partial^4 f}{\partial  h^4}(\yhns)
    \exp((\betasn)^{\top}X)
    \right\}.    
\end{align*}
Assume by contrary that all the coefficients of these elements vanish when $n\to\infty$. 
Looking at the coefficients of 
$ \frac{\partial h}{\partial \eta^{(u)}}(X,\etasn)
    X\frac{\partial f}{\partial  h}(\yhns)\exp((\betasn)^{\top}X) $,
we get for all $w\in[d], u\in[q]$
\begin{align}
\label{pfeq:d2term_beta_eta}
\exp(\tau_n)
{[(\betan-\betasn)^{(w)}][(\Delta\etan)^{(u)}]}
    /\dtwo\to0,
\end{align}
% \begin{align*}
% &\Ione_{n,1,1}(X)=\exp(\tau_n)
%     \Big[\sum_{1\leq w\leq d, 1\leq u\leq q}{[(\betan-\betasn)^{(w)}][(\Delta\etan-\Delta\etasn)^{(u)}]}\frac{\partial h}{\partial \eta^{(u)}}(X,\etasn)
%     \Big],
%     \\
% &\Itwo_{n,1,1}(X)=\exp(\tau_n)
%     \Big[\sum_{1\leq w\leq d, 1\leq u\leq q}{[(\betan-\betasn)^{(w)}][(-\Delta\etasn)^{(u)}]}\frac{\partial h}{\partial \eta^{(u)}}(X,\etasn)
%     \Big],
% \end{align*}
Looking at the coefficients of 
$X\frac{\partial^2 f}{\partial  h^2}(\yhns)\exp((\betasn)^{\top}X) $,
we get for all $w\in[d]$
\begin{align}
\label{pfeq:d2term_beta_nu}
\exp(\tau_n)
{[(\betan-\betasn)^{(w)}](\Delta\nun)}
    /\dtwo\to0,
\end{align}
% \begin{align*}
% &\Ione_{n,2,1}(X)=\frac{\exp(\tau_n)}{2}
%     \Big[
%     \sum_{1\leq w\leq d, 1\leq u\leq q}
%     {(\betan-\betasn)^{(w)}(\Delta\nun-\Delta\nusn)^{(u)}}
%     \Big],
%     \\
% &\Itwo_{n,2,1}(X)=\frac{\exp(\tau_n)}{2}
%     \Big[
%     \sum_{1\leq w\leq d, 1\leq u\leq q}
%     {(\betan-\betasn)^{(w)}(-\Delta\nusn)^{(u)}}
%     \Big],
% \end{align*}
Looking at the coefficients of 
$\dfrac{\partial^2  h}{\partial \eta^{(u)}\partial \eta^{(v)}}(\xetas)\dfrac{\partial f}{\partial  h}(\yhns)\exp((\betasn)^{\top}X) $ 
, 
we get for all $u,v\in[q]$,
\begin{align}
\label{pf:d2_coefficients_1}
[\exp(\tau_n)
{(\Delta \etan-\Delta \etasn)^{(u)}(\Delta \etan-\Delta\etasn)^{(v)}}
+[\exp(\tau^*_n)-\exp(\tau_n)]
( -\Delta \etasn)^{(u)}( -\Delta\etasn)^{(v)}]
\nonumber\\    /\dtwo \to 0,
\end{align}
Looking at the coefficients of 
$\dfrac{\partial  h}{\partial \eta^{(u)}}(\xetas)\dfrac{\partial f}{\partial  h}(\yhns)\exp((\betasn)^{\top}X) $ 
, 
we get for all $u \in[q]$,
\begin{align}
\label{pf:d2_coefficients_2}
    [\exp(\tau_n)(\Delta\etan-\Delta\etasn)^{(u)}
+[\exp(\tau^*_n)-\exp(\tau_n)]
    (-\Delta\etasn)^{(u)}]
\nonumber\\    /\dtwo \to 0,
\end{align}
% \begin{align*}
% &\Ione_{n,1,0}(X)=\exp(\tau_n)
%     \Big[\sum_{u=1}^{q}\{(\Delta\etan-\Delta\etasn)^{(u)}\}\frac{\partial  h}{\partial \eta^{(u)}}(X,\etasn)
%     \\&\hspace{2.8cm}
%     +\sum_{1\leq u,v\leq q}\frac{(\Delta \etan-\Delta \etasn)^{(u)}(\Delta \etan-\Delta\etasn)^{(v)}}{1+\mathbf{1}_{u=v}}\frac{\partial^2 h}{\partial \eta^{(u)}\partial \eta^{(v)}}(X,\etasn)
%     \Big],
% \\
% &\Itwo_{n,1,0}(X)=\exp(\tau_n)
%     \Big[\sum_{u=1}^{q}\{(-\Delta\etasn)^{(u)}\}\frac{\partial  h}{\partial \eta^{(u)}}(X,\etasn)
%     % \\&\hspace{2.8cm}
%     +\sum_{1\leq u,v\leq q}\frac{( -\Delta \etasn)^{(u)}( -\Delta\etasn)^{(v)}}{1+\mathbf{1}_{u=v}}\frac{\partial^2 h}{\partial \eta^{(u)}\partial \eta^{(v)}}(X,\etasn)
%     \Big],
% \\
% &\Ithree_{n,1,0}(X)=\exp(\tau^*_n)
%     \Big[\sum_{u=1}^{q}\{(-\Delta\etasn)^{(u)}\}\frac{\partial  h}{\partial \eta^{(u)}}(X,\etasn)
%     % \\&\hspace{2.8cm}
%     +\sum_{1\leq u,v\leq q}\frac{( -\Delta \etasn)^{(u)}( -\Delta\etasn)^{(v)}}{1+\mathbf{1}_{u=v}}\frac{\partial^2 h}{\partial \eta^{(u)}\partial \eta^{(v)}}(X,\etasn)
%     \Big],
% \end{align*}
Looking at the coefficients of 
$\dfrac{\partial^2 f}{\partial h^2}(\yhns)\exp((\betasn)^{\top}X) $ , we get 
% for all $u,v\in[q]$,
\begin{align}
\label{pf:d2_coefficients_3}
[\exp(\tau_n)
(\Delta\nun-\Delta\nusn)
+[\exp(\tau^*_n)-\exp(\tau_n)]
(-\Delta\nusn)]
    /\dtwo&\to 0,
\end{align}
Looking at the coefficients of 
$\dfrac{\partial h}{\partial \eta^{(u)}}(X,\etasn)
    \dfrac{\partial h}{\partial \eta^{(v)}}(X,\etasn)
    \dfrac{\partial^2 f}{\partial h^2}(\yhns)\exp((\betasn)^{\top}X) $ , we get 
for all $u,v\in[q]$,
\begin{align}
\label{pf:d2_coefficients_4}
[\exp(\tau_n)
(\Delta \etan-\Delta \etasn)^{(u)}(\Delta \etan-\Delta\etasn)^{(v)}
+[\exp(\tau^*_n)-\exp(\tau_n)]
( -\Delta \etasn)^{(u)}( -\Delta\etasn)^{(v)}] 
 \nonumber\\   /\dtwo\to 0,
\end{align}
% \begin{align*}
% &\Ione_{n,2,0}(X)=\exp(\tau_n)
%     \Big[
%     \frac{1}{2}
%     (\Delta\nun-\Delta\nusn)
%     +\sum_{1\leq u,v\leq q}\frac{(\Delta \etan-\Delta \etasn)^{(u)}(\Delta \etan-\Delta\etasn)^{(v)}}{1+\mathbf{1}_{u=v}}
%     \frac{\partial h}{\partial \eta^{(u)}}(X,\etasn)
%     \frac{\partial h}{\partial \eta^{(v)}}(X,\etasn)
%     \Big],
%     \\
% &\Itwo_{n,2,0}(X)=\exp(\tau_n)
%     \Big[
%     \frac{1}{2}
%     (-\Delta\nusn)
%     +\sum_{1\leq u,v\leq q}\frac{( -\Delta \etasn)^{(u)}( -\Delta\etasn)^{(v)}}{1+\mathbf{1}_{u=v}}
%     \frac{\partial h}{\partial \eta^{(u)}}(X,\etasn)
%     \frac{\partial h}{\partial \eta^{(v)}}(X,\etasn)
%     \Big],
% \\
% &\Ithree_{n,2,0}(X)=\exp(\tau^*_n)
%     \Big[
%     \frac{1}{2}
%     (-\Delta\nusn)
%     +\sum_{1\leq u,v\leq q}\frac{( -\Delta \etasn)^{(u)}( -\Delta\etasn)^{(v)}}{1+\mathbf{1}_{u=v}}
%     \frac{\partial h}{\partial \eta^{(u)}}(X,\etasn)
%     \frac{\partial h}{\partial \eta^{(v)}}(X,\etasn)
%     \Big],
% \end{align*}
Looking at the coefficients of 
$\dfrac{\partial h}{\partial \eta^{(u)}}(X,\etasn)
    \dfrac{\partial^3 f}{\partial h^3}(\yhns)\exp((\betasn)^{\top}X) $ , we get 
for all $u\in[q]$,
\begin{align}
\label{pf:d2_coefficients_5}
[\exp(\tau_n)
(\Delta\etan-\Delta\etasn)^{(u)}(\Delta\nun-\Delta\nusn)
+[\exp(\tau^*_n)-\exp(\tau_n)]
(-\Delta\etasn)^{(u)}(-\Delta\nusn)]
 \nonumber\\    /\dtwo&\to 0,
\end{align}
% \begin{align*}
% &\Ione_{n,3,0}(X)=\frac{\exp(\tau_n)}{2}
%     \Big[
%     \sum_{u=1}^{q}
%     (\Delta\etan-\Delta\etasn)^{(u)}
%     (\Delta\nun-\Delta\nusn)^{}
%     \frac{\partial h}{\partial \eta^{(u)}}(X,\etasn) 
%     \Big],\\
% &\Itwo_{n,3,0}(X)=\frac{\exp(\tau_n)}{2}
%     \Big[
%     \sum_{u=1}^{q}
%     (-\Delta\etasn)^{(u)}
%     (-\Delta\nusn)^{}
%     \frac{\partial h}{\partial \eta^{(u)}}(X,\etasn) 
%     \Big],
% \\
% &\Ithree_{n,3,0}(X)=\frac{\exp(\tau^*_n)}{2}
%     \Big[
%     \sum_{u=1}^{q}
%     (-\Delta\etasn)^{(u)}
%     (-\Delta\nusn)^{}
%     \frac{\partial h}{\partial \eta^{(u)}}(X,\etasn) 
%     \Big],
% \end{align*}
Looking at the coefficients of 
$
    \dfrac{\partial^4 f}{\partial h^4}(\yhns)\exp((\betasn)^{\top}X) $ , we get 
% for all $u\in[q]$,
\begin{align}
\label{pf:d2_coefficients_6}
[\exp(\tau_n)
(\Delta\nun-\Delta\nusn)^2
+[\exp(\tau^*_n)-\exp(\tau_n)]
(-\Delta\nusn)^2]
    /\dtwo&\to 0,
\end{align}
% \begin{align*}
% &\Ione_{n,4,0}(X)=\frac{\exp(\tau_n)}{8}
%     (\Delta\nun-\Delta\nusn)^2.
% \\
% &\Itwo_{n,4,0}(X)=\frac{\exp(\tau_n)}{8}
%     (-\Delta\nusn)^2.
% \\
% &\Ithree_{n,4,0}(X)=\frac{\exp(\tau^*_n)}{8}
%     (-\Delta\nusn)^2.
% \end{align*}
\begin{comment}
\begin{align*}
    D_2((\tau,G),(\tau^*,G_*)) :=
    &\exp(\tau)|\beta-\beta^*|
    \Vert (\Delta\eta,\Delta\nu) \Vert
    \\
    +&
    \exp(\tau)\cdot
    \Vert (\Delta\eta,\Delta\nu)-(\Delta\eta^*,\Delta\nu^*)\Vert^2
    \\
    +&
    \left|
    \exp(\tau^*)-\exp(\tau)
    \right|
    \cdot
    \Vert (\Delta\eta,\Delta\nu)\Vert
    \cdot
    \Vert (\Delta\eta^*,\Delta\nu^*)\Vert
    \\
    +&
    \exp(\tau+\tau^*)\cdot
    \Vert (\Delta\eta,\Delta\nu)-(\Delta\eta^*,\Delta\nu^*)\Vert^2
\end{align*}
\begin{align*}
    D_2((\tau,G),(\tau^*,G_*)) :=
    &\exp(\tau)|\beta-\beta^*|
    \Vert (\Delta\eta,\Delta\nu) \Vert
    +
    \exp(\tau+\tau^*)|\beta|
    \Vert (\Delta\eta,\Delta\nu)-(\Delta\eta^*,\Delta\nu^*) \Vert
    \\
    +&
    \exp(\tau)\cdot
    \Vert (\Delta\eta,\Delta\nu)-(\Delta\eta^*,\Delta\nu^*)\Vert^2
    \\
    +&
    \left|
    \exp(\tau^*)-\exp(\tau)
    \right|
    \cdot
    \Vert (\Delta\eta,\Delta\nu)\Vert
    \cdot
    \Vert (\Delta\eta^*,\Delta\nu^*)\Vert
    \\
    +&
    \exp(\tau+\tau^*)\cdot
    \Vert (\Delta\eta,\Delta\nu)-(\Delta\eta^*,\Delta\nu^*)\Vert^2
\end{align*}
\end{comment}
Looking at the coefficients of 
$
    \dfrac{\partial  h}{\partial \eta^{(u)}}(X,\etasn)
    \exp((\betan)^{\top}X)
    \dfrac{\partial f}{\partial h}(\yhns)\exp((\betasn)^{\top}X) $ , we get 
for all $u\in[q]$,
\begin{align}
\label{pf:d2_coefficients_7}
[\exp(\tau^*_n+\tau_n)
(\Delta\etan-\Delta\etasn)^{(u)}]
    /\dtwo&\to 0,
\end{align}
Looking at the coefficients of 
$
    \dfrac{\partial^2 h}{\partial \eta^{(u)}\partial \eta^{(v)}}(X,\etasn)
    \exp((\betan)^{\top}X)
    \dfrac{\partial f}{\partial h}(\yhns)\exp((\betasn)^{\top}X) $ , we get 
for all $u,v\in[q]$,
\begin{align}
\label{pf:d2_coefficients_8}
[\exp(\tau^*_n+\tau_n)
( \Delta\etan-\Delta \etasn)^{(u)}(\Delta\etan -\Delta\etasn)^{(v)}]
    /\dtwo&\to 0,
\end{align}
Looking at the coefficients of 
$
    \exp((\betan)^{\top}X)
    \dfrac{\partial^2 f}{\partial h^2}(\yhns)\exp((\betasn)^{\top}X) $ , we get 
% for all $u,v\in[q]$,
\begin{align}
\label{pf:d2_coefficients_9}
[\exp(\tau^*_n+\tau_n)
(\Delta\nun-\Delta\nusn)]
    /\dtwo&\to 0,
\end{align}
Looking at the coefficients of 
$
\frac{\partial h}{\partial \eta^{(u)}}(X,\etasn)
    \frac{\partial h}{\partial \eta^{(v)}}(X,\etasn)
    \exp((\betan)^{\top}X)
    \dfrac{\partial^2 f}{\partial h^2}(\yhns)\exp((\betasn)^{\top}X) $ , we get 
for all $u,v\in[q]$,
\begin{align}
\label{pf:d2_coefficients_10}
[\exp(\tau^*_n+\tau_n)
(\Delta\etan -\Delta \etasn)^{(u)}( \Delta\etan-\Delta\etasn)^{(v)}]
    /\dtwo&\to 0,
\end{align}
Looking at the coefficients of 
$
\frac{\partial h}{\partial \eta^{(u)}}(X,\etasn)
    \exp((\betan)^{\top}X)
    \dfrac{\partial^3 f}{\partial h^3}(\yhns)\exp((\betasn)^{\top}X) $, we get 
for all $u\in[q]$,
\begin{align}
\label{pf:d2_coefficients_11}
\exp(\tau^*_n+\tau_n)
(\Delta\etan-\Delta\etasn)^{(u)}
(\Delta\nun-\Delta\nusn)
    /\dtwo&\to 0,
\end{align}
Looking at the coefficients of 
$
    \exp((\betan)^{\top}X)
    \dfrac{\partial^4 f}{\partial h^4}(\yhns)\exp((\betasn)^{\top}X) $, we get 
for all $u\in[q]$,
\begin{align}
\label{pf:d2_coefficients_12}
[\exp(\tau^*_n+\tau_n)
(\Delta\nun-\Delta\nusn)^2]
    /\dtwo&\to 0,
\end{align}
% \begin{align*}
% &\Ifour_{n,1,0}(X)=\exp(\tau^*_n+\tau_n)
%     \Big[\sum_{u=1}^{q}\{(\Delta\etan-\Delta\etasn)^{(u)}\}\frac{\partial  h}{\partial \eta^{(u)}}(X,\etasn)
%     \\&\hspace{4.3cm}
%     +\sum_{1\leq u,v\leq q}
%     \frac{( \Delta\etan-\Delta \etasn)^{(u)}(\Delta\etan -\Delta\etasn)^{(v)}}{1+\mathbf{1}_{u=v}}\frac{\partial^2 h}{\partial \eta^{(u)}\partial \eta^{(v)}}(X,\etasn)
%     \Big],
% \\
% &\Ifour_{n,2,0}(X)=\exp(\tau_n^*+\tau_n)
%     \Big[
%     \frac{1}{2}
%     (\Delta\nun-\Delta\nusn)
%     +\sum_{1\leq u,v\leq q}\frac{(\Delta\etan -\Delta \etasn)^{(u)}( \Delta\etan-\Delta\etasn)^{(v)}}{1+\mathbf{1}_{u=v}}
%     \frac{\partial h}{\partial \eta^{(u)}}(X,\etasn)
%     \frac{\partial h}{\partial \eta^{(v)}}(X,\etasn)
%     \Big],
% \\
% &\Ifour_{n,3,0}(X)=\frac{\exp(\tau_n^*+\tau_n)}{2}
%     \Big[
%     \sum_{u=1}^{q}
%     (\Delta\etan-\Delta\etasn)^{(u)}
%     (\Delta\nun-\Delta\nusn)^{}
%     \frac{\partial h}{\partial \eta^{(u)}}(X,\etasn) 
%     \Big],
% \\
% &\Ifour_{n,4,0}(X)=\frac{\exp(\tau_n^*+\tau_n)}{8}
%     (\Delta\nun-\Delta\nusn)^2.
% \end{align*}
% \textcolor{red}{continuing}
% Now consider \eqref{pfeq:d2term_beta_eta} and \eqref{pfeq:d2term_beta_nu}, we will have that 
% \begin{align}   
%     \label{eq:nondistinguishable_independent_1}
%     \exp(\tau_n)
% \|\betan-\Delta \betasn\|
% \|(\Delta\etan,\Delta\nun )\|
%     /\dtwo&\to 0,
% \end{align}
% Now, combining \eqref{pfeq:d2term_beta_eta} and \eqref{pfeq:d2term_beta_nu},
% applying the Cauchy–Schwarz inequality and summing over coordinates, we conclude that
Now, combining \eqref{pfeq:d2term_beta_eta} and \eqref{pfeq:d2term_beta_nu},
recall that all the gating parameters are in compact sets,
and applying the Cauchy–Schwarz inequality followed by summation over coordinates, 
we got that
\begin{align}   
    \label{eq:nondistinguishable_independent_01}
    \exp(\tau_n)
\|\betan-\betasn\|
\|(\Delta\etan,\Delta\nun )\|
    /\dtwo&\to 0.
\end{align}
% By this symmetry in the definition of $ W_n $, we similarly obtain
% \begin{align}   
%     \exp(\tau_n+\tau_n^*)
% \|\betan-\betasn\|
% \|(\Delta\etasn,\Delta\nusn )\|
%     /\dtwo&\to 0.
% \end{align}

{
While it is intuitive that the similar result holds for $\|(\Delta\etasn,\Delta\nusn )\|$, a slightly tricky handle should be employed here. Suppose that 
\begin{align*}
     \exp(\tau_n^*)
\|\betan-\betasn\|
\|\Delta\etasn\|
    /\dtwo \not\to 0.
\end{align*}
By combining this assumption with equation \eqref{pfeq:d2term_beta_eta}, we have there are at least one coordinate $u$ such that $|(\Delta \etasn)^{(u)}/(\Delta \etan)^{(u)}| \to \infty$, which implies that $(\Delta\etasn)/(\Delta \etasn - \Delta \etan)^{(u)} \to 1$. Thus, by multiplying equation \eqref{pf:d2_coefficients_7} with $(\Delta\etasn)/(\Delta \etasn - \Delta \etan)^{(u)} \to 1$, we have 
\begin{align*}
\exp(\tau^*_n)
(\Delta\etasn)^{(u)}
    /\dtwo \to 0.
\end{align*}
Also noting that $\|\betan - \betasn\|$ is bounded as the parameters belongs to a compact set, we have 
\begin{equation*}
    \exp(\tau^*_n)\|\betan-\betasn\|
(\Delta\etasn)^{(u)}
    /\dtwo \to 0,
\end{equation*}
which is a contradiction here. Thus, we have 
\begin{equation}
\
    \exp(\tau_n^*)
\|\betan-\betasn\|
\|\Delta\etasn\|
    /\dtwo\to 0.
\end{equation}
Similarly, also by combining equation \eqref{pfeq:d2term_beta_nu} and \eqref{pf:d2_coefficients_9}, we have 
\begin{equation}
    \exp(\tau_n^*)
\|\betan-\betasn\|
\|\Delta\nusn\|
    /\dtwo\to 0.
\end{equation}
As a result, we have 
\begin{align}
\label{eq:d2proof_new}
    \exp(\tau_n^*)
\|\betan-\betasn\|
\|(\Delta\etasn,\Delta\nusn )\|
    /\dtwo \to 0.
\end{align}
}
% \begin{align}
% \label{pf:d2_coefficients_7}
% \frac{\exp(\tau^*_n+\tau_n)
% (\Delta\etan - \Delta\etasn)^{(u)}}{\dtwo}
% &\to 0, &&\forall u \in [q],
% \\
% \label{pf:d2_coefficients_8}
% \frac{\exp(\tau^*_n+\tau_n)
% (\Delta\etan - \Delta\etasn)^{(u)}(\Delta\etan - \Delta\etasn)^{(v)}}{\dtwo}
% &\to 0, &&\forall u, v \in [q],
% \\
% \label{pf:d2_coefficients_9}
% \frac{\exp(\tau^*_n+\tau_n)
% (\Delta\nun - \Delta\nusn)}{\dtwo}
% &\to 0,
% \\
% \label{pf:d2_coefficients_10}
% \frac{\exp(\tau^*_n+\tau_n)
% (\Delta\etan - \Delta\etasn)^{(u)}(\Delta\etan - \Delta\etasn)^{(v)}}{\dtwo}
% &\to 0, &&\forall u, v \in [q],
% \\
% \label{pf:d2_coefficients_11}
% \frac{\exp(\tau^*_n+\tau_n)
% (\Delta\etan - \Delta\etasn)^{(u)}(\Delta\nun - \Delta\nusn)}{\dtwo}
% &\to 0, &&\forall u \in [q],
% \\
% \label{pf:d2_coefficients_12}
% \frac{\exp(\tau^*_n+\tau_n)
% (\Delta\nun - \Delta\nusn)^2}{\dtwo}
% &\to 0.
% \end{align}
In a similar manner, by considering equations \eqref{pf:d2_coefficients_7} through \eqref{pf:d2_coefficients_12}, 
% and applying analogous reasoning as above, 
we obtain that
\begin{align}
\label{eq:nondistinguishable_independent_02}
    \exp(\tau_n+\tau_n^*)\cdot
    \Vert (\Delta\etan,\Delta\nun)-(\Delta\eta^*_n,\Delta\nu^*_n)\Vert^2
    /\dtwo&\to 0.
\end{align}

Let $u=v$ in the first equation in equation \eqref{pf:d2_coefficients_1}, we achieve that for all $u\in[d]$,
\begin{align}   
    \label{eq:nondistinguishable_independent_4}
    [\exp(\tau_n)
[(\Delta \etan-\Delta \etasn)^{(u)}]^2
+
[\exp(\tau^*_n)
-
\exp(\tau_n)]
[( \Delta \etasn)^{(u)}]^2]
    /\dtwo&\to 0,
\end{align}
which implies that
\begin{align}
    \label{eq:nondistinguishable_independent_5}
[    \exp(\tau_n)
\|(\Delta \etan-\Delta \etasn)\|^2
+
(\exp(\tau^*_n)-\exp(\tau_n))
\| \Delta \etasn\|^2]
    /\dtwo&\to 0
    .
\end{align}

We also have each term inside equation \eqref{eq:nondistinguishable_independent_5} is non-negative, thus  
\begin{align}
\label{eq:nondistinguishable_independent_11}
    (\exp(\tau^*_n)-\exp(\tau_n))\|\Delta\etasn\|^2/\dtwo &\to 0,
    \nonumber\\
    \exp(\tau_n)\|\Delta\etan-\Delta\etasn\|^2/\dtwo &\to 0.
\end{align}

Applying the AM-GM inequality, we have for all $u,v\in[d]$,
\begin{align}
\label{eq:nondistinguishable_independent_12}
    \dfrac{(\exp(\tau^*_n)-\exp(\tau_n))(\Delta\etasn)^{(u)}(\Delta\etasn)^{(v)}}{\dtwo}\to 0,~ \dfrac{\exp(\tau_n)(\Delta\etan-\Delta\etasn)^{(u)}(\Delta\etan-\Delta\etasn)^{(v)}}{\dtwo} &\to 0,
    % ~u,v\in[d],  
    % \\ \label{eq:nondistinguishable_independent_13_hello}   \dfrac{(\lambdasn-\lambdan)(\Delta\asn)^{(u)}(\Delta\bsn)^{ }}{\dtwo}\to 0,~ \dfrac{\lambdan(\Delta\an-\Delta\asn)^{(u)}(\Delta\bn-\Delta\bsn)^{ }}{\dtwo} &\to 0.
\end{align}

Next, by considering the coefficients of $\dfrac{\partial h}{\partial \eta^{(u)}}(X,\etasn)\dfrac{\partial f}{\partial h}( \yhns)\exp((\betasn)^{\top}X)$,
and 
$\dfrac{\partial^2f}{\partial h^2}( \yhns)\exp((\betasn)^{\top}X)$, we have
\begin{align}
    \label{eq:nondistinguishable_independent_1}
    [\exp(\tau_n)(\Delta\etan)^{(u)}-\exp(\tau_n^*)(\Delta\etasn)^{(u)}]/\dtwo&\to 0,\quad u\in[d],\\
    \label{eq:nondistinguishable_independent_2}
    [\exp(\tau_n)(\Delta\nun)^{ }-\exp(\tau_n^*)(\Delta\nusn)^{ }]/\dtwo&\to 0.\quad 
\end{align}
Noting that for $u,v\in[d]$,
\begin{align*}
    &\exp(\tau^*_n)(\detasn)^{(u)}(\detan-\detasn)^{(v)} 
    \\&= (\exp(\tau_n)(\detan)^{(v)} - \exp(\tau^*_n)(\detasn)^{(v)})(\detasn)^{(u)}+(\exp(\tau^*_n)-\exp(\tau_n))(\detan)^{(v)}(\detasn)^{(u)},\\
    &\exp(\tau_n)(\detan)^{(u)}(\detan-\detasn)^{(v)} 
    \\&= \exp(\tau^*_n)(\detasn)^{(u)}(\detan-\detasn)^{(v)} - (\exp(\tau_n)(\detan)^{(u)} - \exp(\tau^*_n)(\detasn)^{(u)})(\detan-\detasn)^{(v)}.
\end{align*}  
Thus, from equation \eqref{eq:nondistinguishable_independent_12} and equation \eqref{eq:nondistinguishable_independent_1}, we achieve that for $u,v\in[d]$,
\begin{align*}
    \exp(\tau^*_n)(\detasn)^{(u)}(\detan-\detasn)^{(v)}/\dtwo &\to 0,
    \\ 
    \exp(\tau_n)(\detan)^{(u)}(\detan-\detasn)^{(v)}/\dtwo &\to 0. 
\end{align*}

% Noting that for all $u\in[d]$,
% \begin{align*}
%     \lambdasn(\dasn)^{(u)}(\dbn-\dbsn)^{ } &= (\lambdan\dbn - \lambdasn\dbsn)(\dasn)^{(u)}+(\lambdasn-\lambdan)(\dbn)(\dasn)^{(u)},\\
%     \lambdan(\dan)^{(u)}(\dbn-\dbsn)^{ } &= \lambdasn(\dasn)^{(u)}(\dbn-\dbsn)^{ } - \left(\lambdan(\dan)^{(u)} - \lambdasn(\dasn)^{(u)}\right)(\dbn-\dbsn)^{ }.
% \end{align*}
% Thus, from 
% % \eqref{eqn:nondistinguishable_independent_14}, 
% equation \eqref{eq:nondistinguishable_independent_1} and equation \eqref{eq:nondistinguishable_independent_1.2}, we have for all $u\in[d]$,
% \begin{align*}
%     \lambdasn(\dasn)^{(u)}(\dbn-\dbsn) /\dtwo &\to 0\\ \lambdan(\dan)^{(u)}(\dbn-\dbsn) /\dtwo &\to 0. 
% \end{align*}

By using the same arguments we will derive 
\begin{align}
    % \lambdan\|\Delta\an\|.\|\Delta\an-\Delta\asn\|/\dtwo\to0,\\
    % \lambdasn\|\Delta\asn\|.\|\Delta\an-\Delta\asn\|/\dtwo\to0,\\
    \label{eq:nondistinguishable_independent_11.3}
    \exp(\tau_n)\|\Delta\etan\|.\|\Delta\etan-\Delta\etasn\|/\dtwo\to0,\\
    \label{eq:nondistinguishable_independent_12.3}
    \exp(\tau^*_n)\|\Delta\etasn\|.\|\Delta\etan-\Delta\etasn\|/\dtwo\to0,
    % \\
    % \label{eq:nondistinguishable_independent_11.4}
    % \exp(\tau_n)\|\Delta\bn\|.\|\Delta\bn-\Delta\bsn\|/\dtwo\to0,\\
    % \label{eq:nondistinguishable_independent_12.4}
    % \exp(\tau^*_n)\|\Delta\bsn\|.\|\Delta\bn-\Delta\bsn\|/\dtwo\to0.
\end{align}
By using the same arguments to derive equation~\eqref{eq:nondistinguishable_independent_4}, equation~\eqref{eq:nondistinguishable_independent_11} and equation~\eqref{eq:nondistinguishable_independent_12}, we can point out that
\begin{align}
    [(\exp(\tau^*_n)-\exp(\tau_n))\|\Delta\nusn\|^2+\exp(\tau_n)\|\Delta\nun-\Delta\nusn\|^2]/\dtwo&\to 0,\nonumber\\
    \exp(\tau_n)\|\Delta\nun\|.\|\Delta\nun-\Delta\nusn\|/\dtwo&\to 0,\nonumber\\
    \nonumber
    \exp(\tau^*_n)\|\Delta\nusn\|.\|\Delta\nun-\Delta\nusn\|/\dtwo&\to 0,\\
    \exp(\tau_n)\|\Delta\etan\|.\|\Delta\nun-\Delta\nusn\|/\dtwo&\to 0,\nonumber\\
    \exp(\tau^*_n)\|\Delta\etasn\|.\|\Delta\nun-\Delta\nusn\|/\dtwo&\to 0. \label{eq:nondistinguishable_independent_13}
    % \\
    % \exp(\tau_n)\|\Delta\bn\|.\|\Delta\nun-\Delta\nusn\|/\dtwo&\to 0,\nonumber\\
    % \exp(\tau^*_n)\|\Delta\bsn\|.\|\Delta\nun-\Delta\nusn\|/\dtwo&\to 0\nonumber.
\end{align}
% Now following the same arguments as above and 
% In the same way,
% consider equations \eqref{pf:d2_coefficients_7} to \eqref{pf:d2_coefficients_12}, we will have that 
% In a similar manner, by considering equations \eqref{pf:d2_coefficients_7} through \eqref{pf:d2_coefficients_12}, and applying analogous reasoning as above, we obtain that
% \begin{align}
% \label{eq:nondistinguishable_independent_14}
%     \exp(\tau_n+\tau_n^*)\cdot
%     \Vert (\Delta\etan,\Delta\nun)-(\Delta\eta^*_n,\Delta\nu^*_n)\Vert^2
%     /\dtwo&\to 0.
% \end{align}

Collecting results in equation~\eqref{eq:nondistinguishable_independent_01},
\eqref{eq:d2proof_new}
and \eqref{eq:nondistinguishable_independent_02}, and equations~\eqref{eq:nondistinguishable_independent_11} to
\eqref{eq:nondistinguishable_independent_13},
we obtain that
\begin{align*}
    1=\dtwo/\dtwo\to 0,
\end{align*}
which is a contradiction. 

\begin{comment}
\textcolor{blue}{\begin{align*}
    D_2((\tau,G),(\tau^*,G_*)) :=
    &\|\beta-\beta^*\|
    \exp(\tau)
    \Vert (\Delta\eta,\Delta\nu) \Vert
    % \\
    % +&
    % \exp(\tau)\cdot
    % \Vert (\Delta\eta,\Delta\nu)-(\Delta\eta^*,\Delta\nu^*)\Vert^2
    \\
    +&
    \left|
    \exp(\tau^*)-\exp(\tau)
    \right|
    \cdot
    \Vert (\Delta\eta,\Delta\nu)\Vert
    \cdot
    \Vert (\Delta\eta^*,\Delta\nu^*)\Vert
    \\
    +&\Vert  (\Delta\eta,\Delta\nu)-(\Delta\eta^*,\Delta\nu^*) \Vert
    \big( 
    \exp(\tau)\Vert (\Delta\eta,\Delta\nu) \Vert
    +\exp(\tau^*)\Vert (\Delta\eta^*,\Delta\nu^*) \Vert
    \big)
    \\
    +&
    \exp(\tau+\tau^*)\cdot
    \Vert (\Delta\eta,\Delta\nu)-(\Delta\eta^*,\Delta\nu^*)\Vert^2
    % \\
    % +&
    % \exp(\tau)[1+\exp(\tau^*)]\cdot
    % \Vert (\Delta\eta,\Delta\nu)-(\Delta\eta^*,\Delta\nu^*)\Vert^2
\end{align*}}

\begin{align*}
    \dtwobarl &:=|\lambdas-\lambda| 
    \Vert (\da,\db,\dnu) \Vert
    \Vert (\das,\dbs,\dnus) \Vert
    \\&
    +\Vert  (\da,\db,\dnu)-(\das,\dbs,\dnus) \Vert
    \big( 
    \lambda\Vert (\da,\db,\dnu) \Vert
    +\lambdas\Vert (\das,\dbs,\dnus) \Vert
    \big)
\end{align*}
\end{comment}

Therefore, not all the coefficients in the representation of 
${W_n}/\dtwo$
tend to 0 as $n\to\infty$. Let us denote by $m_n$ the maximum of the absolute values of those coefficients. Based on the previous result, $1/m_n\not\to\infty$. Additionally, we define
% \begin{align*}
%     [(\exp(\tau^*_n)-\exp(\tau_n))\{-(\Delta \eta^*_{  n})^{(u)}\}+\exp(\tau_n)(\Delta \eta_{  n}-\Delta \eta^*_{ n})^{(u)}]/m_n&\to\alpha_{1 u},
%     \\
%     [(\exp(\tau^*_n)-\exp(\tau_n))\{-(\Delta \nu^*_{  n})^{ }\}+\exp(\tau_n)(\Delta \nu_{  n}-\Delta \nu^*_{ n})^{ }]/m_n&\to\alpha_{30},
%     \\
%     [(\exp(\tau^*_n)-\exp(\tau_n))(\Delta \eta^*_{  n})^{(u)}(\Delta \eta^*_{  n})^{(v)}+\exp(\tau_n)(\Delta \eta_{  n}-\Delta \eta^*_{ n})^{(u)}(\Delta \eta_{  n}-\Delta \eta^*_{  n})^{(v)}]/m_n&\to\beta_{ 1 uv},
%     \\
%     [(\exp(\tau^*_n)-\exp(\tau_n))(\Delta\etasn)^{(u)}(\Delta\nusn)^{ }+\exp(\tau_n)(\Delta\etan-\Delta\etasn)^{(u)}(\Delta\nun-\Delta\nusn)^{ }]/m_n&\to\gamma_{1u},\\
%     [(\exp(\tau^*_n)-\exp(\tau_n)) (\Delta\nusn)^{ 2}+ \exp(\tau_n)(\Delta\nun-\Delta\nusn)^{2 }]/m_n&\to\gamma_{30},
% \end{align*}
% $w\in[d], u,v\in[q]$
\begin{align}
\exp(\tau_n)
{[(\betan-\betasn)^{(w)}][(\Delta\etan)^{(u)}]}
    /m_n\to \alpha_{11,wu0},
    \nonumber
\\
\exp(\tau_n)
{[(\betan-\betasn)^{(w)}](\Delta\nun)}
    /m_n\to \alpha_{21,w00},
\nonumber\\
    [\exp(\tau_n)(\Delta\etan-\Delta\etasn)^{(u)}
+[\exp(\tau^*_n)-\exp(\tau_n)]
    (-\Delta\etasn)^{(u)}]
    /m_n \to \alpha_{10,0u0},
\nonumber\\
[\exp(\tau_n)
{(\Delta \etan-\Delta \etasn)^{(u)}(\Delta \etan-\Delta\etasn)^{(v)}}
+[\exp(\tau^*_n)-\exp(\tau_n)]
( -\Delta \etasn)^{(u)}( -\Delta\etasn)^{(v)}]
    /m_n
\nonumber\\
\to \beta_{10,0uv},
\nonumber\\
[\exp(\tau_n)
(\Delta\nun-\Delta\nusn)
+[\exp(\tau^*_n)-\exp(\tau_n)]
(-\Delta\nusn)]
    /m_n \to \alpha_{20,000},
\nonumber\\
[\exp(\tau_n)
(\Delta \etan-\Delta \etasn)^{(u)}(\Delta \etan-\Delta\etasn)^{(v)}
+[\exp(\tau^*_n)-\exp(\tau_n)]
( -\Delta \etasn)^{(u)}( -\Delta\etasn)^{(v)}]
    /m_n
    \nonumber\\
    \to \beta_{20,0uv},
\nonumber\\
[\exp(\tau_n)
(\Delta\etan-\Delta\etasn)^{(u)}(\Delta\nun-\Delta\nusn)
+[\exp(\tau^*_n)-\exp(\tau_n)]
(-\Delta\etasn)^{(u)}(-\Delta\nusn)]
    /m_n
    \nonumber\\
    \to \beta_{30,0u0},
\nonumber\\
[\exp(\tau_n)
(\Delta\nun-\Delta\nusn)^2
+[\exp(\tau^*_n)-\exp(\tau_n)]
(-\Delta\nusn)^2]
    /m_n
    \to \beta_{40,000},
\nonumber\\
\exp(\tau^*_n+\tau_n)
(\Delta\etan-\Delta\etasn)^{(u)}
    /m_n
    % \nonumber\\
    \to \rho_{1,u0},
\nonumber\\
\exp(\tau^*_n+\tau_n)
( \Delta\etan-\Delta \etasn)^{(u)}(\Delta\etan -\Delta\etasn)^{(v)}
    /m_n
    % \nonumber\\
    \to \pi_{1,uv},
\nonumber\\
\exp(\tau^*_n+\tau_n)
(\Delta\nun-\Delta\nusn)
    /m_n
    % \nonumber\\
    \to \rho_{2,00},
\nonumber\\
\exp(\tau^*_n+\tau_n)
(\Delta\etan -\Delta \etasn)^{(u)}( \Delta\etan-\Delta\etasn)^{(v)}
    /m_n
    % \nonumber\\
    \to \pi_{2,uv},
\nonumber\\
\exp(\tau^*_n+\tau_n)
(\Delta\etan-\Delta\etasn)^{(u)}
(\Delta\nun-\Delta\nusn)
    /m_n
    % \nonumber\\
    \to \pi_{3,u0},
\nonumber\\
\exp(\tau^*_n+\tau_n)
(\Delta\nun-\Delta\nusn)^2
    /m_n
    % \nonumber\\
    \to \pi_{4,00},
\label{d2_all_coefficients}
\end{align}
when $n\to\infty$ for all $w\in[d], u,v\in[q]$. Note that at least one among $\alpha_{\gamma\zeta, wuv},\beta_{\gamma\zeta, wuv}$ and $\rho_{\gamma, uv},\pi_{\gamma, uv}$ where $\gamma \in [4], \zeta\in \{ 0,1\}$ must be different from zero. By applying the Fatou's lemma, we get
\begin{align*}
    0=\lim_{n\to\infty}\frac{1}{m_n}\frac{2\bbE_X[d_V(\plbgn(\cdot|X),\plbgs(\cdot|X))]}{\dtwo}\geq \int\liminf_{n\to\infty}\frac{1}{m_n}\frac{|p_{ \Gn}(Y|X)-p_{ \Gsn}(Y|X)|}{\dtwo}d(X,Y).
\end{align*}
On the other hand,
\begin{align*}
    &
    \frac{1}{m_n}\frac{p_{ \Gn}(Y|X)-p_{ \Gsn}(Y|X)}{\dtwo}
    \\
    \to &  
    \sum_{\gamma=0}^4
    \Bigg[
    \sum_{\zeta=0}^1
    E_{\gamma\zeta}(X)
    X^{\zeta}
    +
    K_{\gamma}(X)\exp(\beta^{\top}X)
    \Bigg]
    % \cdot
     \frac{\partial^{\gamma} f}{\partial h^{\gamma}}(Y|h(X,\eta_0),\nu_0)
     \cdot
    \exp(\beta^{\top}X),
\end{align*}
where 
\begin{align*}
    &E_{11}(X)=\sum_{1\leq w\leq d, 1\leq u\leq q}\alpha_{11,wu0}\frac{\partial h}{\partial \eta^{(u)}}(X,\etasn)\\
    &E_{21}(X)=\frac{1}{2}\sum_{1\leq w\leq d}\alpha_{21,w00}\\
&E_{10}(X)=
    \sum_{u=1}^{q}\alpha_{10,0u0}\frac{\partial  h}{\partial \eta^{(u)}}(X,\eta_0)
    +\sum_{1\leq u,v\leq q}
    \frac{\beta_{10,0uv}}{1+\mathbf{1}_{u=v}}\frac{\partial^2 h}{\partial \eta^{(u)}\partial \eta^{(v)}}(X,\eta_0)
    ,
\\
&E_{20}(X)=
    \frac{1}{2}
    \alpha_{20,000}
    +\sum_{1\leq u,v\leq q}\frac{\beta_{20,0uv}}{1+\mathbf{1}_{u=v}}
    \frac{\partial h}{\partial \eta^{(u)}}(X,\eta_0)
    \frac{\partial h}{\partial \eta^{(v)}}(X,\eta_0),
\\
&E_{30}(X)=\frac{1}{2}
    \sum_{u=1}^{q}
    \beta_{30,0u0}
    \frac{\partial h}{\partial \eta^{(u)}}(X,\eta_0) ,
\\
&E_{40}(X)=\frac{1}{8}
    \beta_{40,000}.
\end{align*}
and
\begin{align*}
&K_{1}(X)=
    \sum_{u=1}^{q}\rho_{1,u0}\frac{\partial  h}{\partial \eta^{(u)}}(X,\eta_0)
    +\sum_{1\leq u,v\leq q}
    \frac{\pi_{1,uv}}{1+\mathbf{1}_{u=v}}\frac{\partial^2 h}{\partial \eta^{(u)}\partial \eta^{(v)}}(X,\eta_0)
    ,
\\
&K_{2}(X)=
    \frac{1}{2}
    \rho_{2,00}
    +\sum_{1\leq u,v\leq q}\frac{\pi_{2,uv}}{1+\mathbf{1}_{u=v}}
    \frac{\partial h}{\partial \eta^{(u)}}(X,\eta_0)
    \frac{\partial h}{\partial \eta^{(v)}}(X,\eta_0),
\\
&K_{3}(X)=\frac{1}{2}
    \sum_{u=1}^{q}
    \pi_{3,u0}
    \frac{\partial h}{\partial \eta^{(u)}}(X,\eta_0) ,
\\
&K_{4}(X)=\frac{1}{8}
    \pi_{4,00}.
\end{align*}

\begin{comment}
Then, we have
\begin{align*}
    % W_{n}&=
&
    \sum_{\gamma=0}^4
    \Bigg[
    \left(
    E_{\Ione_{n,\gamma,0}}(X)+
    E_{\Itwo_{n,\gamma,0}}(X)+
    E_{\Ithree_{n,\gamma,0}}(X)
    \right)
    +
    \sum_{\zeta=1}^2
    \left(
    E_{\Ione_{n,\gamma,\zeta}}(X)+
    E_{\Itwo_{n,\gamma,\zeta}}(X)
    \right)
    X^{\zeta}
    +
    E_{\Ifour_{n,\gamma,0}}(X)\exp((\betan)^{\top}X)
    \Bigg]
    \\&
    \hspace{3cm}
    \cdot
     \frac{\partial^{\gamma} f}{\partial h^{\gamma}}(Y|h(X,\eta_0),\nu_0)
     \cdot
    \exp((\betasn)^{\top}X)=0 .
\end{align*}
\begin{align*}
    % W_{n}&=
&
    \sum_{\gamma=0}^4
    \Bigg[
    % E_{\gamma,0}(X)+
    \sum_{\zeta=0}^2
    E_{\gamma,\zeta}(X)
    X^{\zeta}
    +
    K_{\gamma}(X)\exp(\beta^{\top}X)
    \Bigg]
    \cdot
     \frac{\partial^{\gamma} f}{\partial h^{\gamma}}(Y|h(X,\eta_0),\nu_0)
     \cdot
    \exp(\beta^{\top}X)=0 .
\end{align*}
% \begin{align*}
%     \sum_{\tau=1}^4E_{\tau}(X)\frac{\partial^{\tau}f}{\partial  \sigma^{\tau}}(\ysigmazero )=0.
% \end{align*}
\end{comment}

It is worth noting that for almost surely $(X,Y)$, the set $\mathcal{L}\cup\mathcal{K}$
% \begin{align*}
%     \left\{\frac{\partial^{\tau}f}{\partial  h^{\tau}}(\yhzero)\right\}
%     \cup
%     \left\{X\frac{\partial^{\tau}f}{\partial  h^{\tau}}(\yhzero)\right\}
%     \cup
%     \left\{\exp(\beta^{\top}X)\frac{\partial^{\tau}f}{\partial  h^{\tau}}(\yhzero)\right\}
% \end{align*}
% where $\tau\in[4]$
is linearly independent under non-distinguishable setting
% (for details in Lemma~\ref{lemma:dql_gaussian_independence} and Lemma~\ref{applemma:poly-exp-indpt})
, which leads to the fact that $E_{\tau\zeta}(X)=K_{\tau}(X)=0$ for almost surely $X$ for any $\tau\in[4], \zeta\in\{0,1 \}$. 

% It is worth noting that, for almost every $(X, Y)$, the set $\mathcal{L}\cup\mathcal{K}$:
% \begin{align*}
%     \left\{\frac{\partial^{\tau}f}{\partial  h^{\tau}}(\yhzero)\right\}
%     \cup
%     \left\{X\frac{\partial^{\tau}f}{\partial  h^{\tau}}(\yhzero)\right\}
%     \cup
%     \left\{\exp(\beta^{\top}X)\frac{\partial^{\tau}f}{\partial  h^{\tau}}(\yhzero)\right\}
% \end{align*}
% is linearly independent (see Lemma~\ref{lemma:dql_gaussian_independence} and Lemma~\ref{applemma:poly-exp-indpt}). This implies that $E_{\tau \zeta}(X) = K_{\tau}(X) = 0$ for almost every $X$ and for all $\tau \in [4]$, $\zeta \in \{0, 1\}$.

% Since $h(X,\eta)$ and $\nu$ are algebraically independent, 
% % and $E_3(X)=E_4(X)=0$ for almost surely $X$, 
% % Since the experts followed the 
% we get 
% % $\alpha_{\tau\zeta,wuv}=0$ 
% all the coefficients in equation \eqref{d2_all_coefficients} is zero
% for all $w,u,v$
% . 

% Similar to the proof \ref{app_proof: d1_loss} and recall that the partial differential of Gaussian density are linear independent, 
% we get all the coefficients in equation \eqref{d2_all_coefficients} is zero for all $w,u,v$. 

% Similar to the proof of Theorem~\ref{thm:not_equal}, and recalling that experts are strongly identifiable
% % the partial derivatives of the Gaussian density w.r.t $h$ are linearly independent when taken at different orders, 
% we conclude that all the coefficients in Equation~\eqref{d2_all_coefficients} must be zero for all $w, u, v$.

Similar to the proof of Theorem~\ref{thm:not_equal}, and recalling that the experts are strongly identifiable, we conclude that all the coefficients in Equation~\eqref{d2_all_coefficients} must be zero for all $w, u, v$.

% \begin{lemma}[Linear Independence of Gaussian Derivatives]
% \label{lemma:gaussian_derivatives_independent}
% Let $f(y \mid \mu, \sigma^2)$ denote the univariate Gaussian density with mean $\mu$ and variance $\sigma^2 > 0$:
% \[
% f(y \mid \mu, \sigma^2) = \frac{1}{\sqrt{2\pi\sigma^2}} \exp\left( -\frac{(y - \mu)^2}{2\sigma^2} \right).
% \]
% Then, the set of functions
% \[
% \left\{ \frac{\partial^\tau}{\partial \mu^\tau} f(y \mid \mu, \sigma^2) : \tau = 0, 1, 2, \ldots, T \right\}
% \]
% is linearly independent over $\mathbb{R}$ for any fixed $\sigma^2 > 0$ and any finite $T \in \mathbb{N}$.

% Moreover, in the multivariate case where $f(y \mid \mu, \Sigma)$ is the Gaussian density on $\mathbb{R}^d$ with mean $\mu \in \mathbb{R}^d$ and positive-definite covariance matrix $\Sigma \in \mathbb{R}^{d \times d}$, the collection of partial derivatives
% \[
% \left\{ \frac{\partial^{|\alpha|}}{\partial \mu^\alpha} f(y \mid \mu, \Sigma) : \alpha \in \mathbb{N}^d,   |\alpha| \leq T \right\}
% \]
% is linearly independent over $\mathbb{R}$ for any fixed $\Sigma$ and finite $T$.
% \end{lemma}

% It follows from the result $E_2(X)=0$ and 
% $\gamma_{\tau u}=0$ that 
% $\alpha_{30}=\beta_{\tau uv}=0$ for all $u,v$. Next, $E_1(X)=0$ implies that 
% $\alpha_{\tau u}=0$ for all $u$. 
This contradicts the fact that not all 
coefficients
% $\alpha_{\tau u}$, $\beta_{\tau uv}$ and $\gamma_{uv}$ 
vanish. Thus, we obtain the conclusion for this case.

{
\subsubsection*{Case 2: }
In this case, we consider that $(\etan,\nun)$ and $(\etasn,\nusn)$ share the same limit, but different from $(\eta_0,\nu_0)$. 

From the formulation of the metric $D_1$ in the proof~\ref{app_proof: d1_loss}, it is clear that $D_2 \lesssim D_1$. 
Therefore, we get 
$
% [p_{\lambdan,\Gn}(X,Y)-p_{\lambdasn,\Gsn}(X,Y)]
W_n(X,Y)
/\done\to 0$ as $n\to\infty$. Noting that $(\etan,\nun)$ and $(\etasn,\nusn)$ share the limit $(\eta^*,\nu^*)\neq (\eta_0,\nu_0)$, we have $f_0 = f(Y|h(X,\eta_0), \nu_0)$ and $f(Y|h(X,\eta^*), \nu^*)$ satisfying $f_0$ and $f$ independent up to second order as in Lemma \ref{appendix_lemma:distinguish_linear_independent}. Thus, we can process in a similar way as in Theorem \ref{thm:not_equal} to draw a contradiction. 

\subsubsection*{Case 3: } Lastly, we consider that one of $G_n$ or $G_n^*$ converges to $G_0$, while the other converges to $G' \neq G_0$. Without loss of generality, suppose that $G_n \to G'$ and $G_n^* \to G_0$. By passing through the limit for 
\begin{equation*}
    \bbE_X[h_V(p_{ \Gn}(\cdot|X),p_{ \Gsn}(\cdot|X))]/D_2(G_n,G_n^*) \to 0,
\end{equation*} 
noting that 
\begin{equation*}
    D_2(G_n,G_n^*) \to D_2(G,G_*) \neq 0, \bbE_X[h_V(p_{\Gn}(\cdot|X),p_{\Gsn}(\cdot|X))] \to \bbE_X[h_V(p_{G}(\cdot|X),p_{G_{*}}(\cdot|X))],
\end{equation*}
we have 
\begin{equation*}
    \bbE_X[h_V(p_{G}(\cdot|X),p_{G_{*}}(\cdot|X))]= 0, \text{ or  } p_{G} = p_{G_{*}}, \text{ a.s. }
\end{equation*}

This equation implies that 
\begin{align*}
    f(Y|h(X,\eta_0),\nu_0) &= \frac{1}{1+\exp(\beta^\top X + \tau^*)}f(Y|h(X,\eta_0),\nu_0) + \frac{\exp(\beta^\top X + \tau^*)}{1+\exp(\beta^\top X + \tau^*)}f(Y|h(X,\eta),\nu)
\end{align*}
which further implies that
\begin{align*}
    \frac{\exp(\beta^\top X + \tau^*)}{1+\exp(\beta^\top X + \tau^*)}f(Y|h(X,\eta_0),\nu_0) &= \frac{\exp(\beta^\top X + \tau^*)}{1+\exp(\beta^\top X + \tau^*)}f(Y|h(X,\eta),\nu)
\end{align*}
and hence
\begin{align*}
    f(Y|h(X,\eta_0),\nu_0) &= f(Y|h(X,\eta),\nu) \quad \text{(as $\exp(\beta^\top X + \tau^*) \neq 0$)}.
\end{align*}
% This equation implies that 
% \begin{align*}
%     &f(Y|h(X,\eta_0),\nu_0) = \dfrac{1}{1+\exp(\beta^\top X + \tau^*)}f(Y|h(X,\eta_0),\nu_0) + \dfrac{\exp(\beta^\top X + \tau^*)}{1+\exp(\beta^\top X + \tau^*)}f(Y|h(X,\eta),\nu)\\
%     \Rightarrow&\dfrac{\exp(\beta^\top X + \tau^*)}{1+\exp(\beta^\top X + \tau^*)}f(Y|h(X,\eta_0),\nu_0) = \dfrac{\exp(\beta^\top X + \tau^*)}{1+\exp(\beta^\top X + \tau^*)}f(Y|h(X,\eta),\nu)\\
%     \Rightarrow& f(Y|h(X,\eta_0),\nu_0) = f(Y|h(X,\eta),\nu) \text{ (as $\exp(\beta^\top X + \tau^*) \neq 0$)}
% \end{align*}
This equation means that $G' = G_0$, which is a contradiction. 
}
\end{proof}

\subsection{Proof of Theorem \ref{thm:d2_minimax}}
% \subsection{$D_2$ loss minimax}
\label{apppf:d2_minimax}

% \begin{align*}
%     \overline{D_2}\left( G, G_*) \right)
%  &:=
%     \|\beta-\beta^*\|
%     \exp(\tau)
%     \Vert (\Delta\eta,\Delta\nu) \Vert
%     \\&
%     +
%     \left|
%     \exp(\tau^*)-\exp(\tau)
%     \right|
%     \cdot
%     \Vert (\Delta\eta,\Delta\nu)\Vert
%     \cdot
%     \Vert (\Delta\eta^*,\Delta\nu^*)\Vert
%     \\&
%     +\Vert  (\Delta\eta,\Delta\nu)-(\Delta\eta^*,\Delta\nu^*) \Vert
%     \big( 
%     \exp(\tau)\Vert (\Delta\eta,\Delta\nu) \Vert
%     +\exp(\tau^*)\Vert (\Delta\eta^*,\Delta\nu^*) \Vert
%     \big)
%     \\&
%     +
%     \exp(\tau+\tau^*)\cdot
%     \Vert (\Delta\eta,\Delta\nu)-(\Delta\eta^*,\Delta\nu^*)\Vert^2
% \end{align*}

In what follows, we present the proof of Theorem~\ref{thm:d2_minimax} for the non-distinguishable setting.
\begin{proof}[Proof of Theorem \ref{thm:d2_minimax}]
The proof follows similar steps to the arguments in the previous two sections.
Concretely, define for 
$S_1 = (\tau_1, \beta_1,\eta_1,\nu_1)$, $S_2= (\tau_2,\beta_2,\eta_2,\nu_2)$
: 
\begin{align*}
\begin{cases}
    % d_{\prime}(S_1,S_2)
    % =\exp(\tau_1)\Vert \Delta \eta_1, \Delta\nu_1 \Vert \Vert \beta_1-\beta_2\Vert,
    % \\
    d_{\prime}(S_1,S_2)
    =  \Vert \Delta \eta_1, \Delta\nu_1 \Vert ^2
    | \exp(\tau_1) -\exp(\tau_2) |
    ,
    \\
    d_{\prime\prime}(S_1,S_2)
    =\exp(\tau_1)\Vert \Delta \eta_1, \Delta\nu_1 \Vert
    \Vert (\beta_1,\eta_1,\nu_1) - (\beta_2,\eta_2,\nu_2) \Vert
    .
\end{cases}
\end{align*}
It is straightforward that $d_{\prime}$ and $ d_{\prime\prime}$  satisfy the weak triangle inequality. Following the same schema as in Lemma \ref{prop:lower-distinguish}, we can demonstrate two subsequent results for any $r > 1$: 
\begin{itemize}
    % \item [(i)] Two sequences can be found  
    % \begin{align*}
    % \begin{cases}
    %     S_{1,n}=(\tau_n,\beta_{1,n},\eta_{n},\nu_{n})\in \Xi(l_n),\\
    %     S_{2,n}=(\tau_n,\beta_{2,n},\eta_{n},\nu_{n})\in \Xi(l_n),
    % \end{cases}
    % \end{align*}
    % such that $d_{\prime}(S_{1,n},S_{2,n}) \to 0$ and $h(p_{S_{1,n}}, p_{S_{2,n}})/d_{\prime}^r(S_{1,n},S_{2,n})\to 0$ as $n \to \infty$. 

    \item [(i)] Two sequences can be found  
    \begin{align*}
    \begin{cases}
        S_{1,n}=(\tau_{1,n},\beta_{n},\eta_{n},\nu_{n})\in \Xi(l_n),\\
        S_{2,n}=(\tau_{1,n},\beta_{n},\eta_{n},\nu_{n})\in \Xi(l_n),
    \end{cases}
    \end{align*}
    such that $d_{\prime}(S_{1,n},S_{2,n}) \to 0$ and $\bbE_X[h_H(p_{S_{1,n}}(\cdot|X), p_{S_{2,n}}(\cdot|X))]/d_{\prime}^r(S_{1,n},S_{2,n})\to 0$ as $n \to \infty$.

    \item [(ii)] Two sequences can be found  
    \begin{align*}
    \begin{cases}
        S_{1,n}=(\tau_n,\beta_{1,n},\eta_{1,n},\nu_{1,n})\in \Xi(l_n),\\
        S_{2,n}=(\tau_n,\beta_{2,n},\eta_{2,n},\nu_{2,n})\in \Xi(l_n),
    \end{cases}
    \end{align*}
    such that $d_{\prime\prime }(S_{1,n},S_{2,n}) \to 0$ and $\bbE_X[h_H(p_{S_{1,n}}(\cdot|X), p_{S_{2,n}}(\cdot|X))]/d_{\prime\prime }^r(S_{1,n},S_{2,n})\to 0$ as $n \to \infty$.
\end{itemize}
We can omit the justification for the above results as it can follow a similar approach as in Lemma \ref{prop:lower-distinguish}. This leads to the conclusion of the theorem.   
\end{proof}

%%%%%%%%%%%%%%%%%%%%%%%%%%%%%%%%%%%%%%%%%%%%%%%%%%%%%%%%%%%%
% \section{PROOFS FOR AUXILIARY RESULTS}
\section{Proof of Auxiliary Results}
\label{appendix:ProofsforAuxiliaryResults}

\subsection{Proof of Proposition~\ref{lemma:distinguish-linear sigma not Gaussian}}
\label{appendix:lemma:distinguish-linear sigma not Gaussian}

\begin{proof}
Fix an arbitrary \(x\in\mathcal X\) and abbreviate  
\[
g_{1}(y):=f \bigl(y |  h(x,\eta_{1}),\nu_{1}\bigr), \qquad
g_{2}(y):=f \bigl(y |  h(x,\eta_{2}),\nu_{2}\bigr), \qquad
g_{0}(y):=f_{0} \bigl(y |  h_{0}(x,\eta_{0}),\nu_{0}\bigr).
\]
Because \(f\) is Gaussian in its argument, there exist \(\mu_{1},\mu_{2}\in\mathbb R\) and
\(\sigma_{1}^{2},\sigma_{2}^{2}>0\) such that  
\(g_{j}(y)=\dfrac1{\sqrt{2\pi\sigma_{j}^{2}}}\exp \bigl(-(y-\mu_{j})^{2}/(2\sigma_{j}^{2})\bigr)\) for
\(j=1,2\).

Set  
\[
H_1(y):=\frac{\partial g_{2}}{\partial h}(y)=
  \frac{y-\mu_{2}}{\sigma_{2}^{2}} g_{2}(y), 
\qquad
H_2(y):=\frac{\partial^{2}g_{2}}{\partial h^{2}}(y)=
  \frac{(y-\mu_{2})^{2}-\sigma_{2}^{2}}{\sigma_{2}^{4}} g_{2}(y).
\]

With these notations the assumed identity becomes
\begin{equation}
\label{eq:L-in-y}
b_{0}(x)g_{0}(y)+b_{1}(x)g_{1}(y)
+c_{0}(x)g_{2}(y)+c_{1}(x)H_{1}(y)+\tfrac12 c_{2}(x)H_{2}(y)=0
\quad\text{for a.e.~}y\in\mathbb R .
\end{equation}

\paragraph{\bf 1.  \(b_{0}(x)=0\).}
Because \(g_{0}\) is \emph{not} Gaussian by assumption, while
\(g_{1},g_{2},H_{1},H_{2}\) all belong to the finite–dimensional
linear span
\(
\mathcal G:=\operatorname{span}\{y\mapsto g_1(y), y\mapsto (y-\mu_{2})^{k}g_{2}(y):k=0,1,2\}
\),
we have \(g_{0}\notin\mathcal G\).  Hence the only way
\eqref{eq:L-in-y} can hold on a set of positive measure is with
\(b_{0}(x)=0\).

\paragraph{\bf 2.  Linear independence inside \(\mathcal G\).}
Divide \eqref{eq:L-in-y} (now with \(b_{0}(x)=0\)) by \(g_{2}(y)\); we
obtain the polynomial identity
\[
b_{1}(x) \frac{g_{1}(y)}{g_{2}(y)}
+c_{0}(x)+c_{1}(x) \frac{y-\mu_{2}}{\sigma_{2}^{2}}
+\tfrac12 c_{2}(x) \frac{(y-\mu_{2})^{2}-\sigma_{2}^{2}}{\sigma_{2}^{4}}
=0 \quad\text{for a.e.~}y.
\]

The ratio \(g_{1}/g_{2}\) is the analytic (non‑polynomial) function
\[
\frac{g_{1}(y)}{g_{2}(y)}
  =K\exp \Bigl(\tfrac12\bigl[(y-\mu_{2})^{2}/\sigma_{2}^{2}
                             -(y-\mu_{1})^{2}/\sigma_{1}^{2}\bigr]\Bigr),
\]
with \(K\neq0\).  Since \(\mu_{1}\neq\mu_{2}\) or
\(\sigma_{1}^{2}\neq\sigma_{2}^{2}\), this exponential term cannot be
expressed as a quadratic polynomial in \(y\).  Consequently the set of
functions
\(
\bigl\{g_{1}/g_{2}, 1, y-\mu_{2}, (y-\mu_{2})^{2}\bigr\}
\)
is linearly independent on any interval.  Hence every coefficient in
the polynomial identity must vanish:
\[
b_{1}(x)=c_{0}(x)=c_{1}(x)=c_{2}(x)=0.
\]

\paragraph{\bf 3.  Conclusion.}
We have shown that
\(
b_{0}(x)=b_{1}(x)=c_{0}(x)=c_{1}(x)=c_{2}(x)=0
\)
for the fixed \(x\).  Because the same argument works for almost every
\(x\in\mathcal X\), all coefficients vanish almost surely.  Thus the
unified distinguishability condition of
Definition~\ref{def:distinguishability} is satisfied, completing the
proof.
\end{proof}

\subsection{Proof of Proposition~\ref{prop:identifiability}}
\label{appendix:identifiability}

\begin{proof}
Write the two (single–expert) conditional densities
% \[
% p_{G}(y |  x)
%   =\bigl[1-\lambda(x)\bigr] 
%      f_0 \bigl(y |  h_0(x,\eta_0),\nu_0\bigr)
%    +\lambda(x) 
%      f \bigl(y |  h(x,\eta),\nu\bigr),
% \qquad
% \lambda(x):=\frac{\exp \bigl(\beta^\top x+\tau\bigr)}
%                  {1+\exp \bigl(\beta^\top x+\tau\bigr)},
% \]
\begin{align*}
    p_{G}(y |  x)
  &=\bigl[1-\lambda(x)\bigr] 
     f_0 \bigl(y |  h_0(x,\eta_0),\nu_0\bigr)
   +\lambda(x) 
     f \bigl(y |  h(x,\eta),\nu\bigr),
     \\
    p_{G'}(y |  x)
  &=\bigl[1-\lambda'(x)\bigr] 
     f_0 \bigl(y |  h_0(x,\eta_0),\nu_0\bigr)
   +\lambda'(x) 
     f \bigl(y |  h(x,\eta'),\nu'\bigr), 
\end{align*}
where 
$\lambda(x):=\frac{\exp \bigl(\beta^\top x+\tau\bigr)}
                 {1+\exp \bigl(\beta^\top x+\tau\bigr)}$
and
$\lambda'(x):=\frac{\exp \bigl({\beta'}^\top x+\tau'\bigr)}
                  {1+\exp \bigl({\beta'}^\top x+\tau'\bigr)}$.

% \[
% p_{G}(y |  x)
%   =\bigl[1-\lambda(x)\bigr] 
%      f_0 \bigl(y |  h_0(x,\eta_0),\nu_0\bigr)
%    +\lambda(x) 
%      f \bigl(y |  h(x,\eta),\nu\bigr),
% \qquad
% \lambda(x):=\frac{\exp \bigl(\beta^\top x+\tau\bigr)}
%                  {1+\exp \bigl(\beta^\top x+\tau\bigr)},
% \]

% \[
% p_{G'}(y |  x)
%   =\bigl[1-\lambda'(x)\bigr] 
%      f_0 \bigl(y |  h_0(x,\eta_0),\nu_0\bigr)
%    +\lambda'(x) 
%      f \bigl(y |  h(x,\eta'),\nu'\bigr),
% \qquad
% \lambda'(x):=\frac{\exp \bigl({\beta'}^\top x+\tau'\bigr)}
%                   {1+\exp \bigl({\beta'}^\top x+\tau'\bigr)} .
% \]

Assume the identifiability equality  
\(p_{G}(y |  x)=p_{G'}(y |  x)\) holds for almost every
\((x,y)\in\mathcal X\times\mathcal Y\).
Subtracting the two representations gives
\begin{equation}\label{eq:difference}
\bigl[\lambda(x)-\lambda'(x)\bigr] 
      f_0 \bigl(y |  h_0(x,\eta_0),\nu_0\bigr)
+     \lambda'(x) 
      f \bigl(y |  h(x,\eta'),\nu'\bigr)
-     \lambda(x) 
      f \bigl(y |  h(x,\eta),\nu\bigr)=0 .
\end{equation}

\paragraph*{Step 1.  If \(\lambda(x)\neq\lambda'(x)\).}
Suppose on a set of positive \(x\)-measure,
\(\lambda(x)\neq\lambda'(x)\).  Divide \eqref{eq:difference} by
\(\lambda(x)-\lambda'(x)\); then for those \(x\)
\begin{align*}
f_0 \bigl(y |  h_0,\nu_0\bigr)
   +b(x) f \bigl(y |  h(x,\eta'),\nu'\bigr)
   +c(x) f \bigl(y |  h(x,\eta),\nu\bigr)=0 ,
\end{align*}
where
\begin{align*}
b(x):=\frac{\lambda'(x)}{\lambda'(x)-\lambda(x)}, 
c(x):=\frac{-\lambda(x)}{\lambda'(x)-\lambda(x)} .
\end{align*}

% This is a linear combination of three densities of the
% form required by the \emph{zeroth‑order} part of the new
% Definition~\ref{def:distinguishability}.  
Since \(f\) is distinguishable from \(f_0\), the only possibility is
\(b(x)=c(x)=0\), hence \(\lambda(x)=\lambda'(x)\) a.e.—contradiction.
Therefore
\[
\lambda(x)=\lambda'(x)\quad\text{for a.e.\ }x .
\]

Because the soft‑max map \((\beta,\tau)\mapsto\lambda(\cdot)\) is
injective, we conclude
\[
\beta=\beta',\qquad\tau=\tau'.
\]

% \paragraph*{Step 2.  Equality of expert parameters.}

% With \(\lambda(x)=\lambda'(x)\), equation \eqref{eq:difference} reduces to

% \[
% f \bigl(y |  h(x,\eta),\nu\bigr)
%       =f \bigl(y |  h(x,\eta'),\nu'\bigr)
% \qquad\text{for a.e.\ }(x,y).
% \]

% Fix \(x\) and differentiate both sides with respect to
% \((\eta',\nu')\).  Using the notation
% \(h':=h(x,\eta'),\ f':=f(y |  h',\nu')\) and the chain rule,

% \[
% \Bigl[\tfrac{\partial h}{\partial\eta}(x,\eta') \tfrac{\partial f}{\partial h}\Bigr]_{(\eta',\nu')}
%    =0,\qquad
% \Bigl[\tfrac{\partial f}{\partial\nu}\Bigr]_{(\eta',\nu')}=0\quad
%    \text{in }L^{2}(\mathcal Y).
% \]

% Applying the \emph{first‑order} part of
% Definition~\ref{def:distinguishability} with  
% \(b_{0}(x)=0, c_{(1,0)}(x)=\partial h/\partial\eta, 
%   c_{(0,1)}(x)=1\)
% forces  
% \(\partial h/\partial\eta(x,\eta')=0\) and \(1=0\) — the latter
% impossible.  Hence the only consistent solution is
% \[
% (\eta,\nu)=(\eta',\nu').
% \]

% \dqledit
{
\paragraph*{Step 2.  Equality of expert parameters.}

With \(\lambda(x)=\lambda'(x)\), equation \eqref{eq:difference} reduces to
\[
f \bigl(y |  h(x,\eta),\nu\bigr)
      =f \bigl(y |  h(x,\eta'),\nu'\bigr)
\qquad\text{for a.e.\ }(x,y).
\]

% Applying 
% the \emph{first‑order} part of
Definition~\ref{def:distinguishability} 
% with  
% \(b_{0}(x)=0, b_{1}(x)=1, 
%   c_{(0,0)}(x)=-1\)
forces the situation $(\eta, \nu) \neq (\eta',\nu')$ impossible.  Hence the only consistent solution is
\[
(\eta,\nu)=(\eta',\nu').
\]
}
\paragraph*{Step 3.  Conclusion.}

We have shown
\(\beta=\beta'\), \(\tau=\tau'\), \(\eta=\eta'\), and
\(\nu=\nu'\); hence \(G=G'\).
\end{proof}

\subsection{Proof of Proposition \ref{theorem:ConvergenceRateofDensityEstimation}}
\label{appendix:ConvergenceRateofDensityEstimation}

% \subsection{General theory for the Proof of Theorem \ref{theorem:ConvergenceRateofDensityEstimation}}
% \label{appendix:ConvergenceRateofDensityEstimation}

% At first we restate Proposition \ref{theorem:ConvergenceRateofDensityEstimation}:

% \begin{proposition}
% \label{appendixtheorem:ConvergenceRateofDensityEstimation}
% Assume that the function $f_0$ is bounded with tail 
% $\mathbb{E}_X
% \left(
% -\log f_0(Y|h (X,\eta_0),\nu_0)
% \right)
% \gtrsim
% Y^q
% $
% for almost surely $Y\in\mathcal{Y}$
% for some $q>0$.
% and $f$ is the density function of an univariate Gaussian distribution.
%     % Assume the following assumption holds:
%     % A2. Given a universal constant $J > 0$, there exists $N > 0$, possibly depending on  $\Xi$, such that for all $n \geq N$ and all $\epsilon > (\log(n)/n)^{1/2}$, we have $\mathcal{J}_B(\epsilon, \overline{P}^{1/2}(\Xi, \epsilon)) \leq J \sqrt{n} \epsilon^2$.
% Then, there exists a constant $C > 0$ depending only on $\Xi$ such that for all $n \geq 1$,
% \begin{align*}
%     \sup_{\Gs\in\Xi}
%     \mathbb{E}_{p_{ \Gs}}
%     h(p_{ \widehat{G}_n},p_{ \Gs})
%     \leq
%     C\sqrt{\log n/n}.
% \end{align*}
% \end{proposition}

% \subsection*{Notation and Preliminaries}

We begin by introducing several standard notations used throughout this proof. Let $(\mathcal{P}, d)$ be a metric space, where $d$ is a metric on $\mathcal{P}$. An $\epsilon$-net of $(\mathcal{P}, d)$ is a collection of balls of radius $\epsilon$ whose union covers $\mathcal{P}$. The \emph{covering number} $N(\epsilon, \mathcal{P}, d)$ denotes the minimal cardinality of such a covering, and the \emph{entropy number} is defined as $H(\epsilon, \mathcal{P}, d) := \log N(\epsilon, \mathcal{P}, d)$.

The \emph{bracketing number} $N_B(\epsilon, \mathcal{P}, d)$ is the minimal number of pairs $\{(\underline{f}_i, \overline{f}_i)\}_{i=1}^n$ such that $\underline{f}_i < \overline{f}_i$, $d(\underline{f}_i, \overline{f}_i) < \epsilon$, and $\mathcal{P}$ is covered by the union of the brackets. The corresponding \emph{bracketing entropy} is denoted by $H_B(\epsilon, \mathcal{P}, d) := \log N_B(\epsilon, \mathcal{P}, d)$.

When $\mathcal{P}$ is a family of densities, we take $d$ to be the $L^2(m)$ distance, where $m$ denotes the Lebesgue measure.

In particular, let $\mathcal{P}(\Xi) := \{ p_{\lambda} : \lambda \in \Xi \}$, and define the symmetrized density
\(
\bar{p}_\lambda := \frac{1}{2}(p^* + p_\lambda),
\)
where $p^*$ denotes the true density. We then define the following sets:
$\overline{\mathcal{P}}(\Xi) := \{ \bar{p}_\lambda : \lambda \in \Xi \}$ 
and
$\overline{\mathcal{P}}^{1/2}(\Xi) := \{ \bar{p}_\lambda^{1/2} : \bar{p}_\lambda \in \overline{\mathcal{P}}(\Xi) \}.$
% \begin{align*}
%     \overline{\mathcal{P}}(\Xi) &:= \{ \bar{p}_\lambda : \lambda \in \Xi \}, \\
%     \overline{\mathcal{P}}^{1/2}(\Xi) &:= \{ \bar{p}_\lambda^{1/2} : \bar{p}_\lambda \in \overline{\mathcal{P}}(\Xi) \}.
% \end{align*}
To study convergence rates, we consider the localized version of the symmetrized class:
\(
\overline{\mathcal{P}}^{1/2}(\Xi, \epsilon) := \{ \bar{p}_\lambda^{1/2} \in \overline{\mathcal{P}}^{1/2}(\Xi) : d_H(\bar{p}_\lambda, p^*) \leq \epsilon \},
\)
where $d_H(\cdot,\cdot)$ denotes the Hellinger distance.
Then we assess the complexity of this class via the \emph{bracketing entropy integral} defined in \cite{Vandegeer-2000}:
\(
\mathcal{J}_B(\epsilon, \overline{\mathcal{P}}^{1/2}(\Xi, \epsilon), m) := \int_{\epsilon^2 / 2^{13}}^{\epsilon} \sqrt{H_B(u, \overline{\mathcal{P}}^{1/2}(\Xi, \epsilon), m)}   du \vee \epsilon,
\)
where $a \vee b := \max\{a, b\}$. For brevity, we may omit the dependence on $m$ when it is clear from context.

For the proof at first we consider a general lemma that provides the desired convergence rate, provided that a bracketing entropy condition is satisfied.

\begin{lemma}
\label{applemma:convergence-rate-1}
    Assume the following assumption hold:
    Given a universal constant $J > 0$, there exists $N > 0$, possibly depending on  $\Xi$, such that for all $n \geq N$ and all $\epsilon > (\log(n)/n)^{1/2}$, we have
\begin{align}
\label{assumption:A2}
   \mathcal{J}_B(\epsilon, \overline{P}^{1/2}(\Xi, \epsilon)) \leq J \sqrt{n} \epsilon^2. 
\end{align}    
    Then, there exists a constant $C > 0$ depending only on $\Xi$ 
    such that for all $n\geq1$,
\begin{align*}
    \sup_{\Gs\in\Xi}
    \mathbb{E}_{p_{ \Gs,n}}
    \bbE_X[d_H(p_{ \widehat{G}_n}(\cdot|X),p_{ \Gs}(\cdot|X))]
    \leq
    C\sqrt{\log n/n}.
\end{align*}
\end{lemma}

This lemma indicates that it suffices to verify the entropy condition in Equation~\eqref{assumption:A2} in order to obtain the convergence rate. However, this condition is often technically difficult to establish directly. As a workaround, we may instead prove the following sufficient condition:

\begin{lemma}
\label{applemma:convergence-rate-add-1}
    If the distribution satisfies
    \begin{align}
    \label{assumption:A3}
        H_B(\epsilon,\mathcal{P}(\Xi),d_H)\lesssim\log (1/\epsilon),
    \end{align}
    it will meet the assumption in  Equation~(\ref{assumption:A2}). 
\end{lemma}

Although we have simplified the condition in Equation~\eqref{assumption:A2} to Equation~\eqref{assumption:A3}, verifying Equation~\eqref{assumption:A3} is still nontrivial. Fortunately, for the contaminated model defined in Equation~\eqref{eq:contaminated_pretrain_model_general},
\begin{align*}
    p_{G}(Y|X) & := \frac{1}{1+\exp(\beta ^{\top}X+\tau )}\cdot f_{0}(Y|h_0(X,\eta_0), \nu_{0})  
    + \frac{\exp(\beta^{\top}X+\tau)}{1+\exp(\beta^{\top}X+\tau)}\cdot f(Y|h(X,\eta),\nu),
\end{align*}
we assume that $f_0$ is bounded with light tails and that $f$ is a univariate Gaussian density. Under these assumptions, we can verify Equation~\eqref{assumption:A3} via the following lemma:

\begin{lemma}
\label{applemma:convergence-rate-2}
    Let $\Gamma$ be a compact subsets of $\mathbb{R}^d\times\mathbb{R}$ and $\Theta$ be a bounded subsets of $\mathbb{R}^q\times\mathbb{R}^{+}$, $f$ is a univariate Gaussian density and $f_0$ is bounded with tail 
$\mathbb{E}_X\left(-\log f_0(Y|h (X,\eta_0 ),\nu_0)\right)\gtrsim Y^q$
for almost surely $Y\in\mathcal{Y}$
for some $q>0$.
    Then, for any $ 0 < \varepsilon < \frac{1}{2} $, the following results hold:
\begin{enumerate}[(i)]
    \item $\log N(\epsilon,\mathcal{P}(\Xi),\Vert\cdot\Vert_\infty)\lesssim\log (1/\epsilon)$,
    \item $H_B(\epsilon,\mathcal{P}(\Xi),d_H)\lesssim\log (1/\epsilon)$.
\end{enumerate}
\end{lemma}

Combining the above results, we obtain the desired conclusion for Theorem~\ref{theorem:ConvergenceRateofDensityEstimation}.

Now we will prove Lemma \ref{applemma:convergence-rate-1}, Lemma \ref{applemma:convergence-rate-add-1} and Lemma \ref{applemma:convergence-rate-2} in order. 
At first we need to introduce another Lemma \ref{applemma:convergence-rate-3} before we prove Lemma \ref{applemma:convergence-rate-1}. 
Lemma \ref{applemma:convergence-rate-3} is Theorem 5.11 in \cite{Vandegeer-2000} and its proof can also be found in \cite{Vandegeer-2000}.
\begin{lemma}
\label{applemma:convergence-rate-3}
    Let $R > 0$, $k \geq 1$ and  
 $\mathcal{G}$ is a subset in $\Xi$ where $\Gs\in\mathcal{G}\subset\Xi$ .
 Given $C_1<\infty$, for all $C$ sufficiently large, and for $n\in\mathbb{N}$ and $t>0$ is in the following range
 \begin{align}
     t\leq(8\sqrt{n}R)\wedge(C_1\sqrt{n}R^2/K),
 \end{align}
\begin{align}
     t\geq C^2(C_1+1)\Bigg( R\vee\int^{R}_{t/(2^6\sqrt{n})}H_B^{1/2}\big(\frac{u}{\sqrt{2}},\linephalf(\Xi,R),m \big) du\Bigg),
\end{align}
then we will have
\begin{align}
    \mathbb{P}_{G_{*,n}}
    \Big(
    \sup_{G\in\mathcal{G},\bbE_X[h(\barplbg(\cdot|X),\plbgs(\cdot|X))]\leq R}
    |\mu_n(G)|\geq t
    \Big)
    \leq
    C\exp
    \left(
    -\frac{t^2}{C^2(C_1+1)R^2}
    \right).
\end{align}
\end{lemma}

\begin{proof}[Proof of Lemma \ref{applemma:convergence-rate-1}]
Firstly, by Lemma 4.1 and 4.2 in \cite{Vandegeer-2000}, we have
\begin{align*}
    \frac{1}{16}\bbE_X[d_H^2(\phlbgn(\cdot|X),\plbgs(\cdot|X))]\leq \bbE_X[d_H^2(\barphlbgn(\cdot|X),\plbgs(\cdot|X))] \leq \frac{1}{\sqrt{n}}\mu_n(\hlbgn),
\end{align*}
here $\mu_n(\hlbgn)$ is an empirical process defined as
\begin{align*}
    \mu_n(\hlbgn):=\sqrt{n}\int_{\plbgs>0}
    \frac{1}{2}\log\left(\frac{\barphlbgn}{\plbgs}\right)(\barphlbgn-\plbgs)d(X,Y).
\end{align*}
Thus, for any $\delta>\delta_n:=\sqrt{\log n/n}$, we have
\begin{align*}
    &\mathbb{P}_{G_{*,n}}(\bbE_X[d_H(\phlbgn(\cdot|X),\plbgs(\cdot|X))]\geq\delta)
    \\&\leq\mathbb{P}_{G_{*,n}}
    \left(
    \mu_n(\hlbgn)-\sqrt{n}\bbE_X[d_H^2(\phlbgn(\cdot|X),\plbgs(\cdot|X))]\geq0,
    \bbE_X[d_H(\phlbgn(\cdot|X),\plbgs(\cdot|X))]\geq\frac{\delta}{4}
    \right)
    \\&\leq\mathbb{P}_{G_{*,n}}
    \left(
    \sup_{\lbg:\bbE_X[d_H(\bar{p}_{\lbg}(\cdot|X),\plbgs(\cdot|X))]\geq\delta/4}
    \left[
    \mu_n(\lbg)-\sqrt{n}\bbE_X[d_H^2(\barplbg(\cdot|X),\plbgs(\cdot|X))]
    \right]\geq0
    \right)
    \\&\leq\sum_{s=0}^S\mathbb{P}_{G_{*,n}}
    \left(
    \sup_{\lbg:2^s\delta/4\leq \bbE_X[d_H(\bar{p}_{\lbg}(\cdot|X),\plbgs(\cdot|X))]\leq 2^{s+1}\delta/4}
    \left|
    \mu_n(\lbg)
    \right|
    \geq\sqrt{n}2^{2s}(\frac{\delta}{4})^2
    \right)
    \\&\leq\sum_{s=0}^S\mathbb{P}_{G_{*,n}}
    \left(
    \sup_{\lbg: \bbE_X[d_H(\bar{p}_{\lbg}(\cdot|X),\plbgs(\cdot|X))]\leq 2^{s+1}\delta/4}
    \left|
    \mu_n(\lbg)
    \right|
    \geq\sqrt{n}2^{2s}(\frac{\delta}{4})^2
    \right)
\end{align*}
where $S$ is a smallest number such that $2^S\delta/4 > 1$
% \textcolor{blue}{,as $h(\barplbg,\plbgs)\leq1$} 
. 

Now we will use Lemma \ref{applemma:convergence-rate-3}: choose $R=2^{s+1}\delta, C_1=15$ and $t=\sqrt{n}2^{2s}(\delta/4)^2$.
We can confirm that condition (i) in Lemma 3 is met since $2^{s-1} \delta / 4 \leq 1$ for all $s \leq S$.
For the condition (ii), it is still satisfied since
\begin{align*}
    &\int^R_{t/2^6\sqrt{n}}
    H_B^{1/2}\left(\frac{u}{\sqrt{2}},\mathcal{P}^{1/2}(\Xi,R),\mu  \right)
    du\vee2^{s+1}\delta
    \\&=\sqrt{2}\int^{R/\sqrt{2}}_{R^2/2^{13}}
    H_B^{1/2}\left({u},\mathcal{P}^{1/2}(\Xi,R),\mu  \right)
     du\vee2^{s+1}\delta
     \\&\leq2\mathcal{J}_B\left(R,\mathcal{P}^{1/2}(\Xi,R),\mu  \right)
     \\&\leq2J\sqrt{n}2^{2s+1}\delta^2
     \\&=2^6Jt.
\end{align*}
Now since the two conditions in Lemma \ref{applemma:convergence-rate-3} are all satisfied, we could conclude that
\begin{align}
    \mathbb{P}_{G_{*,n}}
    \left( \bbE_X[d_H(\phlbgn(\cdot|X),\plbgs(\cdot|X))]>\delta \right)
    \leq C\sum_{s=0}^{\infty}
    \exp\left(-\frac{2^{2s}n\delta^2}{2^{14}C^2} \right)
    \leq 
    c\exp\left( -\frac{n\delta^2}{c} \right),
\end{align}
here constant $c$ is a large constant that does not depend on $\Gs$.
Now we could derive the bound on supremum of expectation:
\begin{align*}
    \mathbb{E}_{p_{G_{*,n}}}\bbE_X[d_H({\phlbgn(\cdot|X),\plbgs(\cdot|X)})]
    &=\int^{\infty}_{0}\mathbb{P}
    \left( \bbE_X[d_H({\phlbgn(\cdot|X),\plbgs(\cdot|X)})]>\delta \right)d\delta
    \\&\leq\delta_n+c\int^{\infty}_{\delta_n}\exp \left(-\frac{n\delta^2}{c^2} \right)d\delta
    \\& \leq \Tilde{c}\delta_n,
\end{align*}
here $\Tilde{c}$ is independent from $\lbgs$ and $\delta_n:=\sqrt{\log n/n}$.
So we can conclude that
% \begin{align*}
%     \sup_{\Gs\in\Xi}
%     \mathbb{E}_{p_{ \Gs}}
%     h(p_{ \widehat{G}_n},p_{ \Gs})
%     \leq
%     C\sqrt{\log n/n}.
% \end{align*}
\begin{align*}
    \sup_{\Gs\in\Xi}
    \mathbb{E}_{p_{ \Gs,n}}
    \bbE_X[d_H(p_{ \widehat{G}_n}(\cdot|X),p_{ \Gs}(\cdot|X))]
    \leq
    C\sqrt{\log n/n}.
\end{align*}
\end{proof}

\begin{proof}[Proof of Lemma \ref{applemma:convergence-rate-add-1}]
Because $\overline{\mathcal{P}}^{1/2}(\Xi,\delta)\subset\overline{\mathcal{P}}^{1/2}(\Xi)$ and from the definition of Hellinger distance, we have
\begin{align*}
    H_B(\delta,\overline{\mathcal{P}}^{1/2}(\Xi,\delta),\mu)
    \leq
    H_B(\delta,\overline{\mathcal{P}}^{1/2}(\Xi),\mu)
    =
    H_B\left(\frac{\delta}{\sqrt{2}},\overline{\mathcal{P}}(\Xi),h\right).
\end{align*}
Now, using the fact that for densities $ f^*, f_1, f_2 $, we have $ h^2 \left( \frac{f_1 + f^*}{2}, \frac{f_2 + f^*}{2} \right) \leq \frac{h^2(f_1, f_2)}{2} $, it is easy to verify that $ H_B(\delta/\sqrt{2}, \overline{\mathcal{P}}(\Xi), d_H) \leq H_B(\delta, \mathcal{P}(\Xi), d_H) $.
Hence, if equation (\ref{assumption:A3}) holds true, then
\begin{align*}
   H_B(\delta,\overline{\mathcal{P}}^{1/2}(\Xi,\delta),\mu)
    \leq
   H_B(\delta,{\mathcal{P}}(\Xi),d_H)
    \lesssim\log\left(\frac{1}{\delta} \right).
\end{align*}
This implies that
\begin{align*}
    \mathcal{J}_B
    \left(
    \epsilon,
    \overline{\mathcal{P}}^{1/2}(\Xi,\delta),\mu
    \right)
    \lesssim
    \epsilon
    \left(
    \log(\frac{2^{13}}{\epsilon^2})
    \right)^{\frac{1}{2}}
    <n\epsilon^2,\quad
\text{for all }
d\epsilon>\sqrt{\displaystyle\frac{\log n}{n}}.
\end{align*}
\end{proof}

\begin{proof}[Proof of Lemma \ref{applemma:convergence-rate-2}]

\textbf{Proof for (i):}
Let $\mathcal{E}_\epsilon(S)$ denote an $\epsilon$-net of a set $S$ under the $\|\cdot\|_\infty$ norm. Then 
\begin{align*}
\log |\mathcal{E}_\epsilon(S)| = \log N(\epsilon, S, \|\cdot\|_\infty).
\end{align*}

Let $\mathcal{P}(\Theta) := \{ p_\Upsilon : \Upsilon \in \Theta \}$, where $p_\Upsilon(Y|X) := f(Y | h(X,\eta), \nu) 
% \bar{f}(X)
$. By Lemma 6 in \cite{ho2022gaussian}, we have
\begin{align*}
\log N(\epsilon, \mathcal{P}(\Theta), \|\cdot\|_\infty) \lesssim \log(1/\epsilon).
\end{align*}

We now consider the contaminated model $p_{\Upsilon}$ as a composition of smooth components indexed by $(\beta, \tau, \eta, \nu) \in \Xi := \Gamma \times \Theta$, where $\Gamma \subset \mathbb{R}^{d+1}$ and $\Theta \subset \mathbb{R}^{q}\times\mathbb{R}^{+}$ are compact.

Since $\sigma(\beta^\top X + \tau) := \exp(\beta^\top X + \tau)/(1 + \exp(\beta^\top X + \tau))$ is infinitely differentiable and Lipschitz over compact $\Gamma$, it follows that for any $\lambda = (\beta, \tau) \in \Gamma$, there exists $\widetilde{\lambda} = (\widetilde{\beta}, \widetilde{\tau}) \in \mathcal{E}_\epsilon(\Gamma)$ such that
\begin{align*}
\| \sigma_{\lambda} - \sigma_{\widetilde{\lambda}} \|_\infty 
:= \sup_{X \in \mathcal{X}} \left| \frac{\exp(\beta^\top X + \tau)}{1 + \exp(\beta^\top X + \tau)} - \frac{\exp(\widetilde{\beta}^\top X + \widetilde{\tau})}{1 + \exp(\widetilde{\beta}^\top X + \widetilde{\tau})} \right| \leq \epsilon.
\end{align*}
Likewise, for any $\Upsilon = (\eta, \nu) \in \Theta$, there exists $\widetilde{\Upsilon} \in \mathcal{E}_\epsilon(\Theta)$ such that
\begin{align*}
\| p_\Upsilon - p_{\widetilde{\Upsilon}} \|_\infty \leq \epsilon.
\end{align*}

Now, consider the difference
\begin{align*}
&\quad p_{G}(Y | X) - p_{\widetilde{G}}(Y | X) \\
&= \left( \sigma_{\lambda}(X) - \sigma_{\widetilde{\lambda}}(X) \right) \left[ f(Y | h(X, \eta), \nu) - f_0(Y | h_0(X, \eta_0), \nu_0) \right] \\
&\quad + \sigma_{\widetilde{\lambda}}(X) \left[ f(Y | h(X, \eta), \nu) - f(Y | h(X, \widetilde{\eta}), \widetilde{\nu}) \right],
\end{align*}
so that by the triangle inequality and boundedness of $f_0$ and $f$,
\begin{align*}
\| p_{G} - p_{ \widetilde{G}} \|_\infty
&\leq \| \sigma_{\lambda} - \sigma_{\widetilde{\lambda}} \|_\infty \cdot \left( \|f_0\|_\infty + \|f\|_\infty \right)
+ \| \sigma_{\widetilde{\lambda}} \|_\infty \cdot \| p_\Upsilon - p_{\widetilde{\Upsilon}} \|_\infty \\
&\lesssim \epsilon.
\end{align*}

Hence, the covering number of $\mathcal{P}(\Xi)$ satisfies
\begin{align*}
\log N(\epsilon, \mathcal{P}(\Xi), \| \cdot \|_\infty) 
\leq \log N(\epsilon, \Gamma, \| \cdot \|_\infty) 
+ \log N(\epsilon, \mathcal{P}(\Theta), \| \cdot \|_\infty)
\lesssim \log(1/\epsilon).
\end{align*}

\textbf{Proof for (ii):}
First, let $ \eta \leq \varepsilon $ be a positive number, which will be chosen later. We consider $f$ is the density function of an univariate Gaussian distribution, so $f$ is light tail: 
for any $|Y|\geq2a$ and $X\in\mathcal{X}$, 
\begin{align*}
    f(\yha)\leq\frac{1}{\sqrt{2\pi}\ell}
    \exp\left(-\frac{Y^2}{8u^2} \right).
\end{align*}
Also $f_0$ is bounded with tail 
$
\log f_0(Y|h (X,\eta_0 ),\nu_0)
\lesssim
-Y^q
$
and $f_0(Y|h (X,\eta_0 )),\nu_0)\leq M$
for almost surely $Y\in\mathcal{Y}$
for some $M,q>0$. 
Now let $q=\min\{p,2 \}$ and $C_2=\max\left\{ M,1/{\sqrt{2\pi}\ell} \right\}$, we will have 
\begin{align}
    H(X,Y)=
    \begin{cases}
        C_1\exp(-Y^q)
        % \bar{f}(X)
        ,& |Y|\geq2a \\
        C_2
        % \bar{f}(X)
        ,&|Y|<2a
    \end{cases}
\end{align}
here $C_1$ is a positive constant depending on $\ell$ and $f_0$. Moreover $H(X,Y)$ is an envelope of $\mathcal{P}(\Xi)$.
Next, let $ g_1, \dots, g_N $ represent an $ \eta $-net over $ \mathcal{P}_k(\Xi) $. Then, we construct the brackets $[p^L_i(X, Y), p^U_i(X, Y)]$ as follows:
\begin{align*}
    \begin{cases}
        p^L_i(X, Y):=\max\{g_i(X,Y)-\eta,0 \}\\
        p^U_i(X, Y):=\min\{g_i(X,Y)+\eta, H(X,Y) \}
    \end{cases}
\end{align*}
for $i=1,\cdots,N$. 
As a result, $ \mathcal{P}_k(\Xi) \subset \bigcup_{i=1}^N [p^L_i(X, Y), p^U_i(X, Y)] $ and $ p^U_i(X, Y) - p^L_i(X, Y) \leq \min\{2\eta, H(X, Y)\} $. Consequently,
\begin{align*}
    &\int\left( p^U_i(X, Y)-p^L_i(X, Y)\right)d(X,Y)\\
    &\leq
    \int_{|Y|<2a}\left( p^U_i(X, Y)-p^L_i(X, Y)\right)d(X,Y)
    +
    \int_{|Y|\geq2a}\left( p^U_i(X, Y)-p^L_i(X, Y)\right)d(X,Y)
    \\&\leq
    \int_{|Y|<2a}2\eta d(X,Y)
    +
    \int_{|Y|\geq2a}H(X,Y)d(X,Y)
    \lesssim\eta.
\end{align*}
This shows that $$ H_B(c\eta, \mathcal{P}(\Xi), \|\cdot\|_1) \leq N \lesssim \log(1/\eta). $$
Setting $ \eta = \epsilon/c $, 
we find $$
H_B(\epsilon, \mathcal{P}(\Xi), \|\cdot\|_1) \lesssim \log(1/\epsilon). $$
Since $ h^2 \leq \|\cdot\|_1 $ holds between the Hellinger distance and the total variation distance, we conclude the bracketing entropy bound.
\end{proof}

\end{document}

%% file: math_commands.tex
\usepackage{eqnarray,amsmath}
\usepackage{booktabs}       % professional-quality tables
\usepackage{nicefrac}       % compact symbols for 1/2, etc.

\usepackage{subfig}
\usepackage{epsf}
\usepackage{epsfig}
\usepackage{fancyhdr}
\usepackage{graphics}
\usepackage{graphicx}
\usepackage{epstopdf}
\usepackage{psfrag}
\usepackage{fullpage}
\usepackage{pdfpages}

\usepackage{enumerate}   
\usepackage{multirow}
\usepackage{bm}
\usepackage{verbatim} 

\usepackage{mathtools}

\usepackage{url}% for url's in bib
\usepackage[colorlinks,linkcolor=magenta,citecolor=blue, pagebackref=true]{hyperref}
\renewcommand*{\backrefalt}[4]{%
    \ifcase #1 \footnotesize{(Not cited.)}%
    \or        \footnotesize{(Cited on page~#2.)}%
    \else      \footnotesize{(Cited on pages~#2.)}%
    \fi}
\usepackage{color}

\usepackage{amsthm}
\usepackage{amsmath}
\usepackage{amssymb,bbm}
\usepackage{caption}
\usepackage{algorithmic}
\usepackage{algorithm}
\usepackage{textcomp}
\usepackage{siunitx}
\usepackage{wrapfig}
\usepackage{algorithmic}
\usepackage{algorithm}
\usepackage{multirow}
\usepackage{multicol}% colors
\usepackage{mathtools}

\def\1{\bm{1}}

\DeclareMathAlphabet{\mathsfit}{\encodingdefault}{\sfdefault}{m}{sl}
\SetMathAlphabet{\mathsfit}{bold}{\encodingdefault}{\sfdefault}{bx}{n}

\newcommand{\sigmoid}{\mathrm{sigmoid}}

\newcommand{\lambdan}{\lambda_n}

\newcommand{\Gn}{G_{n}}
\newcommand{\Gsn}{G_{*,n}}

\newcommand{\nun}{\nu_{n}}
\newcommand{\nusn}{\nu^{*}_{n}}

\newcommand{\da}{\Delta a_{}}
\newcommand{\das}{\Delta a^{*}_{}}
\newcommand{\db}{\Delta b_{}}
\newcommand{\dbs}{\Delta b^{*}_{}}
\newcommand{\dnu}{\Delta \nu_{}}
\newcommand{\dnus}{\Delta \nu^{*}_{}}

\newcommand{\hGn}{\widehat{G}_{n}}

\newcommand{\lGn}{\overline{G}_{n}}

\newcommand{\Gs}{G_{*}}
\newcommand{\lambdas}{\lambda^*}

\newcommand{\nus}{\nu^{*}}
\newcommand{\taus}{\tau^{*}}
\newcommand{\etas}{\eta^{*}}
\newcommand{\betas}{\beta^{*}}

\newcommand{\dtwo}{{
{D_2}(\Gn, \Gsn)
}}

\newcommand{\dtwobarl}{{
\overline{D_2}(G,\Gs)
}}

\newcommand{\done}{D_1(\Gn,\Gsn)}

\newcommand{\ysigmaonen}{ Y|\sigma((a_{1,n})^{\top}X+b_{1,n}),\nu_{1,n} }
\newcommand{\ysigmatwon}{ Y|\sigma((a_{2,n})^{\top}X+b_{2,n}),\nu_{2,n} }

\newcommand{\yhns}{ Y|h(X,\eta_n^*),\nu_n^*) }

\newcommand{\yha}{ Y|h(X,\eta),\nu}
\newcommand{\yhonen}{ Y|h(X,\eta_{1,n}),\nu_{1,n} }
\newcommand{\yhtwon}{ Y|h(X,\eta_{2,n}),\nu_{2,n} }

\newcommand{\xetas}{ X,\etasn }

\newcommand{\lbg}{ G}
\newcommand{\lbgs}{ \Gs}
\newcommand{\hlbgn}{ \hGn}
\newcommand{\llbgn}{ \lGn}
\newcommand{\plbg}{p_{\lbg}}
\newcommand{\plbgs}{p_{\lbgs}}
\newcommand{\phlbgn}{p_{\hGn}}
\newcommand{\barplbg}{\Bar{p}_{\lbg}}
\newcommand{\barphlbgn}{\Bar{p}_{\hlbgn}}

\newcommand{\linephalf}{\overline{\mathcal{P}}^{1/2}}

\DeclareMathOperator*{\argmax}{arg\,max}

\usepackage{eqnarray,amsmath}

\usepackage{amsthm}
\usepackage{amsmath}
\usepackage{amssymb,bbm}

\allowdisplaybreaks
\newcommand{\bbE}{\mathbb{E}}
\newcommand{\bbn}{\mathbb{N}}

\newcommand{\plbgn}{p_{ \Gn}}

\newcommand{\cali}{\mathcal{I}}
\newcommand{\calh}{\mathcal{H}}

\newcommand{\calw}{\mathcal{W}}

\newcommand{\Ione}{
{\text{\uppercase\expandafter{\romannumeral1}}}}
\newcommand{\Itwo}{
{\text{\uppercase\expandafter{\romannumeral2}}}
}
\newcommand{\Ithree}{
{\text{\uppercase\expandafter{\romannumeral3}}}
}
\newcommand{\Ifour}{
{\text{\uppercase\expandafter{\romannumeral4}}}
}
\newcommand{\Ifive}{
{\text{\uppercase\expandafter{\romannumeral5}}
}}

\newcommand{\betan}{\beta_n}
\newcommand{\betasn}{\beta_n^*}

\newcommand{\etan}{\eta_{n}}
\newcommand{\etasn}{\eta^{*}_{n}}
\newcommand{\detan}{\Delta\eta_{n}}
\newcommand{\detasn}{\Delta\eta^{*}_{n}}

\usepackage[]{color-edits}
\addauthor{fqy}{blue}

\usepackage[]{color-edits}
\addauthor{dql}{orange}

%% file: references.bib
@inproceedings{do2023deviated,
 author = {Do, Dat and Nguyen, Huy and Nguyen, Khai and Ho, Nhat},
 booktitle = {Advances in Neural Information Processing Systems},
 pages = {30096--30133},
 publisher = {Curran Associates, Inc.},
 title = {Minimax Optimal Rate for Parameter Estimation in Multivariate Deviated Models},
 volume = {36},
 year = {2023}
}

@inproceedings{li-liang-2021-prefix,
    title = "Prefix-Tuning: Optimizing Continuous Prompts for Generation",
    author = "Li, Xiang Lisa  and
      Liang, Percy",
    booktitle = "Proceedings of the 59th Annual Meeting of the Association for Computational Linguistics and the 11th International Joint Conference on Natural Language Processing (Volume 1: Long Papers)",
    month = aug,
    year = "2021",
    address = "Online",
    publisher = "Association for Computational Linguistics"
}

@inproceedings{
nguyen2024sigmoid,
title={Sigmoid Gating is More Sample Efficient than Softmax Gating in Mixture of Experts},
author={Huy Nguyen and Nhat Ho and Alessandro Rinaldo},
booktitle={The Thirty-eighth Annual Conference on Neural Information Processing Systems},
year={2024}
}

@InProceedings{nguyen2024deviated,
  title = 	 {On Parameter Estimation in Deviated {G}aussian Mixture of Experts},
  author =       {Nguyen, Huy and Nguyen, Khai and Ho, Nhat},
  booktitle = 	 {Proceedings of The 27th International Conference on Artificial Intelligence and Statistics},
  year = 	 {2024},
}

@inproceedings{han2024fusemoe,
  title={FuseMoE: Mixture-of-Experts Transformers for Fleximodal Fusion},
  author={Han, Xing and Nguyen, Huy and Harris, Carl and Ho, Nhat and Saria, Suchi},
  booktitle = "Advances in Neural Information Processing Systems",
  year={2024}
}

@article{nguyen2025competesmoe,
      title={CompeteSMoE -- Statistically Guaranteed Mixture of Experts Training via Competition}, 
      author={Nam V. Nguyen and Huy Nguyen and Quang Pham and Van Nguyen and Savitha Ramasamy and Nhat Ho},
      journal={arXiv preprint arXiv:2505.13380},
      year={2025}
}

@inproceedings{yan2025contaminated,
    title = {Understanding Expert Structures on Minimax Parameter Estimation in Contaminated Mixture of Experts},
    author = {Fanqi Yan and Huy Nguyen and Dung Le and Pedram Akbarian and Nhat Ho},
    booktitle = {Proceedings of The 28th International Conference on Artificial Intelligence and Statistics},
    year = 2025
}

@inproceedings{nguyen2025cosine,
    author = {Huy Nguyen and Pedram Akbarian and Trang Pham and Trang Nguyen and Shujian Zhang and Nhat Ho},
    title = {Statistical Advantages of Perturbing Cosine Router in Mixture of Experts},
    booktitle = {International Conference on Learning Representations},
    year = 2025
}

@inproceedings{
le2024mixture,
title={Mixture of Experts Meets Prompt-Based Continual Learning},
author={Minh Le and An Nguyen The and Huy Nguyen and Thien Trang Nguyen Vu and Huyen Trang Pham and Linh Ngo Van and Nhat Ho},
booktitle={The Thirty-eighth Annual Conference on Neural Information Processing Systems},
year={2024}
}

@inproceedings{
le2025revisiting,
title={Revisiting Prefix-tuning: Statistical Benefits of Reparameterization among Prompts},
author={Minh Le and Chau Nguyen and Huy Nguyen and Quyen Tran and Trung Le and Nhat Ho},
booktitle={The Thirteenth International Conference on Learning Representations},
year={2025},
url={https://openreview.net/forum?id=QjTSaFXg25}
}

@article{dai2024deepseekmoe,
      title={DeepSeekMoE: Towards Ultimate Expert Specialization in Mixture-of-Experts Language Models}, 
      author={Damai Dai and Chengqi Deng and Chenggang Zhao and R. X. Xu and Huazuo Gao and Deli Chen and Jiashi Li and Wangding Zeng and Xingkai Yu and Y. Wu and Zhenda Xie and Y. K. Li and Panpan Huang and Fuli Luo and Chong Ruan and Zhifang Sui and Wenfeng Liang},
      year={2024},
     journal = {arXiv preprint arXiv:2401.04088}
}

@inproceedings{
li2025cl,
title={Theory on Mixture-of-Experts in Continual Learning},
author={Hongbo Li and Sen Lin and Lingjie Duan and Yingbin Liang and Ness Shroff},
booktitle={The Thirteenth International Conference on Learning Representations},
year={2025}
}

@inproceedings{kwon_em_2020,
	series = {Proceedings of {Machine} {Learning} {Research}},
	title = {{EM} {Converges} for a {Mixture} of {Many} {Linear} {Regressions}},
	volume = {108},
	url = {https://proceedings.mlr.press/v108/kwon20a.html},
	booktitle = {Proceedings of the {Twenty} {Third} {International} {Conference} on {Artificial} {Intelligence} and {Statistics}},
	publisher = {PMLR},
	author = {Kwon, Jeongyeol and Caramanis, Constantine},
	editor = {Chiappa, Silvia and Calandra, Roberto},
	month = aug,
	year = {2020},
	pages = {1727--1736},
}

@inproceedings{liang_m3vit_2022,
	title = {M$^3${ViT}: {Mixture}-of-{Experts} {Vision} {Transformer} for {Efficient} {Multi}-task {Learning} with {Model}-{Accelerator} {Co}-design},
	booktitle = {{NeurIPS}},
	author = {Liang, Hanxue and Fan, Zhiwen and Sarkar, Rishov and Jiang, Ziyu and Chen, Tianlong and Zou, Kai and Cheng, Yu and Hao, Cong and Wang, Zhangyang},
	year = {2022},
}

@book{Vandegeer-2000,
author= "S. van de Geer",
title="Empirical Processes in M-estimation",
publisher= "Cambridge University Press",
year="2000"
}

@book{Lindsay-1995,
	author="B. Lindsay",
	title="Mixture models: Theory, geometry and applications",
	publisher="In NSF-CBMS Regional Conference Series in Probability and Statistics. IMS, Hayward, CA.",
	year="1995"
}

@article{Ho-Nguyen-EJS-16,
	author="N. Ho and X. Nguyen",
	title="On strong identifiability and convergence rates of parameter estimation in finite mixtures",
	journal="Electronic Journal of Statistics",
	volume="10",
	pages="271-307",
	year="2016"
}

@article{Jacob_Jordan-1991,
	author="R. A. Jacobs and M. I. Jordan and S. J. Nowlan and G. E. Hinton",
	title="Adaptive mixtures of local experts",
	journal="Neural Computation",
	volume="3",
	page="79-87",
	year="1991"
}

@article{ho2022gaussian,
  author  = {Nhat Ho and Chiao-Yu Yang and Michael I. Jordan},
  title   = {Convergence Rates for {G}aussian Mixtures of Experts},
  journal = {Journal of Machine Learning Research},
  year    = {2022},
  volume  = {23},
  number  = {323},
  pages   = {1--81},
}

@inproceedings{
ceron2024rl,
title={Mixtures of Experts Unlock Parameter Scaling for Deep {RL}},
author={Johan Samir Obando Ceron and Ghada Sokar and Timon Willi and Clare Lyle and Jesse Farebrother and Jakob Nicolaus Foerster and Gintare Karolina Dziugaite and Doina Precup and Pablo Samuel Castro},
booktitle={Forty-first International Conference on Machine Learning},
year={2024}
}

@article{Jordan-1994,
	author="M. I. Jordan and R. A. Jacobs",
	title="Hierarchical mixtures of experts and the {EM} algorithm",
	journal="Neural Computation",
	volume="6",
	pages="181-214",
	year="1994"
}

@INPROCEEDINGS{shazeer2017topk,
   AUTHOR = "N. Shazeer and A. Mirhoseini and K. Maziarz and A. Davis and Q. Le and G. Hinton and J. Dean",
   TITLE = "Outrageously Large Neural Networks: The Sparsely-Gated Mixture-of-Experts Layer",
   BOOKTITLE = 	 "In International Conference on Learning Representations", 
   YEAR = 	 2017
}

@inproceedings{You_Speech_MoE,
  title = "SpeechMoE: Scaling to Large Acoustic Models with Dynamic Routing Mixture of Experts",
  author = "Z. You and S. Feng and D. Su and D. Yu",
  booktitle = "Interspeech",
  year = "2021"
}

@inproceedings{You_Speech_MoE_2,
	title = {Speechmoe2: Mixture-of-experts model with improved routing},
	doi = {10.1109/ICASSP43922.2022.9747065},
	booktitle = {{ICASSP} 2022 - 2022 {IEEE} {International} {Conference} on {Acoustics}, {Speech} and {Signal} {Processing} ({ICASSP})},
	author = {You, Zhao and Feng, Shulin and Su, Dan and Yu, Dong},
	year = {2022},
	pages = {7217--7221},
}

@article{le2025expressivenessvisualpromptexperts,
      title={On the Expressiveness of Visual Prompt Experts}, 
      author={Minh Le and Anh Nguyen and Huy Nguyen and Chau Nguyen and Anh Tran and Nhat Ho},
      year={2025},
      Journal = {arxiv preprint arxiv 2501.18936} 
}

@article{nguyen2024hmoe,
      title={On Expert Estimation in Hierarchical Mixture of Experts: Beyond Softmax Gating Functions}, 
      author={Huy Nguyen and Xing Han and Carl William Harris and Suchi Saria and Nhat Ho},
      year={2024},
      Journal = {arxiv preprint arxiv 2410.02935}
}

@inproceedings{
yun2024flexmoe,
title={Flex-MoE: Modeling Arbitrary Modality Combination via the Flexible Mixture-of-Experts},
author={Sukwon Yun and Inyoung Choi and Jie Peng and Yangfan Wu and Jingxuan Bao and Qiyiwen Zhang and Jiayi Xin and Qi Long and Tianlong Chen},
booktitle={The Thirty-eighth Annual Conference on Neural Information Processing Systems},
year={2024}
}

@article{deepseekv3,
  title={Deepseek-v3 technical report},
  author={DeepSeek-AI and others},
  journal={arXiv preprint arXiv:2412.19437},
  year={2024}
}

@article{grattafiori2024llama3,
  title={The Llama 3 Herd of Models},
  author={Aaron Grattafiori and Abhimanyu Dubey and Abhinav Jauhri and Abhinav Pandey and Abhishek Kadian and Ahmad Al-Dahle and Aiesha Letman and Akhil Mathur and others},
  journal={arXiv preprint arXiv:2407.21783},
  year={2024}
}

@article{jiang2024mixtral,
      title={Mixtral of Experts}, 
      author={Albert Q. Jiang and Alexandre Sablayrolles and Antoine Roux and Arthur Mensch and Blanche Savary and Chris Bamford and Devendra Singh Chaplot and Diego de las Casas and Emma Bou Hanna and Florian Bressand and Gianna Lengyel and Guillaume Bour and Guillaume Lample and Lélio Renard Lavaud and Lucile Saulnier and Marie-Anne Lachaux and Pierre Stock and Sandeep Subramanian and Sophia Yang and Szymon Antoniak and Teven Le Scao and Théophile Gervet and Thibaut Lavril and Thomas Wang and Timothée Lacroix and William El Sayed},
      year={2024},
      Journal = {arxiv preprint arxiv 2401.04088}
}

@inproceedings{lepikhin_gshard_2021,
	title = {{GS}hard: {Scaling} {Giant} {Models} with {Conditional} {Computation} and {Automatic} {Sharding}},
	booktitle = {International {Conference} on {Learning} {Representations}},
	author = {D. Lepikhin and H. Lee and Y. Xu and D. Chen and O. Firat and Y. Huang and M. Krikun and N. Shazeer and Z. Chen},
	year = {2021},
}

@inproceedings{Riquelme2021scalingvision,
 author = {C. Riquelme and J. Puigcerver and B. Mustafa and M. Neumann and R. Jenatton and A. Susano Pint and D. Keysers and N. Houlsby},
 booktitle = {Advances in Neural Information Processing Systems},
 pages = {8583--8595},
 publisher = {Curran Associates, Inc.},
 title = {Scaling Vision with Sparse Mixture of Experts},
 volume = {34},
 year = {2021}
}

@inproceedings{chen2022theory,
 author = {Chen, Zixiang and Deng, Yihe and Wu, Yue and Gu, Quanquan and Li, Yuanzhi},
 booktitle = {Advances in Neural Information Processing Systems},
 editor = {S. Koyejo and S. Mohamed and A. Agarwal and D. Belgrave and K. Cho and A. Oh},
 pages = {23049--23062},
 publisher = {Curran Associates, Inc.},
 title = {Towards Understanding the Mixture-of-Experts Layer in Deep Learning},
 volume = {35},
 year = {2022}
}

@inproceedings{nguyen2023demystifying,
      title={Demystifying Softmax Gating Function in {G}aussian Mixture of Experts}, 
      author={Huy Nguyen and TrungTin Nguyen and Nhat Ho},
      booktitle = "Advances in Neural Information Processing Systems",
      year={2023}
}

@InProceedings{nguyen2024general,
  title = 	 {A General Theory for Softmax Gating Multinomial Logistic Mixture of Experts},
  author =       {Nguyen, Huy and Akbarian, Pedram and Nguyen, Trungtin and Ho, Nhat},
  booktitle = 	 {Proceedings of the 41st International Conference on Machine Learning},
  pages = 	 {37617--37648},
  year = 	 {2024}
}

@article{faria2010regression,
author = {Susana Faria and Gilda Soromenho},
title = {Fitting mixtures of linear regressions},
journal = {Journal of Statistical Computation and Simulation},
volume = {80},
number = {2},
pages = {201-225},
year = {2010},
publisher = {Taylor & Francis}
}

@inproceedings{chow_mixture_expert_2023,
	title = {A {Mixture}-of-{Expert} {Approach} to {RL}-based {Dialogue} {Management}},
	url = {https://openreview.net/forum?id=4FBUihxz5nm},
	booktitle = {The {Eleventh} {International} {Conference} on {Learning} {Representations}},
	author = {Chow, Yinlam and Tulepbergenov, Azamat and Nachum, Ofir and Gupta, Dhawal and Ryu, Moonkyung and Ghavamzadeh, Mohammad and Boutilier, Craig},
	year = {2023},
}

@inproceedings{
li2023sparse,
title={Sparse Mixture-of-Experts are Domain Generalizable Learners},
author={Bo Li and Yifei Shen and Jingkang Yang and Yezhen Wang and Jiawei Ren and Tong Che and Jun Zhang and Ziwei Liu},
booktitle={The Eleventh International Conference on Learning Representations },
year={2023}
}

@article{mendes2011convergence,
      title={Convergence Rates for Mixture-of-Experts}, 
      author={Eduardo F. Mendes and Wenxin Jiang},
      year={2011},
      Journal = {arXiv preprint arxiv 1110.2058}
}

@article{fedus2022switch,
  title={Switch transformers: Scaling to trillion parameter models with simple and efficient sparsity},
  author={Fedus, William and Zoph, Barret and Shazeer, Noam},
  journal={Journal of Machine Learning Research},
  volume={23},
  number={120},
  pages={1--39},
  year={2022}
}

@article{gadat2020parameter,
  title={Parameter recovery in two-component contamination mixtures: The L\^{}2 strategy},
  author={Gadat, S{\'e}bastien and Kahn, Jonas and Marteau, Cl{\'e}ment and Maugis-Rabusseau, Cathy},
  year={2020}
}
